\documentclass{article}
\usepackage[utf8]{inputenc}

\usepackage[margin=1in]{geometry}
\usepackage{microtype}

\usepackage{amsmath,amssymb,amsthm,mathtools,bm}

\usepackage{etoc}
\usepackage{graphicx}
\usepackage{wrapfig}
\usepackage{float}
\usepackage{booktabs}
\usepackage{array}
\usepackage{multirow}
\usepackage{colortbl}
\usepackage[dvipsnames]{xcolor}
\usepackage[font=small,labelfont=bf,skip=3pt]{caption}
\usepackage{subcaption}
\usepackage{algorithm,algpseudocode}
\usepackage{xspace}
\usepackage{enumitem}
\usepackage{titletoc}
\setlist{itemsep=2pt, topsep=3pt}

\makeatletter
\renewcommand{\maketitle}{%
  \begin{center}
    \vspace*{0.4in}
    {\LARGE \@title \par}
    \vspace{0.35in}
    \@author\par
    \vspace{0.3in}
  \end{center}}
\makeatother
\renewenvironment{abstract}{%
  \begin{center}\textbf{\abstractname}\end{center}\vspace{-0.8em}
  \quotation}{\endquotation\vspace{0.1in}}

\usepackage[numbers,sort&compress]{natbib}   
\usepackage{algorithm,algpseudocode}

\usepackage{amsmath,amsfonts,bm}

\def\eqref#1{equation~\ref{#1}}

\def\1{\bm{1}}

\def\rmP{{\mathbf{P}}}

\def\rmR{{\mathbf{R}}}

\def\vzero{{\bm{0}}}

\def\vc{{\bm{c}}}
\def\vd{{\bm{d}}}

\def\vf{{\bm{f}}}
\def\vg{{\bm{g}}}
\def\vh{{\bm{h}}}

\def\vp{{\bm{p}}}

\def\vr{{\bm{r}}}
\def\vs{{\bm{s}}}

\def\vu{{\bm{u}}}
\def\vv{{\bm{v}}}
\def\vw{{\bm{w}}}
\def\vx{{\bm{x}}}
\def\vy{{\bm{y}}}
\def\vz{{\bm{z}}}

\def\mC{{\bm{C}}}

\def\mG{{\bm{G}}}

\def\mI{{\bm{I}}}
\def\mJ{{\bm{J}}}

\def\mM{{\bm{M}}}

\def\mU{{\bm{U}}}
\def\mV{{\bm{V}}}

\DeclareMathAlphabet{\mathsfit}{\encodingdefault}{\sfdefault}{m}{sl}
\SetMathAlphabet{\mathsfit}{bold}{\encodingdefault}{\sfdefault}{bx}{n}

\def\gD{{\mathcal{D}}}

\def\gF{{\mathcal{F}}}
\def\gG{{\mathcal{G}}}

\def\gN{{\mathcal{N}}}

\def\gP{{\mathcal{P}}}

\def\gS{{\mathcal{S}}}

\def\gX{{\mathcal{X}}}
\def\gY{{\mathcal{Y}}}

\newcommand{\pdata}{p_{\rm{data}}}

\newcommand{\E}{\mathbb{E}}

\newcommand{\R}{\mathbb{R}}

\usepackage{hyperref}
\usepackage{url}

\newcommand{\vsig}{\bm{\sigma}}
\newcommand{\vdelta}{\bm{\delta}}
\newcommand{\simplex}{\Delta}
\def\eqref#1{Eq.~(\ref{#1})}

\usepackage[capitalize,nameinlink]{cleveref}
\hypersetup{
  colorlinks=true,
  linkcolor=blue,
  citecolor=green!50!black,
  urlcolor=Rhodamine,
  pdftitle={Learning Where to Steer: Noise-Space Geometry for Efficient Offline Multi-Objective Optimization with Generative Models},
  pdfauthor={Yuan Lu, Esha Singh, Yi-An Ma, Yusu Wang}
}

\newcommand{\method}{ASNS\xspace}

\definecolor{oursrow}{HTML}{E3E0F8} 
\definecolor{oursrowlt}{HTML}{F1EFFB}
\definecolor{oursbest}{HTML}{D2EEF0}   

\newtheorem{theorem}{Theorem}
\newtheorem{proposition}{Proposition}
\newtheorem{corollary}{Corollary}
\newtheorem{lemma}{Lemma}
\newtheorem{assumption}{Assumption}
\newtheorem{definition}{Definition}
\newtheorem{remark}{Remark}

\newcommand{\myparagraph}[1]  {{\noindent{\bf #1~}}}

\newcommand{\dlist}{\itemsep 0.1pt\parsep=1pt\partopsep 0pt}

\newcommand{\fullref}[3]{#1~\hyperref[#3]{\ref*{#2}}}

\newcommand{\restatedname}{}
\newtheorem*{restatedinner}{\restatedname}
\newenvironment{restated}[1]
  {\renewcommand{\restatedname}{#1}\begin{restatedinner}}
  {\end{restatedinner}}

\title{Learning Where to Steer: Noise-Space Geometry for Efficient Offline Multi-Objective Optimization with Generative Models}

\author{%
  {\Large Yuan Lu$^{\ast1}$ \qquad Esha Singh$^{\ast2}$ \qquad Yi-An Ma$^{1}$ \qquad Yusu Wang$^{1}$}\\[10pt]
  {\normalsize
   $^{1}$Halıcıoğlu Data Science Institute, UC San Diego \\
   $^{2}$Department of Computer Science \& Engineering, UC San Diego}
}

\date{}   

\begin{document}
\maketitle

\begingroup
\renewcommand{\thefootnote}{\fnsymbol{footnote}}%
\footnotetext[1]{Equal contribution.}%
\endgroup

\begin{abstract}
Offline multi-objective optimization (MOO) seeks solutions with better objective trade-offs using only a fixed dataset, without querying the objectives. Diffusion models trained on such data have emerged as a promising approach, but their samples are not inherently better than the data and must be steered toward the Pareto front. Existing methods guide or condition every sampling step. We instead act on the initial noise and leave the sampling process unchanged. Across Off-MOO-Bench, we observe that the objectives, as functions of the noise, are sensitive to only a few directions. We estimate these directions once per task via a Recursive Feature Machine using function values alone, and a small cache serves every trade-off, so each candidate costs one noise displacement and one ODE solve. We prove that this displacement increases the learned scalarized objective in expectation, and that sweeping trade-offs recovers the flow's attainable front up to proxy and steering errors. With additional guidance, for which we introduce novel data-adaptive and Pareto-aware operators, our method attains the best average hypervolume rank among generative methods on 47 tasks, at comparable or lower sampling cost. Steering alone outranks the best prior generative method at a fraction of its sampling cost.
\end{abstract}

\etocdepthtag.toc{mtmain}
\section{Introduction}
In offline multi-objective optimization (MOO), we are usually given a fixed dataset $\gD = \{(\vx_i, \vy_i)\}_{i=1}^N$ consisting of candidate solutions $X = \{\vx_i\}^{N}_{i=1} \subset \R^D$ and their values on $m$ objectives $\vy_i = \vf(\vx_i) \in \R^m$. The goal is to propose new solutions that approximate the Pareto front without querying the true objectives on new data points. Generative modeling has recently emerged as a promising approach to this problem. In particular, a flow or diffusion model trained on $\gD$ can generate new solutions that lie within the support of the input data $X$. 
However, faithfully modeling the observed data distribution is not sufficient for optimization, because samples drawn from the learned distribution are not inherently biased toward solutions with better objective values than those already present in the dataset. Improving upon the observed Pareto front may require generating out-of-distribution (OOD) solutions beyond the regions covered by the input data $X$. Thus, the generative process must be explicitly steered toward high-quality regions of the solution space and, in particular, toward the Pareto front. 

\begin{wrapfigure}{r}{0.47\linewidth}
    \centering
    \vspace{-\baselineskip}
    \includegraphics[width=1.0\linewidth, trim=0 7 0 0, clip]{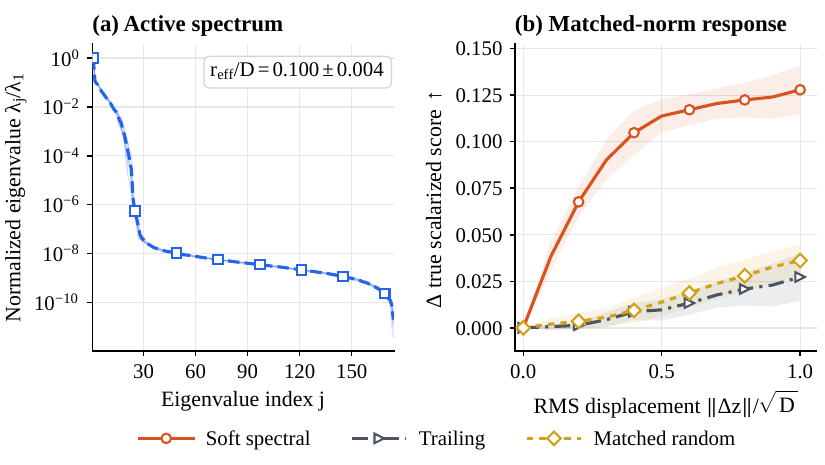}
    \caption{The noise-space objective is low rank. (a) Spectrum of $\mM_\vw$ on IN-1K/MOP8. A handful of directions carry nearly all the sensitivity. (b) Displacing along them raises the true score. A trailing or random direction of equal norm does not.}
    \label{fig:teaser}
    \vspace{-\baselineskip}
\end{wrapfigure}

Existing approaches such as ParetoFlow \citep{paretoflow} and PGD-MOO \citep{pgdmoo} achieve this by guiding the sampling trajectory of a flow or diffusion model using proxy-objective gradients or a learned dominance classifier. Such methods require an additional correction or guidance term applied at each integration step during the inference trajectory, increasing inference cost. Moreover, guidance is often constructed from a denoised estimate of the final sample, and this estimate can be unreliable during the early, highly noisy portion of the trajectory. 
An alternative approach, PCD \citep{pcd}, which achieves best prior generative performance, directly trains an objective-conditioned diffusion model and uses classifier-free guidance (CFG) \citep{ho2021classifierfree} at inference time to bias generation toward desirable objective values. This approach avoids the need for an external proxy model, but places a greater burden on the conditional generative model, which must learn a sufficiently accurate relationship between objective values and designs, including in the high-performing regions that are typically sparsely represented in the offline dataset. PCD also reports diminishing returns as the conditioning strength is increased.

\paragraph{Our work.} 
In this paper, we propose \emph{Active-Subspace Noise Steering} (\method), a novel and lightweight \emph{noise-space steering} strategy for any ODE-sampled diffusion model. To illustrate the key idea, consider a flow matching model. Recall that the design space of the MOO problem is $\R^D$ and the $m$ objectives are $\vf:\R^D \to \R^m$. At inference time, sampling begins from a noise vector $\vz\in\R^D$ and follows the ODE trajectory of the learned flow to produce a clean sample $\vx\in\R^D$, a new candidate solution. Since the flow dynamics are deterministic, the entire sampling procedure induces a deterministic map $\Phi_\theta:\R^D\to\R^D$, where $\theta$ denotes the parameters of the learned flow model. Thus, properties of the final solution are already encoded through its dependence on the initial noise. This motivates our central question:
\textit{Can we steer the initial noise efficiently and effectively toward regions that generate solutions near the Pareto front of the MOO problem?}

\vspace{1\baselineskip}
Interestingly, we find that such noise-space steering can be highly effective. Consider a scalarization of the multi-objective function, $h_\vw(\vx):=\vw^\top \vf(\vx)$ (in practice with learned proxies $\hat\vf$), parameterized by a trade-off vector $\vw\in\Delta^{m-1}$. Through the learned generative map $\Phi_\theta$, this induces a scalar-valued function directly on noise space, $H_\vw(\vz):=h_\vw\!\left(\Phi_\theta(\vz)\right)$. Our key empirical observation is that $H_\vw$ is often sensitive to only a few directions in the high-dimensional noise space $\R^D$. In particular, the average pullback metric $\mM_\vw$ induced by $H_\vw$ is approximately low-rank across nearly all datasets we consider (see Section~\ref{sec:lowrank}) for details and Figure~\ref{fig:teaser}a for an example on IN-1K/MOP8 dataset. This suggests that, although the noise space may be high-dimensional, the objective-relevant variation is concentrated in a low-dimensional subspace. If these important directions can be identified, we can steer the initial noise along them so that the resulting samples move toward the Pareto front. 

To identify these objective-relevant directions, we build on the Average Gradient Outer Product (AGOP) and Recursive Feature Machines (RFM) framework of \citet{radhakrishnan2024rfm}, which can extract task-dependent directions effectively using a few examples. RFM+AGOP has previously been used for activation steering in LLMs \citep{beaglehole2026steering} and for steering the activations of image diffusion models \citep{wang2026narfm}. Our setting, however, introduces an important additional challenge: rather than steering with respect to a single target function, we must handle an entire family of scalarized objectives $H_\vw$, parameterized by the trade-off vector $\vw$. In Section~\ref{sec:geometry}, we show that this family possesses additional structure that allows us to precompute and cache only a small collection of matrices that suffices to support arbitrary trade-offs $\vw\in\R^m$.
At inference time, given a desired trade-off vector $\vw$, we efficiently construct the corresponding steering directions from these cached matrices and perturb the initial noise along these directions. The resulting ``improved'' noise vectors are then passed through {\bf the original flow model without further modification}, producing candidate solutions that are biased toward better objective values and closer to the unknown Pareto front. In Section~\ref{sec:steering}, we provide theoretical justification for this noise-space steering principle, showing conditions under which moving along the identified objective-relevant directions provably increases the scalarized proxy objective in expectation. 

\vspace{0.5\baselineskip}
Our noise-space steering strategy is also modular, since it acts only on the initial noise, and can be applied to any pretrained diffusion model, with or without inference-time guidance. In Section~\ref{sec:guidance}, we develop complementary guidance operators for flow and diffusion models, which further improve performance when combined with noise steering. We summarize our contributions as follows.
\begin{itemize}\dlist
    \item We introduce \method, a novel offline MOO approach that steers only the initial noise, leaving the sampler unmodified. The directions are the leading eigenvectors of the AGOP of the noise-space objective, estimated once per task with an RFM, for every trade-off (Section~\ref{sec:method}).
    \item We prove that these directions capture the most squared gradient of any subspace of their dimension, that steering along them increases the learned scalarized objective in expectation, and that sweeping trade-offs approaches the best front the flow can generate, up to estimation and steering errors (Section~\ref{sec:steering}).
    \item Where a single displacement cannot reach, we introduce two families of guidance operators that use information prior guidance discards and never differentiate through the ODE solve. \emph{Data-adaptive} operators restrict the objective gradient to the learned geometry and the local data manifold, and \emph{Pareto-aware} operators combine the per-objective gradients, not their scalarized sum, into a direction along which no objective decreases to first order (Section~\ref{sec:guidance}).
    \item Across 47 tasks and 14 strong baselines, steering with our guidance operators attains the best average hypervolume rank among generative methods and steering alone outranks the best prior generative method at up to an order of magnitude lower sampling cost. Ablations isolate the chosen directions, rank, weighting, and steering strength (Section~\ref{sec:experiments}, \S~\ref{app:ablations}).
\end{itemize}

\paragraph{More related work.}\label{sec:related-works}
Surrogate-based methods fit proxies on $\gD$ and search them with NSGA-II \citep{xue2024offmoo, deb2002nsga2}, often regularized against exploitation away from the data \citep{trabucco2021coms, yu2021roma, qi2022iom, yuan2023ict, chen2023trimentoring}. Beyond the generative methods above, earlier work models $\gD$ for single-objective design \citep{kumar2020mins, krishnamoorthy2023ddom}, SPREAD guides with a common-descent direction \citep{hotegni2026spread}, and a recent diagnosis finds generative methods more conservative than evolutionary search when the data lie far from the front \citep{offlinefrontier2026}. Diffusion noise has also been optimized per sample and per reward \citep{wallace2023doodl, eyring2024reno, zhou2024goldennoise, ma2025inference} or edited along a few patch-PCA components \citep{wang2025seeds}, whereas we learn objective-sensitive directions once and share them across noise vectors and trade-offs.

\section{Preliminaries}
\label{sec:prelim}
\textbf{Offline multi-objective optimization.}
We consider the problem
\begin{equation}\label{eq:moo}
    \max_{\vx \in \gX} \; \vf(\vx) = \bigl(f_1(\vx), \dots, f_m(\vx)\bigr),
\end{equation}
where $\gX \subseteq \R^D$ is a compact design space of candidate solutions and $\vf : \gX \to \R^m$ is a vector of $m \ge 2$ objectives, all maximized without loss of generality. No queries to $\vf$ are permitted. We observe only a fixed dataset $\gD = \{(\vx_i, \vy_i)\}_{i=1}^{N}$ with $\vy_i = \vf(\vx_i)$, from which we train differentiable proxies $\hat{\vf} = (\hat f_1, \dots, \hat f_m)$. The true $\vf$ is reserved for final evaluation.

\textbf{Pareto concepts and scalarization.}
For $\vy, \vy' \in \R^m$, $\vy$ \emph{dominates} $\vy'$ ($\vy \succ \vy'$) if $y_a \ge y'_a$ for all $a$ with strict inequality for some $a$. A solution is \emph{Pareto optimal} if no feasible solution dominates it. The Pareto set $\gP^\star$ collects them and the Pareto front is its image $\gF^\star = \vf(\gP^\star)$. The goal is a finite set $\gS$ whose image is close to $\gF^\star$ and spreads across it, measured jointly by hypervolume.
We scalarize the objectives by a preference weight $\vw$ on the simplex $\Delta^{m-1}$. Throughout we use the weighted sum, and sweep a set of weights to cover the front, though our development is agnostic to the choice of scalarization.
\begin{equation}\label{eq:scalarization}
    h_\vw(\vx) = \vw^\top \hat{\vf}(\vx) = \sum_{a=1}^{m} w_a\, \hat f_a(\vx),
\end{equation}
\textbf{Flow models and guidance.}
Our method applies to any generative model whose sampler is a deterministic map from Gaussian noise to solutions, which includes flow matching models and diffusion models sampled through their probability flow ODE. Such a model has a velocity field $\vv_\theta$, and sampling integrates
\begin{equation}\label{eq:ode}
    \frac{d \vx_t}{d t} = \vv_\theta(\vx_t, t),
    \qquad \vx_0 = \vz \sim \gN(\vzero, \mI_D),
\end{equation}
which defines a map $\Phi_\theta : \vz \mapsto \vx_1$ from noise to solutions. In our experiments we use a rectified flow \citep{liu2023rectified}, trained to match $\vx - \vz$ along $\vx_t = (1-t)\vz + t\vx$ for $\vx \sim \pdata$.
Guidance adds a correction to the velocity at every step,
\begin{equation}\label{eq:guidance}
    \tilde{\vv}(\vx_t, t) = \vv_\theta(\vx_t, t) + \eta(t)\,
    \nabla_{\vx_t} h_\vw\!\bigl(\hat{\vx}_1(\vx_t, t)\bigr),
\end{equation}
evaluated through the one-step denoised estimate $\hat{\vx}_1$ because the proxy is trained on clean solutions. For a rectified flow
$\hat{\vx}_1(\vx_t, t) = \vx_t + (1-t)\,\vv_\theta(\vx_t, t)$, which is
accurate near $t = 1$ and unreliable at small $t$. Prior guided methods for offline MOO follow this scheme (\S~\ref{sec:related-works}). Our method acts on the initial noise $\vx_0 = \vz$ and integrates \eqref{eq:ode} unmodified, and uses \eqref{eq:guidance} only late in the trajectory. 

\subsection{Problem setting}
\label{sec:problem}
Composing the proxies with the sampler turns every objective on solutions into an objective on noise. We write $g_a(\vz) := \hat f_a(\Phi_\theta(\vz))$ for $a = 1, \dots, m$, and for a preference weight $\vw \in \Delta^{m-1}$
\begin{equation}\label{eq:noise-obj}
    H_\vw(\vz) := h_\vw\bigl(\Phi_\theta(\vz)\bigr) = \sum_{a=1}^{m} w_a\, g_a(\vz),
\end{equation}
so that $\nabla H_\vw = \sum_a w_a \nabla g_a$. We reparameterize \eqref{eq:moo} and seek, for each $\vw$, a noise vector $\vz_\vw$ at which $H_\vw$ is large, decoding the candidate $\hat\vx_\vw = \Phi_\theta(\vz_\vw)$ by \eqref{eq:ode}. Sweeping $\vw$ over a finite weight set with several draws per weight, and retaining the non-dominated candidates under $\hat{\vf}$, yields the solution set $\gS$.
The reference for $\gS$ is set by the sampler. It is the best front attainable among the solutions $\Phi_\theta$ decodes, and two approximations separate $\vf(\gS)$ from it. The proxies $\hat{\vf}$ defining $H_{\vw}$ only estimate $\vf$, and $\vz_\vw$ is reached by steering rather than by exact maximization. Appendix~\ref{sec:recovery} defines this reference formally and bounds the distance. How far the reference lies from the true front
$\gF^\star$ depends on the flow's generalization beyond its data, and cannot be bounded without querying $\vf$
(Section~\ref{sec:steering}). 
\section{Noise-Space Steering}
\label{sec:method}
\begin{algorithm}[t]
\small
\caption{Active-Subspace Noise Steering (ASNS)}
\label{alg:steer}
\begin{algorithmic}[1]
\Require A flow map $\Phi_\theta$, proxies $\hat \vf=(\hat f_1,\dots,\hat f_m)$, weight $\vw\in\simplex^{m-1}$, sample budget $n$, steering strength $\gamma$, rank tolerance $\varepsilon_{\mathrm{rank}}$, \emph{optional} guidance operator $\mathcal{G}_\vw$ with onset $t_{\mathrm{start}}$ and strength $\eta(t)$
\Ensure Steered solution $\hat\vx=\Phi_\theta(\vz_{\vw})$: an unmodified-flow sample steered toward the $\vw$-optimum of $H_\vw$
\State \textbf{Precomputation} \emph{(once, independent of $\vw$)} \Comment{cache the noise-space geometry}
\State \quad draw $\vz_1,\dots,\vz_n\sim\gN$;\ \ decode $\vx_i\gets\Phi_\theta(\vz_i)$;\ \ score $\hat\vy_i\gets\hat \vf(\vx_i)$
\State \quad Fit RFM surrogates $\phi_a\approx g_a$ for each objective on $\{(\vz_i,\hat y_{i,a})\}_{i=1}^n$, $a=1,\dots,m$ \Comment{\eqref{eq:agop}}\\
\quad Cache $\mC_{ab}\gets\tfrac1n\sum_{i}\nabla\phi_a(\vz_i)\nabla\phi_b(\vz_i)^\top$, $\ 1\le a\le b\le m$
\State \textbf{Steering} \emph{(per weight $\vw$)}
\State \quad $\widehat\mM_\vw\gets\sum_{a,b}w_a w_b\,\mC_{ab}$\\
\quad  Eigendecompose $\widehat\mM_\vw=\sum_j\widehat\lambda_j^{(\vw)}\widehat{\vu}_\vw^{(j)}\widehat{\vu}_\vw^{(j)\top}$ \ $\lambda_1^{(\vw)}\ge\cdots\ge\lambda_D^{(\vw)}$ \Comment{reuse cached blocks, \eqref{eq:decomp}}
\State \quad $\beta_j\gets\mathbf{1}[j\le k_{\mathrm{eff}}(\vw)]$ (hard) or $(\widehat\lambda_j^{(\vw)}/\widehat\lambda_1^{(\vw)})^{\alpha}$ (soft) \Comment{\eqref{eq:effrank}}
\State \quad $\sigma_j\gets\mathrm{sgn}\,\mathrm{corr}\bigl(\{\langle\vz_i,\widehat{\vu}_\vw^{(j)}\rangle\}_i,\ \{\vw^\top\hat{\vy}_i\}_i\bigr)$ \Comment{orientation from cached data, no ODE}
\State \quad draw $\vz\sim\gN$; $\vz_\vw\gets\vz+\gamma\sum_{j}\beta_j\,\sigma_j\,\widehat{\vu}_\vw^{(j)}$ \label{line:steer} \Comment{weighted step, Thm.~\ref{thm:ascent}}
\State \textbf{Decoding}: if $\mathcal{G}_\vw$ is given, replace $\vv_\theta$ in \eqref{eq:ode} by $\vv_\theta+\eta(t)\,\mathcal{G}_\vw(\vx_t;\{\mC_{ab}\})$ for $t\ge t_{\mathrm{start}}$ \label{line:guide} \Comment{Sec.~\ref{sec:guidance}}
\State \Return $\hat\vx_\vw\gets\Phi_\theta(\vz_\vw)$ \Comment{one ODE solve, unmodified when $\mathcal{G}_\vw$ is absent}
\end{algorithmic}
\end{algorithm}

\subsection{The objective is low rank in noise space}
\label{sec:lowrank}
The noise-space objective $H_\vw$ is accessible only by evaluating the sampler, so the natural approach is to ascend it from a random $\vz$, but each step requires
\begin{equation}\label{eq:chainrule}
    \nabla_\vz H_\vw(\vz) = \mJ_{\Phi_\theta}(\vz)^{\top}\, \nabla_\vx h_\vw(\vx), \qquad \vx = \Phi_\theta(\vz),
\end{equation}
where $\mJ_{\Phi_\theta}(\vz)\in\R^{D\times D}$ is the Jacobian of the sampler at $\vz$. This gradient differentiates through the ODE solve and must be recomputed for every noise vector and every weight. Prior guided methods avoid this cost by correcting the velocity at each sampling step. Yet \eqref{eq:chainrule} also suggests why exact gradients may be unnecessary. When the data lie near a $k$-dimensional manifold, the learned score at low noise levels is approximately orthogonal to that manifold \citep{stanczuk2024diffusion}, so we expect the sampler to contract the remaining $D-k$ noise directions. $\mJ_{\Phi_\theta}(\vz)^\top$ then has about $k$ dominant singular directions, and every gradient concentrates in a $k$-dimensional subspace. If that subspace is shared across $\vz$, the averaged metric $\mM_\vw = \E[\nabla H_\vw\nabla H_\vw^\top]$ has effective rank about $k$ (Eq.~\ref{eq:effrank}). We find empirical support for this hypothesis. On every benchmark task the estimated spectrum of $\mM_\vw$ concentrates on a few directions (Figure~\ref{fig:rankscaling}). Within the tasks of one benchmark family such as C-10/MOP1--9 (Section~\ref{sec:experiments}), the effective rank changes little as $D$ grows, which suggests that it depends on the intrinsic dimension of the data instead. Differences appear mainly between families, and pooling all 47 tasks, which span nearly four orders of magnitude in $D$, a log-log fit gives $r_{\mathrm{eff}} \propto D^{0.42}$, where $r_{\mathrm{eff}}$ is the effective number of directions steering uses (Eq.~\ref{eq:preff}, \S~\ref{app:rankscaling}). The retained fraction of space therefore shrinks as $D$ grows.
The active directions are a property of the sampler and the proxies rather than of any particular $\vz$, and can be estimated once (Section~\ref{sec:geometry}). A small rank also matters for steering, since displacing along fewer directions moves the noise less for the same gain (Theorem~\ref{thm:ascent}). None of the guarantees below assumes that the spectrum decays. They hold for any rank, with the cost of truncation bounded in Proposition~\ref{prop:optframe}.

\subsection{Noise-space geometry}
\label{sec:geometry}
\textbf{Learning the metric.}
$\mM_\vw$ is the average gradient outer product (AGOP) of $H_\vw$ \citep{radhakrishnan2024rfm}, the second moment of its gradient under the prior, which is prohibitively expensive to form by \eqref{eq:chainrule}, so we estimate it from function values. We decode $n$ noise vectors $\vz_i \sim \gN(\vzero, \mI_D)$ and score the results with the proxies, obtaining pairs $(\vz_i, \hat{\vy}_i)$ with $\hat{\vy}_i = \hat{\vf}(\Phi_\theta(\vz_i))$. 
For each objective $a$, following \citep{radhakrishnan2024rfm} we fit a so-called Recursive Feature Machine (RFM) to these pairs,
a kernel ridge surrogate $\phi_a \approx g_a$ whose Mahalanobis metric $\mM_a$ estimates the AGOP of $g_a$ and is learned by alternating the fit with the AGOP update, 
\begin{equation}\label{eq:agop}
    \mM_a^{(q+1)} \leftarrow \frac{1}{n}\sum_{i=1}^{n} \nabla \phi_a^{(q)}(\vz_i)\, \nabla \phi_a^{(q)}(\vz_i)^{\top}, \qquad \mM_a^{(0)} = \mI_D .
\end{equation}
Kernel machine gradients have a closed form, so the geometry is recovered from input-output pairs alone and $\Phi_\theta$ is never differentiated. The final surrogate gradients $\nabla\phi_a(\vz_i)$ are retained. As shown in \citep{radhakrishnan2024rfm, beaglehole2025xrfm}, RFM is well aligned with the AGOP and also particularly effective when there are only a few samples. 
\begin{figure}[t]
    \centering
    \includegraphics[width=1\linewidth]{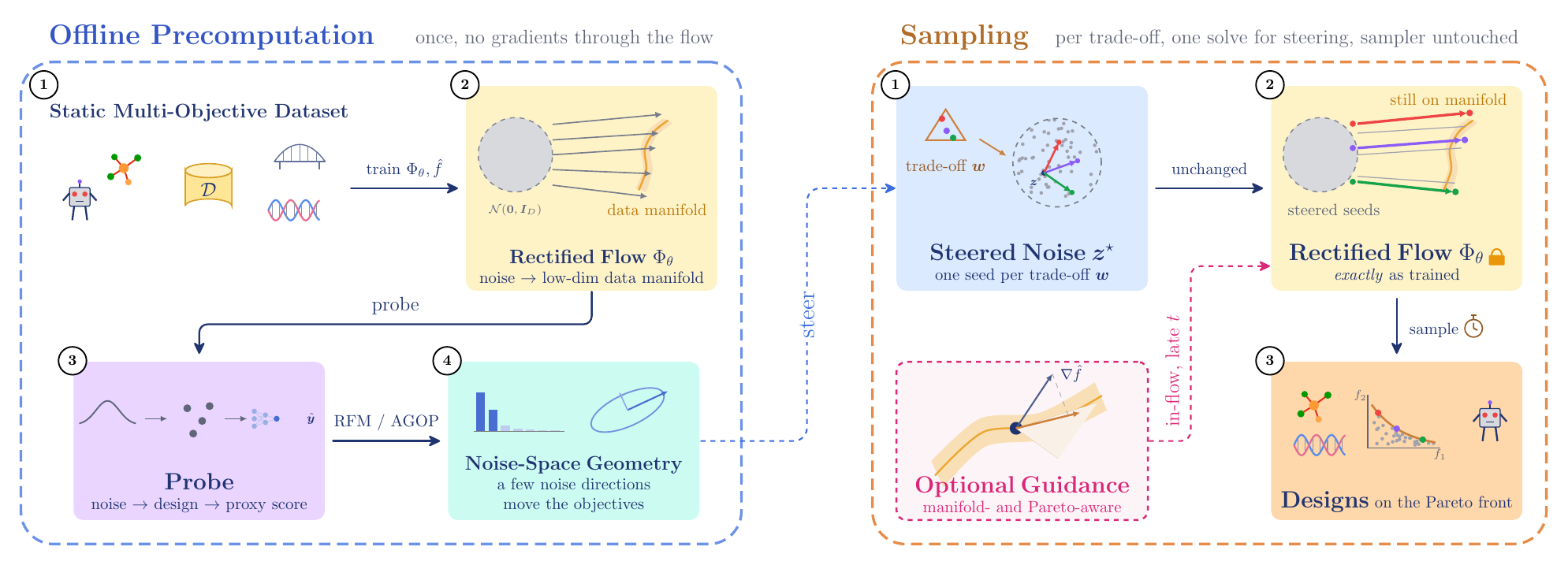}
    \caption{\textbf{Noise-space RFM steering.} \emph{(Left, once per task.)} Random noise is decoded by the flow $\Phi_\theta$ and scored by the proxies $\hat{\vf}$ , and the RFM recovers from these pairs a metric whose few leading directions are the ones that move the objectives. \emph{(Right, per trade off $\vw$)}. A noise vector is displaced along the cached directions for $\vw$ and decoded by the unmodified flow in one ODE solve, with optional late guidance (Section~\ref{sec:guidance}).}
    \label{fig:overview}
\end{figure}

\textbf{One precomputation for all trade off $\vw$'s.}
Let $\mG_i := [\nabla\phi_1(\vz_i) \cdots \nabla\phi_m(\vz_i)] \in \R^{D \times m}$ collect the cached gradients at sample $i$. For any differentiable scalarization $s_\vw:\R^m\to\R$ of the objective values, with $h_\vw=s_\vw\circ\hat\vf$ (the weighted sum \eqref{eq:scalarization} is $s_\vw(\vy)=\vw^\top\vy$), the surrogate $\phi_\vw = s_\vw(\phi_1, \dots, \phi_m)$ has gradient $\nabla\phi_\vw(\vz_i) = \mG_i \vc_i$ with $\vc_i := \nabla s_\vw(\phi_1(\vz_i), \dots, \phi_m(\vz_i))$, so the metric of every trade off is assembled from the cache,
\begin{equation}\label{eq:decomp}
    \widehat{\mM}_\vw := \frac{1}{n}\sum_{i} \nabla\phi_\vw(\vz_i)\nabla\phi_\vw(\vz_i)^\top = \frac{1}{n}\sum_{i} \mG_i \vc_i \vc_i^\top \mG_i^\top \;=\; \sum_{a,b=1}^{m} w_a w_b\, \mC_{ab},
\end{equation}
where the last equality holds for the weighted sum \eqref{eq:scalarization}, under which $\vc_i \equiv \vw$ and $\mC_{ab} := \frac{1}{n}\sum_{i} \nabla\phi_a(\vz_i)\nabla\phi_b(\vz_i)^\top$, with $\mC_{aa}$ the converged iterate of \eqref{eq:agop}. The $m(m{+}1)/2$ distinct blocks ($\mC_{ba}=\mC_{ab}^\top$) are stored once,
and a new weight costs one contraction of them and one eigendecomposition. The directions depend on neither the noise vector nor the trajectory, so one precomputation serves every weight, whereas guidance recomputes its signal at every step for every weight. 

\textbf{Active subspace.}
Compute the eigendecomposition $\mM_\vw = \sum_j \lambda_j^{(\vw)} \vu_\vw^{(j)}\vu_\vw^{(j)\top}$ with $\lambda_1^{(\vw)} \ge \dots \ge \lambda_D^{(\vw)}$. Each eigenvalue is the mean squared derivative of $H_\vw$ along its eigenvector, so the leading eigenvectors are the noise directions the objective responds to most \citep{constantine2015}. We retain the fewest that carry all but an $\varepsilon_{\mathrm{rank}}$ fraction of the total spectrum,
\begin{equation}\label{eq:effrank}
    k_{\mathrm{eff}}(\vw) = \min\Bigl\{ r : \textstyle\sum_{j>r} \lambda_j^{(\vw)} < \varepsilon_{\mathrm{rank}} \sum_{j} \lambda_j^{(\vw)} \Bigr\},
\end{equation}
and the discarded tail is exactly the expected squared gradient outside the retained directions (Proposition~\ref{prop:optframe}). We write $\mU_\vw = [\vu_\vw^{(1)} \cdots \vu_\vw^{(r)}]$ with $r = k_{\mathrm{eff}}(\vw)$ for the retained basis, whose span is the active subspace, and the soft weighting of Section~\ref{sec:steering} keeps every direction and weights it by its eigenvalue instead. The same objects computed from $\widehat{\mM}_\vw$ carry hats. $\widehat{\mM}_\vw$ converges at rate $O(n^{-1/2})$ to the metric of the RFM surrogates, and so do its retained directions under an eigengap, or at $O(n^{-\alpha/2})$ under soft weights. Appendix~\ref{sec:metric-props} bounds the remaining surrogate and proxy errors.

\subsection{Noise-space Steering}
\label{sec:steering}
Given $\vz \sim \gN(\vzero, \mI_D)$ and a weight $\vw$, we displace the noise within the active subspace to a new noise sample $\vz_\vw$, and then decode $\vz_\vw$ through the unmodified ODE to generate a sample $\hat\vx_\vw$: 
\begin{equation}\label{eq:steer}
    \vz_\vw = \vz + \gamma \sum_{j} \beta_j\, \sigma_j\, \widehat{\vu}_\vw^{(j)}, \qquad \hat\vx_\vw = \Phi_\theta(\vz_\vw),
\end{equation}
with \textit{steering strength} $\gamma$ and weights $\beta_j = \mathbf{1}[j \le k_{\mathrm{eff}}(\vw)]$ (hard) or $\beta_j = (\widehat{\lambda}_j^{(\vw)}/\widehat{\lambda}_1^{(\vw)})^{\alpha}$ (soft). Since the metric carries no orientation, $\sigma_j$ is the sign of the Pearson correlation between $\langle \vz_i, \widehat{\vu}_\vw^{(j)}\rangle$ and $\vw^\top\hat{\vy}_i$ over the cache, which by Stein's identity is the population sign of the mean derivative of $H_\vw$ along $\widehat{\vu}_\vw^{(j)}$, the sign Theorem~\ref{thm:ascent} requires. The following theorem states that, averaged over the noise, this displacement indeed increases the scalarized objective. 
\begin{theorem}[Steering ascends the scalarized objective]\label{thm:ascent}
Let $H_\vw$ be smooth with $\|\nabla^2 H_\vw\|_2 \le L_\vw$ (Assumption~\ref{ass:reg}). Then,
\begin{equation}\label{eq:ascent}
    \E_{\vz}\bigl[H_\vw(\vz_\vw) - H_\vw(\vz)\bigr] \;\ge\; \gamma\,\langle \boldsymbol\beta, |\bar{\vd}_\vw| \rangle - \tfrac{1}{2} L_\vw\,\gamma^2\, \|\boldsymbol\beta\|_2^2 ,
\end{equation}
which is positive for $0<\gamma<2\gamma^\star$ and maximal at $\gamma^\star = \langle\boldsymbol\beta,|\bar{\vd}_\vw|\rangle/(L_\vw\|\boldsymbol\beta\|_2^2)$, where it equals $\langle\boldsymbol\beta,|\bar{\vd}_\vw|\rangle^2/(2L_\vw\|\boldsymbol\beta\|_2^2)$. 
Here $\bar{d}_{\vw,j} := \vu_\vw^{(j)\top}\E_{\vz}\nabla H_\vw$, with $\sigma_j = \mathrm{sgn}(\bar{d}_{\vw,j})$ in \eqref{eq:steer}.
\end{theorem}
The gain is linear in $\gamma$ and the curvature penalty quadratic, so a sufficiently small steering strength always ascends (proof in Appendix~\ref{sec:app:ascent}). The subspace enters \eqref{eq:ascent} through $\bar{\vd}_\vw$, the mean gradient it retains, and the next result controls what it misses.

\begin{proposition}[Optimality of the active subspace]\label{prop:optframe}
Among all $r$ dimensional subspaces, $\mathrm{span}(\mU_\vw)$ captures the largest expected squared gradient of $H_\vw$, equal to $\sum_{j\le r}\lambda_j^{(\vw)}$, and the expected squared gradient it leaves out is exactly $\sum_{j>r}\lambda_j^{(\vw)}$.
\end{proposition}
With $\rmP_\vw^\perp := \mI_D - \mU_\vw\mU_\vw^\top$, Jensen's inequality gives $\|\rmP_\vw^\perp\E\nabla H_\vw\|_2^2\le\sum_{j>r}\lambda_j^{(\vw)}$ for the mean gradient outside the active subspace, so the rank rule \eqref{eq:effrank} bounds what Theorem~\ref{thm:ascent} loses to truncation (proof in Appendix~\ref{app:active-subspace}). Figure~\ref{fig:teaser}b is consistent with this. Displacing along $\widehat{\vu}_\vw^{(1)}$ raises the true objective, while a trailing eigenvector or a random direction of equal norm does not, and the gain saturates beyond a task dependent step.

\textbf{Steering beyond the support of the data.}
The displacement has norm $R := \gamma\|\boldsymbol\beta\|_2$. Because steering acts only on the noise, $R$ measures how far the steered noise departs from the region the flow was trained on, and we vary it directly (\S~\ref{app:tier2_strength}). As $R$ grows, hypervolume degrades on ZDT4, saturates on the NAS tasks, and keeps increasing on RE36. Offline benchmarks withhold the top solutions from the training set \citep{xue2024offmoo, trabucco2022designbench}, so a sampling method needs to reach beyond the data, which generative methods do only conservatively \citep{offlinefrontier2026}. Proxies, however, are reliable only near the data \citep{trabucco2021coms, brookes2019cbas, kim2026review}, so their predictions far from its support tend to be inaccurate. Theorem~\ref{thm:recovery} (\S~\ref{sec:recovery}) separates the two effects of $R$. A larger $R$ reaches more of the best front the flow can produce, but the proxies are then evaluated farther from the data, where their error is larger, and which effect dominates depends on the task. Section~\ref{sec:guidance} recovers part of the remaining value without a larger $R$.
\section{Guided Noise-Space Steering}
\label{sec:guidance}
Noise-space-only steering stops improving when the mean gradient leaves the active subspace or when the optimum lies beyond the reach $R$ (Theorem~\ref{thm:recovery}). To recover part of this loss, we add guidance late in the same trajectory, replacing the gradient term in \eqref{eq:guidance} for $t\ge t_{\mathrm{start}}$ with further geometry-adapted projections of the proxy gradient within the RFM+AGOP framework. The RFM geometry is computed once per task, independent of the noise and the trajectory, so each guided step costs one proxy gradient and never differentiates through the ODE solve. Plain guidance follows the scalarized gradient $\vg = \sum_a w_a\nabla\hat f_a$ as is, and we introduce two families of operators that exploit information it discards. \emph{Data-adaptive} operators (RFM-APG, RFM-DAMG) constrain $\vg$ to the geometry of the objective and of the samples, so a step moves only where the objective responds and the flow can follow. \emph{Pareto-aware} operators (RFM-Cone) replace $\vg$ by a direction all proxies agree on, so a step does not decrease any objective to first order.

\textbf{Data adaptive guidance.}
The geometry of Section~\ref{sec:geometry}, rebuilt at any $t$, supplies the first constraint and the candidates sampled in parallel the second. \emph{RFM-APG} repeats the construction once more per task at $t_{\mathrm{start}}$, on intermediate states instead of noise, and projects $\vg$ onto the resulting active subspace, so the step is restricted to directions in which the remaining flow still moves the objective. \emph{RFM-DAMG} reweights $\vg$ by the cached spectrum, $\rmR_\vw = \widehat{\mU}_\vw\operatorname{diag}(\beta_j)\widehat{\mU}_\vw^\top$, and projects the result onto a tangent space estimated by local PCA over the nearest other candidates at the same $t$, which come from nearby preferences. When the local rank is not certified by an eigenvalue threshold the projection is skipped and the operator reduces to $\rmR_\vw\vg$, so RFM-DAMG exploits a manifold where one exists and reduces to plain guidance where neither a manifold nor a spectral gap exists.

\textbf{Pareto aware guidance.}
The per objective proxy gradients are available at every step but scalarized guidance collapses them into one before looking at them, so a step can raise the sum while lowering an objective. \emph{RFM-Cone} uses them jointly. With $\widehat\mJ_u$ the matrix of unit normalized proxy gradients at $\hat\vx_1$ and $\boldsymbol\pi$ the normalized profile of the current proxy scores, it forms the preference direction $\vu_\vw = \widehat\mJ_u(2\vw - \boldsymbol\pi)$, which pulls hardest on the objectives lagging their target weight, and projects it onto the cone $\mathcal{K}(\vx) = \{\vs : \langle\nabla\hat f_a(\vx), \vs\rangle \ge 0\ \forall a\}$, falling back to the MGDA min norm direction \citep{desideri2012mgda} where the projection vanishes. The step therefore follows $\vw$ as far as the objectives permit and never trades one objective for another, with no parameters beyond the onset $t_{\mathrm{start}}$ and strength $\kappa$, where $\eta(t) = \kappa\|\vv_\theta\|/\|\gG_\vw\|$, shared by all operators. Additional details in Appendix~\ref{sec:complementary}.

\begin{theorem}[Per objective improvement]\label{thm:cone}
The RFM-Cone step $\vs$ satisfies $\langle\nabla\hat f_a(\vx),\vs\rangle\ge0$ for every $a$. If the proxy gradients are $L$-Lipschitz, then for a small enough step every objective with $\langle\nabla\hat f_a(\vx),\vs\rangle>0$ increases and the rest change only at second order. The cone contains a strictly ascending direction if and only if $\vx$ is not Pareto stationary, and where the projection vanishes at a stationary point the min-norm fallback is zero and the step reduces to the unmodified flow.
\end{theorem}
Each guided step improves every proxy to first order, which no scalarized guidance can guarantee (proof in \S~\ref{app:thm:cone}).
Steering alone adds no gradient evaluation. Each operator adds one proxy backpropagation per guided step, and RFM-Cone an active-set projection whose cost grows with $m$, so RFM-DAMG runs below ParetoFlow's cost and RFM-Cone trades cost for the best rank (Table~\ref{tab:wallclock}).

\section{Experiments}
\label{sec:experiments}
We evaluate on the $47$ tasks of Off-MOO-Bench \citep{xue2024offmoo} against $14$ representative baselines, comprising the $11$ DNN-based surrogate methods of Off-MOO-Bench and three generative methods. The tasks span $2$ to $10{,}184$ design dimensions, $2$ to $8$ objectives, and continuous, categorical, and mixed-integer design spaces.

\paragraph{Benchmark tasks.} Each task supplies a dataset $\gD$ and an oracle $\vf$ used for evaluation. The $47$ tasks form six families: synthetic DTLZ1--7 and ZDT1--4, ZDT6,\footnote{\citet{paretoflow, pcd} excluded DTLZ2--6 due to a since-corrected evaluation error (\url{https://github.com/lamda-bbo/offline-moo/issues/14}).} neural architecture search C-10/MOP1--9 and IN-1K/MOP1--9 \citep{lu2023nas}, MO-Hopper and MO-Swimmer from multi-objective RL \citep{xu2020pgmorl}, and $15$ RE engineering design problems \citep{tanabe2020re}. Categorical NAS designs are optimized as continuous logits, so $D$ denotes the logit dimension. Following \citet{paretoflow, pcd}, we exclude the combinatorial tasks (see \S~\ref{app:tasks}).

\begin{table}[t]\centering
\caption{Average rank ($\downarrow$) per task family (number of tasks in parentheses), based on 100th-percentile hypervolume (Section~\ref{sec:setup}); mean $\pm$ std over 5 seeds. Best per column in bold, second best underlined. The last two rows count tasks on which RFM-Cone wins, ties, or loses (W/T/L) against PCD and against ParetoFlow$^{\star}$, the better of ParetoFlow and ParetoFlow$+$ on each task, with ties within $5\times10^{-4}$ in hypervolume.}
\label{tab:avg_rank}
\small\setlength{\tabcolsep}{3.5pt}\resizebox{\textwidth}{!}{%
\begin{tabular}{lccccccc}
\toprule
Methods & DTLZ (7) & ZDT (5) & C-10/MOP (9) & IN-1K/MOP (9) & MORL (2) & RE (15) & All Tasks (47) \\ \midrule
$\mathcal{D}$(best) & 12.43 $\pm$ 0.32 & 13.00 $\pm$ 0.45 & 10.58 $\pm$ 0.51 & 6.38 $\pm$ 0.10 & 4.10 $\pm$ 0.89 & 16.80 $\pm$ 0.27 & 12.02 $\pm$ 0.13 \\ \midrule
E2E & \textbf{3.34 $\pm$ 0.61} & \underline{5.16 $\pm$ 0.62} & 12.41 $\pm$ 0.55 & 11.38 $\pm$ 1.06 & 17.40 $\pm$ 1.95 & 8.37 $\pm$ 1.20 & 9.01 $\pm$ 0.59 \\
E2E + GradNorm & 9.71 $\pm$ 0.76 & 13.32 $\pm$ 0.59 & 13.80 $\pm$ 1.24 & 17.60 $\pm$ 0.41 & 7.70 $\pm$ 3.78 & 13.04 $\pm$ 0.67 & 13.37 $\pm$ 0.46 \\
E2E + PcGrad & \underline{4.97 $\pm$ 1.19} & 7.56 $\pm$ 1.82 & 11.94 $\pm$ 0.36 & 12.42 $\pm$ 0.75 & 14.90 $\pm$ 3.45 & \underline{6.05 $\pm$ 0.22} & 8.78 $\pm$ 0.33 \\
MH & 6.00 $\pm$ 0.96 & \textbf{4.88 $\pm$ 1.23} & 12.26 $\pm$ 1.68 & 13.58 $\pm$ 0.38 & 11.20 $\pm$ 3.56 & 6.59 $\pm$ 0.30 & 8.94 $\pm$ 0.43 \\
MH + GradNorm & 9.91 $\pm$ 0.45 & 11.36 $\pm$ 0.55 & 15.56 $\pm$ 1.42 & 16.47 $\pm$ 0.56 & 8.10 $\pm$ 6.18 & 11.37 $\pm$ 0.69 & 12.79 $\pm$ 0.56 \\
MH + PcGrad & 8.03 $\pm$ 0.52 & 10.40 $\pm$ 0.88 & 11.43 $\pm$ 1.00 & 12.87 $\pm$ 0.50 & 14.10 $\pm$ 1.29 & 7.17 $\pm$ 0.40 & 9.84 $\pm$ 0.32 \\
MM & 5.86 $\pm$ 0.45 & 5.60 $\pm$ 0.76 & 10.56 $\pm$ 0.54 & 9.16 $\pm$ 0.96 & 12.30 $\pm$ 4.04 & 8.19 $\pm$ 0.76 & 8.38 $\pm$ 0.36 \\
MM + RoMA & 13.17 $\pm$ 0.54 & 9.64 $\pm$ 1.50 & 11.66 $\pm$ 0.75 & 10.78 $\pm$ 0.92 & 9.40 $\pm$ 1.67 & 10.48 $\pm$ 0.73 & 11.03 $\pm$ 0.26 \\
MM + IOM & 9.49 $\pm$ 0.67 & 10.24 $\pm$ 0.46 & 8.38 $\pm$ 1.02 & 6.51 $\pm$ 0.91 & 11.90 $\pm$ 4.71 & 7.47 $\pm$ 0.47 & 8.24 $\pm$ 0.45 \\
MM + ICT & 10.71 $\pm$ 0.80 & 9.00 $\pm$ 0.28 & 9.70 $\pm$ 0.99 & 8.18 $\pm$ 0.94 & \textbf{3.70 $\pm$ 1.25} & \underline{6.05 $\pm$ 0.81} & \textbf{8.07 $\pm$ 0.54} \\
MM + Tri-Mentor & 11.37 $\pm$ 0.67 & 11.76 $\pm$ 0.48 & 9.13 $\pm$ 0.82 & 11.40 $\pm$ 0.58 & 8.20 $\pm$ 3.46 & \textbf{5.65 $\pm$ 0.65} & 9.03 $\pm$ 0.17 \\
\midrule
PCD & 9.60 $\pm$ 0.31 & 11.08 $\pm$ 1.14 & 7.64 $\pm$ 1.22 & \textbf{5.60 $\pm$ 0.80} & 8.50 $\pm$ 4.23 & 12.95 $\pm$ 0.54 & 9.64 $\pm$ 0.42 \\
ParetoFlow & 15.31 $\pm$ 0.41 & 14.12 $\pm$ 0.41 & 10.60 $\pm$ 0.83 & 11.09 $\pm$ 0.89 & 12.70 $\pm$ 3.07 & 15.40 $\pm$ 0.37 & 13.39 $\pm$ 0.07 \\
ParetoFlow+ & 16.20 $\pm$ 0.13 & 12.16 $\pm$ 0.67 & 10.82 $\pm$ 0.42 & 10.11 $\pm$ 0.72 & 14.60 $\pm$ 2.68 & 16.37 $\pm$ 0.36 & 13.56 $\pm$ 0.07 \\
\midrule
RFM-AN (ours) & 12.40 $\pm$ 0.64 & 11.56 $\pm$ 0.33 & \underline{5.96 $\pm$ 0.66} & 6.74 $\pm$ 1.34 & 8.10 $\pm$ 2.90 & 9.64 $\pm$ 0.49 & 8.93 $\pm$ 0.54 \\
RFM-APG (ours) & 10.77 $\pm$ 0.75 & 10.36 $\pm$ 1.03 & 6.03 $\pm$ 0.52 & 6.62 $\pm$ 1.25 & \underline{5.40 $\pm$ 2.95} & 9.39 $\pm$ 0.26 & 8.36 $\pm$ 0.30 \\
\rowcolor{oursrow}RFM-DAMG (ours) & 10.91 $\pm$ 0.51 & 9.32 $\pm$ 1.25 & 6.13 $\pm$ 0.16 & 6.81 $\pm$ 0.73 & 9.50 $\pm$ 2.85 & 9.38 $\pm$ 0.58 & 8.49 $\pm$ 0.34 \\
\rowcolor{oursrow}RFM-Cone (ours) & 9.80 $\pm$ 0.48 & 9.48 $\pm$ 1.01 & \textbf{5.41 $\pm$ 0.73} & \underline{6.31 $\pm$ 0.91} & 8.20 $\pm$ 2.51 & 9.63 $\pm$ 0.48 & \underline{8.13 $\pm$ 0.22} \\
\midrule
\rowcolor{oursrow}RFM-Cone vs.\ PCD (W/T/L) & 4/0/3 & 3/0/2 & 5/0/4 & 4/0/5 & 1/0/1 & 10/2/3 & 27/2/18 \\
\rowcolor{oursrow}RFM-Cone vs.\ ParetoFlow$^{\star}$ (W/T/L) & 4/0/3 & 3/0/2 & 7/0/2 & 8/0/1 & 2/0/0 & 12/0/3 & 36/0/11 \\
\bottomrule\end{tabular}}\end{table}

\subsection{Experimental Setup}
\label{sec:setup}
We compare three families of methods, separated by rules in Table~\ref{tab:avg_rank}. \emph{DNN-based} methods fit surrogates on $\gD$ and optimize them with NSGA-II \citep{deb2002nsga2}. The surrogate is End-to-End (E2E), Multi-Head (MH), or Multiple Models (MM). E2E and MH are optionally trained with GradNorm or PcGrad, and MM is optionally regularized per objective with RoMA, IOM, ICT, or Tri-Mentoring \citep{chen2018gradnorm, yu2020pcgrad, yu2021roma, qi2022iom, yuan2023ict,
chen2023trimentoring}. \emph{Generative} methods are our primary comparison, since they share our modeling assumption and differ only in where the optimization signal enters. ParetoFlow and ParetoFlow+ guide the flow velocity with a weighted proxy gradient \citep{paretoflow}, where ParetoFlow+ is obtained by replacing ParetoFlow’s flow network and proxy predictor with those used in our RFM variants for a fair comparison. PCD \citep{pcd} conditions a diffusion model on extrapolated targets under classifier-free guidance (CFG). PGD-MOO is excluded (\S~\ref{sec:app:baseline_algo}).

\emph{Ours} comprises \textbf{RFM-AN}, our active-subspace noise steering (\method) 
\eqref{eq:steer} followed by one unmodified ODE solve, and three variants that add a guidance operator $\gG_\vw$ (Algorithm~\ref{alg:steer}, line~\ref{line:guide}): the data-adaptive \textbf{RFM-APG} (active-subspace projected ascent) and \textbf{RFM-DAMG} (manifold-constrained ascent), and the Pareto-aware \textbf{RFM-Cone}. None of them differentiates through the ODE solve. All four share a rectified-flow backbone trained on the full dataset and one geometry precomputation per task (Table~\ref{tab:hyperparams}).

Everything other than where the signal enters is held fixed across the proxy-based methods. Surrogates follow the MM configuration of \citet{xue2024offmoo}, the weighted-sum scalarization is swept over the same Das--Dennis weights \citep{das1998nbi}, and generative methods select the final $256$ candidates by non-dominated sorting under the proxies \citep{paretoflow}. PCD needs no proxies, so we follow its published configuration. All baselines are rerun from released code on identical hardware, seeds, and evaluation (additional details in Appendix~\ref{sec:app:baseline_algo}).

\textbf{Evaluation.}
Following \citet{xue2024offmoo}, every method returns $256$ solutions, scored by hypervolume (HV) against each task's reference point (Table~\ref{tab:refpoints}), which measures optimality and diversity jointly. We report HV of the full set in the main text and of its non-dominated top half in Appendix~\ref{app:hv50}. As a reference level, $\gD(\text{best})$ is the HV of the non-dominated subset of $\gD$, which a method must exceed to generalize past its data. Since HV is not commensurable across tasks, we aggregate by average rank (lower is better). On each task and seed we rank all methods and $\gD(\text{best})$ by HV (1 is best), average the ranks over the tasks of a family, and report mean $\pm$ std over $5$ seeds. Per-task values are in Appendix~\ref{app:hv100}.

\subsection{Results}
\label{sec:results}
Table~\ref{tab:avg_rank} and~\ref{tab:wallclock} support three claims: (i) Noise steering is effective, (ii) additional guidance operator further improves MOO performance, and (iii) our RFM+AGOP approaches are among the cheapest generative methods to sample, and noise steering alone is the fastest of all.

\myparagraph{Noise steering improves generative backbones.}
RFM-Cone attains the best overall rank among generative methods and the second best overall (Table~\ref{tab:avg_rank}), and every RFM variant outranks $\gD(\text{best})$ overall. ParetoFlow+, which applies ParetoFlow's per-step guidance to our flow and proxies, isolates the effect of steering, since it differs from RFM-AN only in where the optimization signal enters. RFM-AN improves the overall rank from $13.56$ to $8.93$, so the gain comes from noise steering rather than the backbone. RFM-AN also outranks PCD, the best prior generative method ($8.93$ vs.\ $9.64$). With guidance, RFM-Cone outperforms ParetoFlow$^{\star}$ in every task family ($36$ of $47$ tasks) and matches or beats PCD at about a third of its sampling cost.

\textbf{Guidance further improves performance.}
Each of our guidance operators further improves on RFM-AN in overall rank. The two data-adaptive operators perform similarly, and the Pareto-aware RFM-Cone improves most, with the largest gains on DTLZ and ZDT, where noise-space-only steering (RFM-AN) is weakest. These are synthetic, box-supported families whose fronts lie far outside the data, and surrogate-based search leads there by the widest margin. This is consistent with the regime Theorem~\ref{thm:recovery} isolates, where the best solutions lie beyond one steering step and late correction on the proxies recovers only part of the gap. On the NAS and MORL families, where the data lie on low-dimensional structure, ours outranks every generative baseline or matches it in mean HV. 

\textbf{Sampling / inference cost.}
RFM-AN is the fastest method on every timed task, roughly four times faster than ParetoFlow and an order of magnitude ($5\times$ to $24\times$) faster than PCD (Table~\ref{tab:wallclock}). On average, RFM-DAMG samples faster than ParetoFlow with a better HV rank, and RFM-Cone trades further cost for the best HV rank among generative methods, so the family spans the accuracy--cost front, and both data-adaptive variants dominate every generative baseline in rank and mean cost.

\textbf{Ablations.}
Appendix~\ref{app:ablations} isolates the remaining components. Displacing along the leading eigenvector raises the true objective while a trailing eigenvector or a random direction of equal norm does not, and the gain saturates at a task-dependent step (\S~\ref{app:tier2_strength} and Figure~\ref{fig:teaser}b). Under soft weighting the retained rank is insensitive across a wide range, unlike the hard cutoff (\S~\ref{app:tier2_rank}--\ref{app:soft_exponent}), and the RFM fit is robust (\S~\ref{app:tier2_agop}). One precomputation serves Tchebycheff scalarizations with the same front quality as the weighted sum (\S~\ref{app:tier3_scala}), and \S~\ref{app:guidance_ablations} studies the guidance strength.

\begin{table}[htb]
\centering
\caption{Sampling wall-clock time in seconds (mean $\pm$ std over 5 seeds) to produce the final 256 candidates, same hardware for all methods (see \S~\ref{app:hardware}). The first row is the one-off per-task RFM precomputation, cached and shared within seeds and variants and excluded from the sampling times. Best per column in bold. }
\label{tab:wallclock}
\setlength{\tabcolsep}{4pt}\resizebox{\textwidth}{!}{%
\begin{tabular}{lcccccc|c}
\toprule
Method & DTLZ2 & ZDT4 & C-10/MOP1 & MO-Hopper & RE21 & RE61 & Mean \\
\midrule
RFM precomputation (one-off, shared by all RFM variants) & 7.6 $\pm$ 0.1 & 7.4 $\pm$ 0.0 & 7.5 $\pm$ 0.1 & 22.7 $\pm$ 0.1 & 7.4 $\pm$ 0.1 & 7.5 $\pm$ 0.0 & 10.0 \\ \midrule
ParetoFlow & 7.4 $\pm$ 0.1 & 8.9 $\pm$ 2.0 & 7.9 $\pm$ 0.1 & 10.8 $\pm$ 0.2 & 8.6 $\pm$ 0.3 & 11.8 $\pm$ 0.2 & 9.2 \\
ParetoFlow+ & 12.4 $\pm$ 0.2 & 12.2 $\pm$ 0.2 & 11.1 $\pm$ 0.2 & 16.2 $\pm$ 0.5 & 13.5 $\pm$ 0.1 & 19.3 $\pm$ 0.1 & 14.1 \\
PCD & 6.8 $\pm$ 1.5 & 7.6 $\pm$ 2.3 & 22.2 $\pm$ 5.3 & 149.4 $\pm$ 30.0 & 22.6 $\pm$ 4.9 & 25.3 $\pm$ 4.2 & 39.0 \\
\midrule
\rowcolor{oursrow}RFM-AN (ours) & \textbf{1.3 $\pm$ 0.0} & \textbf{1.3 $\pm$ 0.1} & \textbf{2.0 $\pm$ 0.0} & \textbf{6.3 $\pm$ 0.2} & \textbf{1.2 $\pm$ 0.4} & \textbf{1.2 $\pm$ 0.0} & \textbf{2.2} \\
\rowcolor{oursrow}RFM-APG (ours) & 3.3 $\pm$ 0.2 & 4.2 $\pm$ 0.4 & 8.2 $\pm$ 0.2 & 31.0 $\pm$ 1.3 & 2.8 $\pm$ 0.1 & 3.7 $\pm$ 0.1 & 8.9 \\
\rowcolor{oursrow}RFM-DAMG (ours) & 3.9 $\pm$ 1.1 & 3.4 $\pm$ 0.0 & 4.7 $\pm$ 0.1 & 27.3 $\pm$ 0.3 & 3.3 $\pm$ 0.9 & 4.6 $\pm$ 1.3 & \underline{7.9} \\
\rowcolor{oursrow}RFM-Cone (ours) & 6.8 $\pm$ 0.2 & 4.6 $\pm$ 0.1 & 6.0 $\pm$ 0.1 & 32.2 $\pm$ 0.9 & 4.2 $\pm$ 0.2 & 20.0 $\pm$ 0.5 & 12.3 \\
\bottomrule\end{tabular}}
\end{table}


\section{Conclusion}
We introduced noise space steering for offline MOO, which learns the geometry of the objectives in the noise space of a rectified flow with a Recursive Feature Machine and displaces the initial noise along a few cached directions, leaving the sampler as trained. One precomputation serves every trade off, each candidate costs one ODE solve, the step provably ascends the scalarized proxy, and two guidance families extend it where a single displacement cannot reach, attaining the best average rank among generative methods, and steering alone samples fastest of all generative approaches. 
Our main limitations are that all guarantees are on the proxies, surrogate based search remains stronger on box supported tasks whose fronts lie far outside the data, and the method needs a deterministic sampler and a continuous representation of the design space (Appendix~\ref{app:limitations}).

\section*{Acknowledgements}
This work is partially supported by the National Science Foundation via grants CCF-2112665, CNS-2622219, and DMS-2502084, as well as by DARPA under grant HR001125CD020. 

\bibliography{iclr2027_conference}

@inproceedings{paretoflow,
  title={ParetoFlow: Guided Flows in Multi-Objective Optimization},
  author={Yuan, Ye and Chen, Can and Pal, Christopher and Liu, Xue},
  booktitle={The Thirteenth International Conference on Learning Representations},
  year={2025}
}

@inproceedings{pgdmoo,
  title={Preference-Guided Diffusion for Multi-Objective Offline Optimization},
  author={Annadani, Yashas and Belakaria, Syrine and Ermon, Stefano and Bauer, Stefan and Engelhardt, Barbara E.},
  booktitle={The Thirty-ninth Annual Conference on Neural Information Processing Systems},
  year={2025}
}

@inproceedings{pcd,
  title={Pareto-Conditioned Diffusion Models for Offline Multi-Objective Optimization},
  author={Shrestha, Jatan and Heiskanen, Santeri and Hepola, Kari and Rissanen, Severi and J{\"a}{\"a}skel{\"a}inen, Pekka and Pajarinen, Joni},
  booktitle={The Fourteenth International Conference on Learning Representations},
  year={2026}
}

@article{radhakrishnan2024rfm,
  title={Mechanism for feature learning in neural networks and backpropagation-free machine learning models},
  author={Radhakrishnan, Adityanarayanan and Beaglehole, Daniel and Pandit, Parthe and Belkin, Mikhail},
  journal={Science},
  volume={383},
  number={6690},
  pages={1461--1467},
  year={2024},
}

@inproceedings{liu2023rectified,
  title={Flow Straight and Fast: Learning to Generate and Transfer Data with Rectified Flow},
  author={Liu, Xingchao and Gong, Chengyue and Liu, Qiang},
  booktitle={The Eleventh International Conference on Learning Representations },
  year={2023}
}

@inproceedings{xue2024offmoo,
  title={Offline Multi-Objective Optimization},
  author={Xue, Ke and Tan, Rongxi and Huang, Xiaobin and Qian, Chao},
  booktitle={Proceedings of the 41st International Conference on Machine Learning},
  year={2024}
}

@article{deb2002nsga2,
  title={A fast and elitist multiobjective genetic algorithm: NSGA-II},
  author={Deb, Kalyanmoy and Pratap, Amrit and Agarwal, Sameer and Meyarivan, T.},
  journal={IEEE Transactions on Evolutionary Computation},
  volume={6},
  number={2},
  pages={182--197},
  year={2002}
}

@inproceedings{trabucco2021coms,
  title={Conservative Objective Models for Effective Offline Model-Based Optimization},
  author={Trabucco, Brandon and Kumar, Aviral and Geng, Xinyang and Levine, Sergey},
  booktitle={Proceedings of the 38th International Conference on Machine Learning},
  year={2021}
}

@inproceedings{yu2021roma,
title={Ro{MA}: Robust Model Adaptation for Offline Model-based Optimization},
author={Sihyun Yu and Sungsoo Ahn and Le Song and Jinwoo Shin},
booktitle={Advances in Neural Information Processing Systems},
year={2021},
}

@inproceedings{yuan2023ict,
 author = {Yuan, Ye and Chen, Can and Liu, Zixuan and Neiswanger, Willie and Liu, Xue},
 booktitle = {Thirty-seventh Conference on Neural Information Processing Systems},
 title = {Importance-aware Co-teaching for Offline Model-based Optimization},
 year = {2023}
}

@inproceedings{qi2022iom,
  author={Han Qi and Yi Su and Aviral Kumar and Sergey Levine},
  title={Data-Driven Offline Decision-Making via Invariant Representation Learning},
  year={2022},
  booktitle={Advances in Neural Information Processing Systems},
}

@inproceedings{chen2023trimentoring,
title={Parallel-mentoring for Offline Model-based Optimization},
author={Can Chen and Christopher Beckham and Zixuan Liu and Xue Liu and Christopher Pal},
booktitle={Thirty-seventh Conference on Neural Information Processing Systems},
year={2023},
}

@inproceedings{kumar2020mins,
 author = {Kumar, Aviral and Levine, Sergey},
 booktitle = {Advances in Neural Information Processing Systems},
 title = {Model Inversion Networks for Model-Based Optimization},
 year = {2020}
}

@Inproceedings{krishnamoorthy2023ddom,
  title = 	 {Diffusion Models for Black-Box Optimization},
  author =       {Krishnamoorthy, Siddarth and Mashkaria, Satvik Mehul and Grover, Aditya},
  booktitle = 	 {Proceedings of the 40th International Conference on Machine Learning},
  year = 	 {2023},
}

@inproceedings{wallace2023doodl,
  title={End-to-End Diffusion Latent Optimization Improves Classifier Guidance},
  author={Bram Wallace and Akash Gokul and Stefano Ermon and Nikhil Naik},
  booktitle={Proceedings of the IEEE/CVF International Conference on Computer Vision (ICCV)},
  year={2023},
}

@inproceedings{
eyring2024reno,
title={Re{NO}: Enhancing One-step Text-to-Image Models through Reward-based Noise Optimization},
author={Luca Eyring and Shyamgopal Karthik and Karsten Roth and Alexey Dosovitskiy and Zeynep Akata},
booktitle={The Thirty-eighth Annual Conference on Neural Information Processing Systems},
year={2024},
}

@article{zhou2024goldennoise,
  title={Golden Noise for Diffusion Models: A Learning Framework},
  author={Zikai Zhou and Shitong Shao and Lichen Bai and Shufei Zhang and Zhiqiang Xu and Bo Han and Zeke Xie},
  journal={Proceedings of the IEEE/CVF International Conference on Computer Vision (ICCV)},
  year={2025},
  pages={17688-17697},
}

@InProceedings{ma2025inference,
    author    = {Ma, Nanye and Tong, Shangyuan and Jia, Haolin and Hu, Hexiang and Su, Yu-Chuan and Zhang, Mingda and Yang, Xuan and Li, Yandong and Jaakkola, Tommi and Jia, Xuhui and Xie, Saining},
    title     = {Scaling Inference Time Compute for Diffusion Models},
    booktitle = {Proceedings of the IEEE/CVF Conference on Computer Vision and Pattern Recognition (CVPR)},
    month     = {June},
    year      = {2025},
    pages     = {2523-2534}
}

@book{constantine2015,
author = {Constantine, Paul G.},
title = {Active Subspaces: Emerging Ideas for Dimension Reduction in Parameter Studies},
year = {2015},
isbn = {1611973856},
publisher = {Society for Industrial and Applied Mathematics},
address = {USA}
}

@inproceedings{wang2026narfm,
title={General and Efficient Steering of Unconditional Diffusion Models},
author={Qingsong Wang and Mikhail Belkin and Yusu Wang},
booktitle={Proceedings of the 43rd International Conference on Machine Learning},
year={2026},
}

@inproceedings{wang2025seeds,
title={Seeds of Structure: Patch {PCA} Reveals Universal Compositional Cues in Diffusion Models},
author={Qingsong Wang and Zhengchao Wan and Mikhail Belkin and Yusu Wang},
booktitle={The Thirty-ninth Annual Conference on Neural Information Processing Systems},
year={2025},
}

@inproceedings{stanczuk2024diffusion,
title={Diffusion Models Encode the Intrinsic Dimension of Data Manifolds},
author={Jan Pawel Stanczuk and Georgios Batzolis and Teo Deveney and Carola-Bibiane Sch{\"o}nlieb},
booktitle={Proceedings of the 41st International Conference on Machine Learning},
year={2024},
}

@inproceedings{hotegni2026spread,
title={{SPREAD}: Sampling-based Pareto front Refinement via Efficient Adaptive Diffusion},
author={Sedjro Salomon Hotegni and Sebastian Peitz},
booktitle={The Fourteenth International Conference on Learning Representations},
year={2026},
}

@inproceedings{offlinefrontier2026,
title={The Offline-Frontier Shift: Diagnosing Distributional Limits in Generative Multi-Objective Optimization},
author={Stephanie Holly and Alexandru-Ciprian Zavoianu and Siegfried Silber and Sepp Hochreiter and Werner Zellinger},
booktitle={Workshop on Scientific Methods for Understanding Deep Learning},
year={2026},
url={https://openreview.net/forum?id=4XN7MqWVKD}
}

@InProceedings{trabucco2022designbench,
  title = 	 {Design-Bench: Benchmarks for Data-Driven Offline Model-Based Optimization},
  author =       {Trabucco, Brandon and Geng, Xinyang and Kumar, Aviral and Levine, Sergey},
  booktitle = 	 {Proceedings of the 39th International Conference on Machine Learning},
  year = 	 {2022},
}

@article{desideri2012mgda,
title = {Multiple-gradient descent algorithm (MGDA) for multiobjective optimization},
journal = {Comptes Rendus Mathematique},
volume = {350},
number = {5},
pages = {313-318},
year = {2012},
issn = {1631-073X},
doi = {https://doi.org/10.1016/j.crma.2012.03.014},
url = {https://www.sciencedirect.com/science/article/pii/S1631073X12000738},
author = {Jean-Antoine Désidéri}
}

@InProceedings{xu2020pgmorl,
  title = 	 {Prediction-Guided Multi-Objective Reinforcement Learning for Continuous Robot Control},
  author =       {Xu, Jie and Tian, Yunsheng and Ma, Pingchuan and Rus, Daniela and Sueda, Shinjiro and Matusik, Wojciech},
  booktitle = 	 {Proceedings of the 37th International Conference on Machine Learning},
  year = 	 {2020},
}

@InProceedings{chen2018gradnorm,
  title = 	 {{G}rad{N}orm: Gradient Normalization for Adaptive Loss Balancing in Deep Multitask Networks},
  author =       {Chen, Zhao and Badrinarayanan, Vijay and Lee, Chen-Yu and Rabinovich, Andrew},
  booktitle = 	 {Proceedings of the 35th International Conference on Machine Learning},
  year = 	 {2018},
}

@article{das1998nbi,
author = {Das, Indraneel and Dennis, J. E.},
title = {Normal-Boundary Intersection: A New Method for Generating the Pareto Surface in Nonlinear Multicriteria Optimization Problems},
journal = {SIAM Journal on Optimization},
volume = {8},
number = {3},
pages = {631-657},
year = {1998},
doi = {10.1137/S1052623496307510},

URL = { 
    
        https://doi.org/10.1137/S1052623496307510
    
    

},
eprint = { 
    
        https://doi.org/10.1137/S1052623496307510
    
    

}
}

@InProceedings{brookes2019cbas,
  title = 	 {Conditioning by adaptive sampling for robust design},
  author =       {Brookes, David and Park, Hahnbeom and Listgarten, Jennifer},
  booktitle = 	 {Proceedings of the 36th International Conference on Machine Learning},
  year = 	 {2019},
}

@article{kim2026review,
title={Offline Model-Based Optimization: Comprehensive Review},
author={Minsu Kim and Jiayao Gu and Ye Yuan and Taeyoung Yun and Zixuan Liu and Yoshua Bengio and Can Chen},
journal={Transactions on Machine Learning Research},
year={2026},
}

@inproceedings{yu2020pcgrad,
author = {Yu, Tianhe and Kumar, Saurabh and Gupta, Abhishek and Levine, Sergey and Hausman, Karol and Finn, Chelsea},
title = {Gradient surgery for multi-task learning},
year = {2020},
booktitle = {Advances in Neural Information Processing Systems},
}

@article{tanabe2020re,
title = {An easy-to-use real-world multi-objective optimization problem suite},
journal = {Applied Soft Computing},
volume = {89},
pages = {106078},
year = {2020},
issn = {1568-4946},
doi = {https://doi.org/10.1016/j.asoc.2020.106078},
url = {https://www.sciencedirect.com/science/article/pii/S1568494620300181},
author = {Ryoji Tanabe and Hisao Ishibuchi}
}

@article{lu2023nas,
  author={Lu, Zhichao and Cheng, Ran and Jin, Yaochu and Tan, Kay Chen and Deb, Kalyanmoy},
  journal={IEEE Transactions on Evolutionary Computation}, 
  title={Neural Architecture Search as Multiobjective Optimization Benchmarks: Problem Formulation and Performance Assessment}, 
  year={2024},
  volume={28},
  number={2},
  pages={323-337},
  doi={10.1109/TEVC.2022.3233364}}

@article{yu2015dk,
  title={A useful variant of the Davis--Kahan theorem for statisticians},
  author={Yi Yu and Tengyao Wang and Richard J. Samworth},
  journal={Biometrika},
  number = {2},
  year={2015},
  volume={102},
  pages={315-323},
  issn = {0006-3444},
  doi = {10.1093/biomet/asv008},
  url = {https://doi.org/10.1093/biomet/asv008},
  eprint = {https://academic.oup.com/biomet/article-pdf/102/2/315/9642505/asv008.pdf},
}

@article{vial1983,
 ISSN = {0364765X, 15265471},
 URL = {http://www.jstor.org/stable/3689591},
 author = {Jean-Philippe Vial},
 journal = {Mathematics of Operations Research},
 number = {2},
 pages = {231--259},
 publisher = {INFORMS},
 title = {Strong and Weak Convexity of Sets and Functions},
 urldate = {2026-09-25},
 volume = {8},
 year = {1983}
}

@inproceedings{
ho2021classifierfree,
title={Classifier-Free Diffusion Guidance},
author={Jonathan Ho and Tim Salimans},
booktitle={NeurIPS 2021 Workshop on Deep Generative Models and Downstream Applications},
year={2021},
url={https://openreview.net/forum?id=qw8AKxfYbI}
}

@article{beaglehole2026steering,
author = {Daniel Beaglehole  and Adityanarayanan Radhakrishnan  and Enric Boix-Adserà  and Mikhail Belkin },
title = {Toward universal steering and monitoring of AI models},
journal = {Science},
volume = {391},
number = {6787},
pages = {787-792},
year = {2026},
}

@inproceedings{
beaglehole2025xrfm,
title={x{RFM}: Accurate, scalable, and interpretable feature learning models for tabular data},
author={Daniel Beaglehole and David Holzm{\"u}ller and Adityanarayanan Radhakrishnan and Mikhail Belkin},
booktitle={The Fourteenth International Conference on Learning Representations},
year={2026},
}

@article{ando1988,
  title   = {Comparison of norms {$|\!|\!|f(A)-f(B)|\!|\!|$} and {$|\!|\!|f(|A-B|)|\!|\!|$}},
  author  = {And{\^o}, Tsuyoshi},
  journal = {Mathematische Zeitschrift},
  year    = {1988},
  volume  = {197},
  pages   = {403--409},
  url     = {https://api.semanticscholar.org/CorpusID:122492348}
}

@inproceedings{trivedi2014egop,
author = {Trivedi, Shubhendu and Wang, Jialei and Kpotufe, Samory and Shakhnarovich, Gregory},
title = {A consistent estimator of the expected gradient outerproduct},
year = {2014},
isbn = {9780974903910},
publisher = {AUAI Press},
address = {Arlington, Virginia, USA},
booktitle = {Proceedings of the Thirtieth Conference on Uncertainty in Artificial Intelligence},
pages = {819–828},
numpages = {10},
location = {Quebec City, Quebec, Canada},
series = {UAI'14}
}

@article{radhakrishnan2025linrfm,
author = {Adityanarayanan Radhakrishnan  and Mikhail Belkin  and Dmitriy Drusvyatskiy },
title = {Linear Recursive Feature Machines provably recover low-rank matrices},
journal = {Proceedings of the National Academy of Sciences},
volume = {122},
number = {13},
pages = {e2411325122},
year = {2025},
doi = {10.1073/pnas.2411325122},
URL = {https://www.pnas.org/doi/abs/10.1073/pnas.2411325122},
eprint = {https://www.pnas.org/doi/pdf/10.1073/pnas.2411325122}}

\clearpage
\appendix

\etocdepthtag.toc{mtappendix}
\etocsettagdepth{mtmain}{none}
\etocsettagdepth{mtappendix}{subsubsection}

\section*{Appendix}
\hrule height 1pt \vspace{1em}
{\itshape\etocsettocstyle{}{}\tableofcontents}
\vspace{1em}\hrule height 1pt \vspace{1em}

\section{Additional Experimental details}
\label{sec:app:add_exps_details}

\subsection{Benchmark Tasks}
\label{app:tasks}
Each Off-MOO-Bench task supplies a fixed dataset $\gD$ and an oracle $\vf$ reserved for evaluation; Table~\ref{tab:refpoints} lists the dimension, number of objectives, and dataset size of every task.

\textbf{Synthetic.} DTLZ1--7 (three objectives) and ZDT1--4, ZDT6 (two objectives) are continuous functions with $D\le30$ and $60{,}000$ offline designs each, and their fronts are known in closed form. \citet{paretoflow, pcd} excluded DTLZ2--6 due to an evaluation error in the benchmark that has since been corrected, so we include them.

\textbf{Neural architecture search.} C-10/MOP1--9 and IN-1K/MOP1--9 are the NAS benchmarks of \citet{lu2023nas} with $2$ to $8$ objectives. The categorical design space (up to $34$ variables) is optimized as continuous logits and decoded by argmax before evaluation, so $D$ is the logit dimension rather than the variable count \citep{trabucco2022designbench, xue2024offmoo}.

\textbf{MORL} MO-Hopper and MO-Swimmer are two-objective continuous control problems in which a design is a full policy parameter vector ($D=10{,}184$ and $9{,}734$). The data are policies collected during PG-MORL training \citep{xu2020pgmorl}, and the oracle is a MuJoCo rollout.

\textbf{Real-world engineering.} RE21--25, RE31--37, RE41--42, and RE61 are engineering design problems \citep{tanabe2020re} with $2$ to $7$ variables and $2$, $3$, $4$, or $6$ objectives. Several have integer variables, and none has a closed-form front.

Following \citet{paretoflow, pcd}, we exclude the combinatorial tasks (MO-TSP, MO-CVRP, MO-KP), which would require discrete generative models.

\subsection{Evaluation Methodology}
\label{sec:app:eval_methodology}
\paragraph{Hypervolume.}
Let $\gS \subset \gX$ be a solution set and $F = \vf(\gS) \subset \R^m$ its image under the oracle. For a reference point $\vr \in \R^m$ the hypervolume indicator is the Lebesgue measure of the region dominated by $F$ and bounded by $\vr$,
\begin{equation}\label{eq:hv}
    \mathrm{HV}_\vr(F) = \mathrm{vol}_m\Bigl(\,\bigcup_{\vy \in F}
    \bigl\{\, \vp \in \R^m \,\mid\, \vr \preceq \vp \preceq \vy \,\bigr\}\Bigr),
\end{equation}
where $\preceq$ is weak dominance under the maximization convention of Section~\ref{sec:prelim}, so $\vr$ lies below every attainable objective vector. Hypervolume is the only indicator that is strictly monotone with respect to Pareto dominance, and it rewards both proximity to $\gF^\star$ and spread across it, which is why we report it alone rather than a separate diversity measure. Objectives are min--max normalized using the benchmark's stored per-objective bounds and sign-reversed to match our maximization convention before \eqref{eq:hv} is evaluated. We compute hypervolume using the exact \texttt{pymoo} implementation for all objective counts.

\paragraph{Percentile protocol.}
Every method returns exactly $\lvert \gS \rvert = 256$ solutions. To report the $P$-th percentile hypervolume we apply non-dominated sorting to $\vf(\gS)$, which partitions $\gS$ into fronts
$F_1 \succ F_2 \succ \cdots$, order solutions by front index with ties broken by crowding distance, discard the last $(100 - P)\%$ of the ordering, and evaluate \eqref{eq:hv} on the remaining
$\lfloor 256 P / 100 \rfloor$ solutions. $P = 100$ therefore scores the full returned set and $P = 50$ scores its better half, which is less sensitive to a small number of strong outliers. We report $P = 100$ in the main text and $P = 50$ in Appendix~\ref{app:hv50}.

\paragraph{Dataset baseline.} $\gD(\text{best})$ is the hypervolume of the non-dominated subset of the
offline dataset, computed with the same $\vr$ after selecting its top $256$ solutions by non-dominated sorting, and reported at the 100th percentile only. It has no seed variance and is the level any offline method must exceed to have generalized beyond its data.

\paragraph{Reference and Ideal points.} Table~\ref{tab:refpoints} lists, for every task, the design variables $d$,the number of objectives $m$, the size of the offline dataset $\lvert \gD \rvert$, the type of design space, and the reference point $\vr \in \R^m$ against which hypervolume is computed. Reference points follow Off-MOO-Bench \citep{xue2024offmoo}, using the task-specific nadirs from the benchmark snapshot bundled with our code. Under our maximization convention, the reference point is
$\vr=-2.2\,\mathcal{N}(\mathbf{n})$, where $\mathbf{n}$ is the benchmark nadir and $\mathcal{N}$ is the same min--max normalization applied to the objectives. Changes to task-specific nadirs across benchmark versions can account for differences in reported HV and $\gD(\text{best})$ from PCD and ParetoFlow papers \citep{pcd,paretoflow}.

\paragraph{Oracle access and decoding.}
The oracle $\vf$ is called once per method, seed, and task, on the final $256$ solutions only, and never during training, hyperparameter selection, or sampling. NAS logits are decoded to a discrete architecture by taking the argmax over categories for each variable. For some of the RE tasks, the benchmark oracle applies task-specific rounding or mapping to allowed discrete values during evaluation. So hypervolume is computed on feasible solutions rather than on continuous outputs.

\paragraph{Statistical reporting.}
Each method is run with $5$ seeds per task, covering proxy and flow initialization, the noise draws, and any sampler stochasticity. The per-task tables (Appendix~\ref{app:hv100}) report mean $\pm$ sample standard deviation of hypervolume over the five seeds; methods within one standard deviation of the best, or within $5\times10^{-4}$ of it, are bolded. For Table~\ref{tab:avg_rank} we rank all methods and $\gD(\text{best})$ on each task separately for each seed, average the ranks over the tasks of a family, and report mean $\pm$ sample standard deviation of that family average over the five seeds, so the $\pm$ there is seed variability of the average rank, not dispersion across tasks. The W/T/L rows compare seed-averaged hypervolume per task, counting a tie when the gap is below $5\times10^{-4}$; no significance test is applied. Wall-clock times (Table~\ref{tab:wallclock}) are mean $\pm$ standard deviation over the same five seeds.

\begin{table}[h]
\centering
\caption{Task metadata and hypervolume reference points for all $47$ tasks. Reference points are in normalized maximization coordinates (negated from the evaluator's minimization convention). $d$ is the number of design variables (categorical variables for NAS, whose logits have dimension $D$), $\lvert\gD\rvert$ the offline dataset size, and \texttt{Cont.}/\texttt{Cat.} continuous/categorical.}
\label{tab:refpoints}
\tiny
\setlength{\tabcolsep}{4pt}
\renewcommand{\arraystretch}{1.05}
\begin{tabular}{@{}l r r r l l@{}}
\toprule
\vspace{1pt}
Task & $d$ & $m$ & $\lvert \gD \rvert$ & Space & Reference point $\vr$ \\
\midrule
\multicolumn{6}{@{}l}{\emph{DTLZ}} \\
DTLZ1 & 7 & 3 & 60{,}000 & Cont. & $(-2.2,\ -2.2,\ -2.2)$ \\
DTLZ2 & 10 & 3 & 60{,}000 & Cont. & $(-2.07826,\ -2.44443,\ -2.45812)$ \\
DTLZ3 & 10 & 3 & 60{,}000 & Cont. & $(-2.22594,\ -2.07323,\ -2.1453)$ \\
DTLZ4 & 10 & 3 & 60{,}000 & Cont. & $(-2.29461,\ -2.91967,\ -2.63431)$ \\
DTLZ5 & 10 & 3 & 60{,}000 & Cont. & $(-2.74136,\ -2.99806,\ -2.6272)$ \\
DTLZ6 & 10 & 3 & 60{,}000 & Cont. & $(-2.5467,\ -2.53751,\ -2.46451)$ \\
DTLZ7 & 10 & 3 & 60{,}000 & Cont. & $(-2.20001,\ -2.20001,\ -2.2)$ \\
\midrule
\multicolumn{6}{@{}l}{\emph{ZDT}} \\
ZDT1 & 30 & 2 & 60{,}000 & Cont. & $(-2.2,\ -2.2)$ \\
ZDT2 & 30 & 2 & 60{,}000 & Cont. & $(-2.2,\ -2.46332)$ \\
ZDT3 & 30 & 2 & 60{,}000 & Cont. & $(-2.2,\ -2.70741)$ \\
ZDT4 & 10 & 2 & 60{,}000 & Cont. & $(-2.20018,\ -2.52935)$ \\
ZDT6 & 10 & 2 & 60{,}000 & Cont. & $(-2.2,\ -2.21336)$ \\
\midrule
\multicolumn{6}{@{}l}{\emph{C-10/MOP}} \\
C-10/MOP1 & 26 & 2 & 12{,}084 & Cat. & $(-0.700653,\ -2.27345)$ \\
C-10/MOP2 & 26 & 3 & 26{,}316 & Cat. & $(-0.270857,\ -2.33973,\ -2.3436)$ \\
C-10/MOP3 & 5 & 3 & 19{,}661 & Cat. & $(-2.2,\ -2.2,\ -2.2)$ \\
C-10/MOP4 & 5 & 4 & 19{,}661 & Cat. & $(-2.2,\ -2.2,\ -2.2,\ -2.2)$ \\
C-10/MOP5 & 6 & 5 & 9{,}375 & Cat. & $(-2.2,\ \ldots,\ -2.2)$ \\
C-10/MOP6 & 6 & 6 & 9{,}375 & Cat. & $(-2.2,\ \ldots,\ -2.2)$ \\
C-10/MOP7 & 6 & 8 & 9{,}375 & Cat. & $(-2.2,\ \ldots,\ -2.2)$ \\
C-10/MOP8 & 32 & 2 & 60{,}000 & Cat. & $(-2.47149,\ -2.42687)$ \\
C-10/MOP9 & 32 & 3 & 60{,}000 & Cat. & $(-2.46312,\ -2.57324,\ -2.5169)$ \\
\midrule
\multicolumn{6}{@{}l}{\emph{IN-1K/MOP}} \\
IN-1K/MOP1 & 25 & 2 & 60{,}000 & Cat. & $(-2.85067,\ -2.46423)$ \\
IN-1K/MOP2 & 25 & 2 & 60{,}000 & Cat. & $(-2.8185,\ -2.77927)$ \\
IN-1K/MOP3 & 25 & 3 & 60{,}000 & Cat. & $(-2.82779,\ -2.57366,\ -2.5582)$ \\
IN-1K/MOP4 & 34 & 2 & 60{,}000 & Cat. & $(-2.13485,\ -2.31778)$ \\
IN-1K/MOP5 & 34 & 2 & 60{,}000 & Cat. & $(-2.14702,\ -2.39576)$ \\
IN-1K/MOP6 & 34 & 3 & 60{,}000 & Cat. & $(-2.21115,\ -2.43552,\ -2.44032)$ \\
IN-1K/MOP7 & 21 & 2 & 60{,}000 & Cat. & $(-2.97769,\ -2.29048)$ \\
IN-1K/MOP8 & 21 & 3 & 60{,}000 & Cat. & $(-2.95398,\ -2.26424,\ -2.27795)$ \\
IN-1K/MOP9 & 21 & 4 & 60{,}000 & Cat. & $(-2.2,\ -2.2,\ -2.2,\ -2.2)$ \\
\midrule
\multicolumn{6}{@{}l}{\emph{MORL}} \\
MO-Hopper & 10{,}184 & 2 & 4{,}500 & Cont. & $(-2.45514,\ -3.41206)$ \\
MO-Swimmer & 9{,}734 & 2 & 8{,}571 & Cont. & $(-2.40532,\ -1.76713)$ \\
\midrule
\multicolumn{6}{@{}l}{\emph{RE}} \\
RE21 & 4 & 2 & 60{,}000 & Cont. & $(-2.2,\ -2.2)$ \\
RE22 & 3 & 2 & 60{,}000 & Cont. & $(-2.2,\ -2.2)$ \\
RE23 & 4 & 2 & 60{,}000 & Cont. & $(-2.2,\ -2.20012)$ \\
RE24 & 2 & 2 & 60{,}000 & Cont. & $(-2.2,\ -2.2)$ \\
RE25 & 3 & 2 & 60{,}000 & Cont. & $(-2.2,\ -2.2)$ \\
RE31 & 3 & 3 & 60{,}000 & Cont. & $(-2.2,\ -2.2,\ -2.2)$ \\
RE32 & 4 & 3 & 60{,}000 & Cont. & $(-2.2,\ -2.2,\ -2.2)$ \\
RE33 & 4 & 3 & 60{,}000 & Cont. & $(-2.2,\ -2.2,\ -2.2)$ \\
RE34 & 5 & 3 & 60{,}000 & Cont. & $(-2.2,\ -2.2,\ -2.2)$ \\
RE35 & 7 & 3 & 60{,}000 & Cont. & $(-2.2,\ -2.2,\ -2.2)$ \\
RE36 & 4 & 3 & 60{,}000 & Cont. & $(-2.2,\ -2.2,\ -2.2)$ \\
RE37 & 4 & 3 & 60{,}000 & Cont. & $(-2.17549,\ -1.69879,\ -2.04536)$ \\
RE41 & 7 & 4 & 60{,}000 & Cont. & $(-2.2,\ -2.2,\ -2.2,\ -2.2)$ \\
RE42 & 6 & 4 & 60{,}000 & Cont. & $(-2.2,\ -2.2,\ -2.2,\ -2.2)$ \\
RE61 & 3 & 6 & 60{,}000 & Cont. & $(-2.2,\ \ldots,\ -2.2)$ \\
\bottomrule
\end{tabular}
\end{table}

\subsection{Hyperparameter and Tuning used for reporting results}
\label{app:hyperparams}
All reported RFM results use the configuration of Table~\ref{tab:hyperparams}. These defaults were fixed before the sensitivity studies of Appendix~\ref{app:tier2} were run and are not tuned per task. Those studies test robustness around the defaults rather than select them. The MORL tasks differ only where their dimension requires it. At $D\approx10^4$ the RFM precomputation is compute-bound, so we fit it with fewer samples and AGOP iterations and retain a basis of $100$ directions, and the flow is trained for more epochs.

\begin{table}[H]\centering
\caption{Hyperparameters for flow and proxy training (left) and noise-space steering and sampling (right), shared by all tasks. MORL values are given in parentheses where they differ.}
\label{tab:hyperparams}
\small\setlength{\tabcolsep}{5pt}\renewcommand{\arraystretch}{1.08}
\resizebox{\textwidth}{!}{%
\begin{tabular}{@{}lll|ll@{}}
\toprule
\multicolumn{3}{c|}{\textbf{Flow and Proxy Training}} & \multicolumn{2}{c}{\textbf{Noise-Space Steering and Sampling}} \\ \midrule
\textbf{Parameter} & \textbf{Flow $\vv_\theta$} & \textbf{Proxy $\hat f_a$} & \textbf{Parameter} & \textbf{Value(s)} \\ \midrule
Architecture      & Residual MLP        & MLP                 & Kernel                               & Laplace (Mahalanobis) \\
Depth             & 4 residual blocks   & 2 hidden layers     & Bandwidth $\ell$                     & 20 \\
Width             & 512                 & 2048                & Ridge $\lambda$                      & $10^{-3}$ \\
Activation        & SiLU                & LeakyReLU           & RFM samples $n$                      & 10{,}000 (5{,}000) \\
Normalization     & LayerNorm           & ---                 & AGOP iterations $T$                  & 7 (4) \\
Optimizer         & AdamW               & Adam                & Retained directions                  & $D$ (100) \\
Learning Rate     & $2\times10^{-4}$    & $10^{-3}$           & Steering Strength $\gamma$                   & 10 \\
LR Schedule       & Cosine Annealing    & $\times0.995$/epoch & Soft exponent $\alpha$               & 0.3 \\
Batch Size        & 128                 & 128                 & Weights $\vw$ $\times$ draws         & $256 \times 2$ (Das--Dennis) \\
Epochs            & 1{,}000 (3{,}000)   & $\le$5{,}000 (early stopping) & Returned set $|\gS|$       & 256 \\
Gradient Clipping & 1.0                 & ---                 & Euler steps                          & 1{,}000 \\
Checkpoint        & EMA 0.999           & best val.\ MSE      & Seeds                                & 0--4 \\
\midrule
\multicolumn{5}{c}{\textbf{RFM Variants} (all use soft steering, $\beta_j=(\widehat\lambda_j^{(\vw)}/\widehat\lambda_1^{(\vw)})^{\alpha}$)} \\ \midrule
RFM-AN   & \multicolumn{2}{l|}{unguided decoding}                                      & RFM-DAMG & $t_{\mathrm{start}}=0.7$, $\kappa=0.5$, $\eta\le11$, kNN $k=12$, $\tau_{\mathrm{pca}}=0.2$ \\
RFM-APG  & \multicolumn{2}{l|}{$t_{\mathrm{start}}=0.8$, $\eta(t)=\eta_{\max}\tfrac{t-0.8}{0.2}$, for $t\ge0.8$, $\eta_{\max}=11$} & RFM-Cone & $t_{\mathrm{start}}=0.7$, $\kappa=0.5$, uncapped norm matching; min-norm fallback \\
\bottomrule
\end{tabular}}
\end{table}

\subsection{Hardware and Compute}
\label{app:hardware}
Each experiment runs on a single NVIDIA RTX A6000 GPU
(48\,GB VRAM), on a server with 64 physical CPU cores and
approximately 1\,TiB of RAM, using PyTorch 2.0.1 and
CUDA 11.8. No method is distributed across devices.
Baselines use their released implementations with the
sampling configurations described above.

We compare wall-clock sampling times on the same hardware
with identical CPU thread limits and an output budget of
256 solutions. For ParetoFlow and PCD, we follow the sampling configurations and hyperparameters reported in their respective papers \citep{paretoflow,pcd}. We report the mean and standard deviation over five seeds, synchronizing GPU computation at the timing boundaries. These measurements exclude model training and oracle evaluation.

RFM additionally requires an initial geometry precomputation, performed once per task and trained-model seed and reused across preference weights. This takes 7.4--22.7 seconds on the six timed tasks and is reported separately from sampling time in Table~\ref{tab:wallclock}.

The full set of experiments, including all seeds, baselines,
and ablations, took approximately 550 GPU-hours, of which
300 were spent on baselines and 250 on our method.

\subsection{Training Details}
\label{sec:app:training_details}
For each task and seed, we train a rectified-flow model and objective-specific proxy predictors on the offline dataset. All RFM variants share these trained models and differ only in their steering and guidance operators. ParetoFlow+ uses the same flow and proxies while retaining ParetoFlow's sampling scheme, allowing us to separate the contribution of our method from that of the backbone and predictors. For PCD, we follow the training and sampling procedures described in its paper and released implementation, with the boundary handling described below.

We ensure that input normalization is consistent between training and inference and that generated designs are mapped back to the original design space using the corresponding inverse transformation. We correct preprocessing inconsistencies in the baseline pipelines and apply task-specific bound handling where needed. In particular, PCD can generate continuous designs outside the prescribed box constraints, so we clip these designs to the task-specific bounds before evaluation \footnote{\scriptsize \citet{pcd} do not explicitly handle this in their code: \url{https://github.com/jatan12/PCD/}}. Discrete designs are decoded using the benchmark's decoding rules.

\subsection{Baseline Algorithm Details}
\label{sec:app:baseline_algo}
All RFM variants and ParetoFlow+ share the same trained rectified-flow
backbone and objective proxies. Both our flow network and PCD's
denoiser use residual MLPs, but PCD includes an additional
input/condition \citep{pcd}. Table~\ref{tab:network_architecture}
summarizes the architectures. Each of our objective proxy is an independent MLP with two hidden
layers of width 2048 and LeakyReLU activations, similar to ParetoFlow \citep{paretoflow}.

\begin{table}[H]
    \centering
    \small
    \setlength{\tabcolsep}{5pt}
    \caption{Architecture comparison with the released PCD
    implementation. Main-path layer counts exclude the separate
    time/noise embedding MLP.}
    \label{tab:network_architecture}
    \begin{tabular}{@{}lll@{}}
        \toprule
        Component & RFM / ParetoFlow+ & PCD \\
        \midrule
        Generative model & Rectified flow & EDM diffusion \\
        Backbone & Residual MLP & Residual MLP \\
        Hidden width & 512 & 512  \\
        Residual blocks & 4 & 4 \\
        Normalization & LayerNorm & LayerNorm \\
        Input projection layers & 1 & 2 \\
        Output projection layers & 1 & 1 \\
        Backbone activation & SiLU & ReLU \\
        Time/noise embedding MLP
            & $129 \rightarrow 256 \rightarrow 512$
            & $17 \rightarrow 128 \rightarrow 128$ \\
        \bottomrule
    \end{tabular}
\end{table}

Our backbone follows the residual MLP design used in PCD, adapted to predict a rectified-flow velocity field. We use four residual blocks of width 512, with differences in the input projection, activation functions, and time embedding summarized in Table~\ref{tab:network_architecture}. We notice that ParetoFlow, PGD-MOO, and PCD \citep{pcd, paretoflow, pgdmoo} use different generative model architectures. We follow PCD's residual MLP design as closely as practical, with adaptations for our rectified-flow backbone.

PGD-MOO is \textbf{excluded} in our experiments as its released code covers only RE, ZDT, and DTLZ \citep{pgdmoo}, and could not be run under our setup, which has also been pointed out in \cite{pcd}.

\subsubsection{Validity of the Benchmarks}
\label{sec:app:validity}
Hypervolume depends on the reference point and objective normalization, so results from different benchmark versions are not always directly comparable. We found that several task-specific nadir values had been revised across Off-MOO-Bench versions \citep{xue2024offmoo}. Our experiments use the nadir configuration retained in our implementation from the May 2025--June 2026 development period. Under this configuration, our $\mathcal{D}(\mathrm{best})$ values are generally close to those reported by PGD-MOO \citep{pgdmoo} (within +/-2\% of other baseline papers).
We evaluate our methods and baseline reruns using the same task-specific normalization and reference points. These settings help explain differences from previously published values. Our NSGA-II-based baseline reruns also obtain positive hypervolume on some tasks reported as zero in the PCD paper \citep{pcd}. We report the results of these reruns under our common evaluation protocol.

Additionally, following \citet{xue2024offmoo, pcd}, we omit GP-based baselines, which are costly at this data scale and underperform DNN surrogates.

\subsection{Limitations}
\label{app:limitations}
Our work has several limitations. Our guarantees hold for the learned proxies, and the recovery bound is relative to the best front the flow can generate, not the true Pareto front. For our proposed approach, the best steering strength is task dependent (Appendix~\ref{app:tier2_strength}) and it lags behind surrogate-based methods on tasks whose fronts lie far outside the data. The method also needs a deterministic sampler and a continuous design representation, so we exclude combinatorial tasks. 
In future work, we plan to choose the steering strength adaptively per task and to combine noise-space steering with surrogate-based search for fronts far from the data. We also plan to extend the approach to discrete generative models.

\section{Ablation Studies}
\label{app:ablations}
\subsection{Effective rank across tasks}
\label{app:rankscaling}
Figure~\ref{fig:rankscaling} plots the effective rank of the noise-space metric against the noise dimension $D$ for all 47 tasks. For each task we form $\widehat\mM_\vw$ by \eqref{eq:decomp} at the uniform weight $\vw=\mathbf{1}/m$, with the RFM settings of Table~\ref{tab:hyperparams}, and report the participation ratio of the soft weights of \eqref{eq:steer},
\begin{equation}\label{eq:preff}
    r_{\mathrm{eff}} = \frac{\bigl(\sum_j \beta_j\bigr)^2}{\sum_j \beta_j^2}, \qquad \beta_j = \bigl(\widehat\lambda_j^{(\vw)} / \widehat\lambda_1^{(\vw)}\bigr)^{\alpha}, \quad \alpha = 0.3.
\end{equation}
We use the participation ratio rather than the cutoff rank $k_{\mathrm{eff}}$ of \eqref{eq:effrank}. $k_{\mathrm{eff}}$ depends on the choice of $\varepsilon_{\mathrm{rank}}$ and jumps by whole directions, so two tasks with similar spectra can receive different ranks, whereas $r_{\mathrm{eff}}$ has no threshold and varies continuously with the spectrum, so it can be compared across tasks. It also equals the number of directions the soft steering step effectively uses. Here $D$ is the dimension of the noise, which is the dimension of the continuous representation the flow is trained on, and not the number of design variables $d$ listed in Table~\ref{tab:refpoints}.
\begin{figure}[H]
    \centering
    \includegraphics[width=0.8\linewidth]{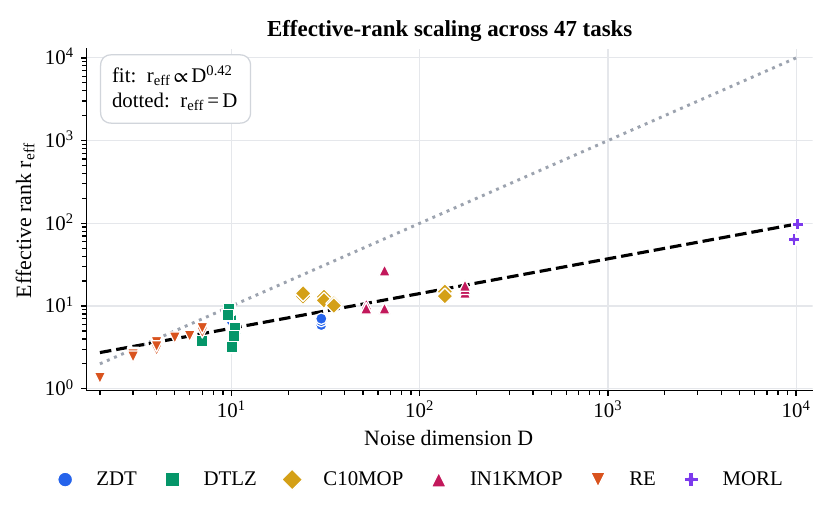}
    \caption{Effective rank \eqref{eq:preff} of $\widehat\mM_\vw$ at $\vw=\mathbf{1}/m$ against noise dimension $D$ for all 47 tasks, log-log axes. The dashed line is a least squares fit, $r_{\mathrm{eff}}\propto D^{0.42}$, and the dotted line is $r_{\mathrm{eff}}=D$.}
    \label{fig:rankscaling}
\end{figure}
\textbf{Across tasks.}
Pooling the 47 tasks, which span nearly four orders of magnitude in $D$, a least squares fit of $\log r_{\mathrm{eff}}$ on $\log D$ gives $r_{\mathrm{eff}}\propto D^{0.42}$. The effective rank grows sublinearly, and the gap to $r_{\mathrm{eff}}=D$ widens with dimension. The RE tasks use $1.4$ to $5.5$ directions at $D\le7$, the NAS tasks $9$ to $27$ at $D\le174$, and MORL $65$ and $100$ at $D\approx10^4$. The two MORL tasks lie at or just below the fitted line, so the highest-dimensional tasks follow the same trend as the rest, although their RFM retains $100$ directions, which caps $r_{\mathrm{eff}}$ at $100$. The fraction of noise space the objective responds to therefore falls from nearly all of it on the smallest tasks to about one percent on MORL.

\textbf{Within a family.}
The rank does not grow with $D$ inside a family. C-10/MOP stays between $10$ and $15$ as $D$ grows from $24$ to $140$, IN-1K/MOP at $D=174$ lies within the range of IN-1K/MOP at $D\approx50$ to $65$, and ZDT lies at $6$ to $7$ at both $D=10$ and $D=30$. This is the pattern predicted in Section~\ref{sec:lowrank}. The rank tracks the intrinsic dimension of the data manifold the flow has learned, which each family fixes independently of $D$, and the global fit summarizes how the families are arranged along $D$.

\subsection{Controlled Parameter studies}
\label{app:tier2}
Noise-space steering has seven hyperparameters. Four belong to the precomputation, the RFM fit that recovers the noise-space geometry. These are the sample budget $n$, the number of AGOP iterations $T$, the bandwidth $\ell$, and the ridge $\lambda$. Three belong to the steering step \eqref{eq:steer}, namely the retained rank $r$, the weighting (hard, or soft with exponent $\alpha$), and the steering strength $\gamma$. The guided variants add the guidance onset $t_{\mathrm{start}}$ and the guidance strength $\kappa$, which we study separately in Appendix~\ref{app:guidance_ablations}. The main results use one configuration for all tasks (Table~\ref{tab:hyperparams}), and the sensitivity studies below instead anchor the steering strength and the rank per task (Table~\ref{tab:tier2_anchor}).

We vary each of the seven in turn and ask which of them the method is sensitive to, and whether these sensitivities match the theory. The precomputation parameters should matter little. They enter only through the surrogate error, and $\widehat\mM_\vw$ converges to the metric of the surrogates as $n$ grows (Proposition~\ref{prop:agop}). The steering parameters are where the theory makes predictions. The choice of directions matters (Proposition~\ref{prop:optframe}), and a displacement helps only up to a task-dependent steering strength (Theorems~\ref{thm:ascent} and~\ref{thm:recovery}). Five studies follow below, on the steering strength, the rank and weighting, the soft exponent, the RFM fit, step length versus shape and scalarization.

\paragraph{Common setup.}
We use one task from each of five families, RE36, ZDT4, C-10/MOP1, IN-1K/MOP8, and MO-Hopper, fixed before any ablation was run, with noise dimensions in Table~\ref{tab:tier2_anchor}. For every swept value we run RFM-AN (steering only) and RFM-Cone (steering followed by Pareto-aware guidance, Section~\ref{sec:guidance}), so each conclusion covers both noise-space-only steering and the full method. Within a task and seed, all runs share checkpoints, proxies, evaluation noise, and the Das--Dennis weights, and each run generates $512$ candidates and keeps $256$ by non-dominated sorting on the proxy predictions. We use five seeds per setting and report HV normalized by $\mathrm{HV}(\gD(\mathrm{best}))$, so $1.0$ is the offline data.

Every parameter not being swept is held at the anchor of Table~\ref{tab:hyperparams}, namely soft weighting with $\alpha=0.3$, $n=10{,}000$, $T=7$, $\ell=20$, and $\lambda=n\lambda_{\mathrm{stat}}=10^{-3}$, where the sweeps vary $\lambda_{\mathrm{stat}}$. The steering strength $\gamma^\ast$ and the rank $r^\ast$ are instead set per task from preliminary sweeps (Table~\ref{tab:tier2_anchor}), since no single $\gamma$ suits all tasks. $\gamma^\ast$ is the largest $\gamma$ within one seed standard deviation of the task maximum, and $r^\ast$ is the rank with the highest mean HV under soft weighting, with ties resolved toward the smaller rank. Each study is also repeated at $r=4$ and at the full basis $r=r_{\max}$, or at $r=100$ and $300$ on MO-Hopper, so that no conclusion depends on the rank.

\begin{table}[H]\centering
\caption{Per-task anchor settings for the sensitivity studies. $D$ is the noise dimension and $r_{\max}$ the number of available directions. $\gamma^\ast$ is the largest steering strength within one seed standard deviation of the task's best HV in preliminary sweeps, and $r^\ast$ the rank with the highest mean HV under soft weighting, with ties resolved toward the smaller rank. On MO-Hopper both are the main-result defaults. The last column is the seed standard deviation of normalized HV at these settings.}

\label{tab:tier2_anchor}
\small\setlength{\tabcolsep}{6pt}
\begin{tabular}{lrrrrc}
\toprule
Task & $D$ & $r_{\max}$ & $r^\ast$ & $\gamma^\ast$ & seed s.d.\ at anchor \\ \midrule
RE36 & 4 & 4 & 3 & 20 & 0.01 \\
ZDT4 & 10 & 10 & 1 & 2.5 & $<0.01$ \\
C-10/MOP1 & 31 & 31 & 1 & 10 & 0.02--0.03 \\
IN-1K/MOP8 & 174 & 174 & 4 & 20 & 0.01 \\
MO-Hopper & 10{,}184 & 300 & 100 & 10 & 0.03--0.04 \\
\bottomrule
\end{tabular}
\end{table}

\paragraph{Robustness criterion.}
Overlapping error bars do not show that a parameter is unimportant, so we use a fixed criterion. A parameter is \emph{robust} over its tested range if, for every swept value, the seed-paired difference from the anchor has a 95\% confidence interval inside $\pm0.02$ normalized HV. Otherwise it is \emph{sensitive}, and we report where its good range lies.

\subsubsection{Steering strength and the reach limit}
\label{app:tier2_strength}
A single noise displacement helps where the front lies beyond the data and hurts where it does not, and the steering strength at which it stops helping is task dependent in the way the theory predicts. Theorem~\ref{thm:ascent} guarantees ascent on the proxy for a small enough $\gamma$, with a gain linear in $\gamma$, and the term $\rho(\vw)$ of Theorem~\ref{thm:recovery}, the value the steered noise has not yet collected, shrinks as the displacement grows. HV should therefore rise with $\gamma$ at first. Beyond some point, two effects take over. The curvature penalty of Theorem~\ref{thm:ascent} grows quadratically, and once the steered noise leaves the region the flow was trained on, the decode becomes less reliable (Remark~\ref{rem:reach}). Tasks should therefore fall into three regimes according to where their front lies relative to that region. A front already inside it gains nothing and degrades once the displacement leaves it, a front reached after a moderate displacement rises and saturates, and a front beyond it keeps rising while the flow extrapolates.

\textbf{Setup.} We sweep $\gamma\in\{0,1,2.5,5,7.5,10,20\}$ with soft weighting at $r=4$, all other settings at the anchor. On MO-Hopper we sweep $\gamma\in\{0,10,20\}$ under two RFM fits, $r=100$ with $n=10{,}000$ and $r=300$ with $n=40{,}000$. At $\gamma=0$, RFM-AN reduces to unguided decoding while RFM-Cone keeps its late guidance. Since the same $\gamma$ moves the noise by different amounts across tasks, Figure~\ref{fig:abl_strength} also plots HV against the per-coordinate displacement $\|\Delta\vz\|_2/\sqrt{D}$, the quantity in Remark~\ref{rem:reach}.

\begin{figure}[htb]
    \centering
    \includegraphics[width=\linewidth]{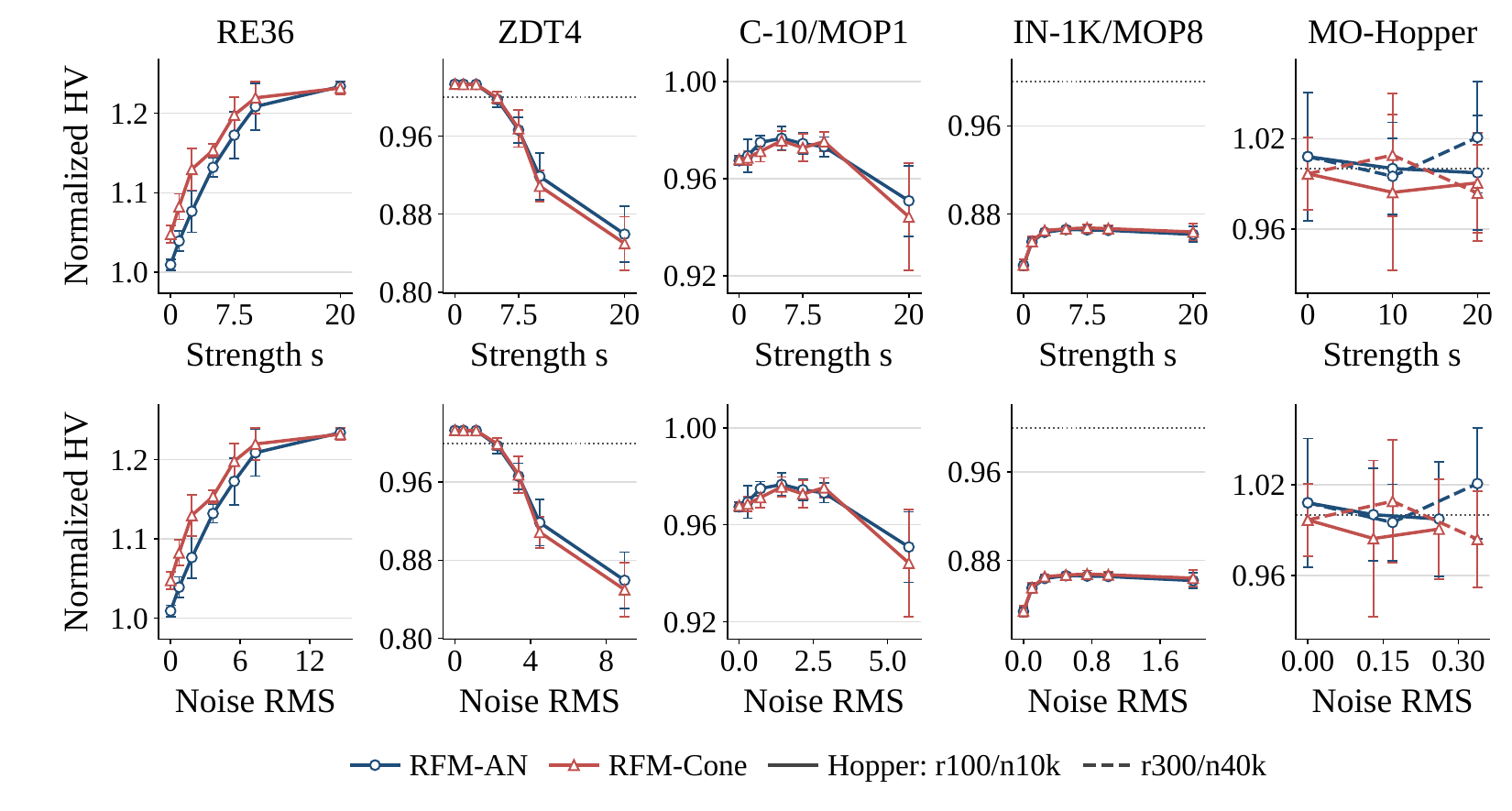}
    \caption{\textbf{Steering strength sweep.} Normalized HV against the steering strength $\gamma$ (top) and against the per-coordinate displacement $\|\Delta\vz\|_2/\sqrt{D}$ (bottom), mean $\pm$ s.d.\ over five seeds. The dotted line is the offline data. On MO-Hopper, solid and dashed lines are the RFM fits with $r=100$, $n=10{,}000$ and $r=300$, $n=40{,}000$.}
    \label{fig:abl_strength}
\end{figure}

\textbf{Results.}
All three regimes appear. ZDT4 is in the first. HV is flat up to $\gamma=2.5$ and falls by about $15\%$ at $\gamma=20$. The NAS tasks are in the second. IN-1K/MOP8 saturates from $\gamma=2.5$, and C-10/MOP1 peaks near $\gamma=5$ and declines only at $\gamma=20$. RE36 is in the third, still rising at $\gamma=20$ to $1.22\times$ the offline data. The displacement axis shows that where gains stop is set by the task. IN-1K/MOP8 saturates near $0.5$ units of per-coordinate displacement and C-10/MOP1 near $1.4$, while RE36 still gains at about $15$. MO-Hopper is insensitive to $\gamma$ within its seed noise under both RFM fits. RFM-Cone matches or exceeds RFM-AN at nearly every steering strength, so guidance never shifts the useful range and $\gamma$ can be chosen for steering alone.

\textbf{Recommendation.}
A steering strength of $\gamma\le2.5$ is safe on every task, never losing HV. Larger strengths help where the front lies beyond the data. The default $\gamma=10$ used in the main results lies within $0.02$ of each task's best except on ZDT4, and we recommend $\gamma=2.5$ when no such headroom is expected. The sensitivity studies use a fixed per-task $\gamma^\ast$ (Table~\ref{tab:tier2_anchor}), within $0.016$ of the best steering strength of this sweep on every task, and $\gamma^\ast=10$ on MO-Hopper, where HV does not depend on $\gamma$.

\subsubsection{Retained rank and hard versus soft weighting}
\label{app:tier2_rank}
The retained rank $r$ sets how many directions the step moves along, and the weighting sets how far it moves along each. Indicator weights move the full steering strength along every retained direction. Adding a direction the objective barely responds to therefore lengthens the step without raising the objective, so the step leaves the region where the flow and proxies are reliable and HV falls. Soft weights scale each direction by its eigenvalue. A weak direction receives a small weight, so adding it barely changes the step, and HV should not depend on how many directions are kept. Remark~\ref{rem:weights} makes this precise. The ascent bound for a weighted step is $\gamma\langle\boldsymbol\beta,|\bar\vd_\vw|\rangle-\tfrac12L_\vw\gamma^2\|\boldsymbol\beta\|_2^2$. With indicator weights $\|\boldsymbol\beta\|_2^2=r$, so each added direction adds a full unit to the curvature penalty while trailing directions add little to the gain. With soft weights $\beta_j=(\lambda_j/\lambda_1)^{\alpha}$ the penalty is set by the spectrum, so directions with small $\lambda_j$ change neither term. We therefore expect HV to fall with $r$ past $r^\ast$ under indicator weights and to stay flat up to the full basis under soft weights.

\textbf{Setup.}
We sweep $r$ up to each task's $r_{\max}$, and up to $300$ on MO-Hopper with the $r=300$ fit, at $\gamma=\gamma^\ast$ and all other settings at the anchor. Three weightings are compared on the same ordered basis. These are soft weights with $\alpha=0.3$, indicator weights (the hard mode of \eqref{eq:steer}), and geometric weights $1,\tfrac12,\tfrac14,\dots$ within the top $r$, a control that decays without using the spectrum. At $r=1$ all three coincide.

\begin{figure}[h]
    \centering
    \includegraphics[width=0.9\textwidth]{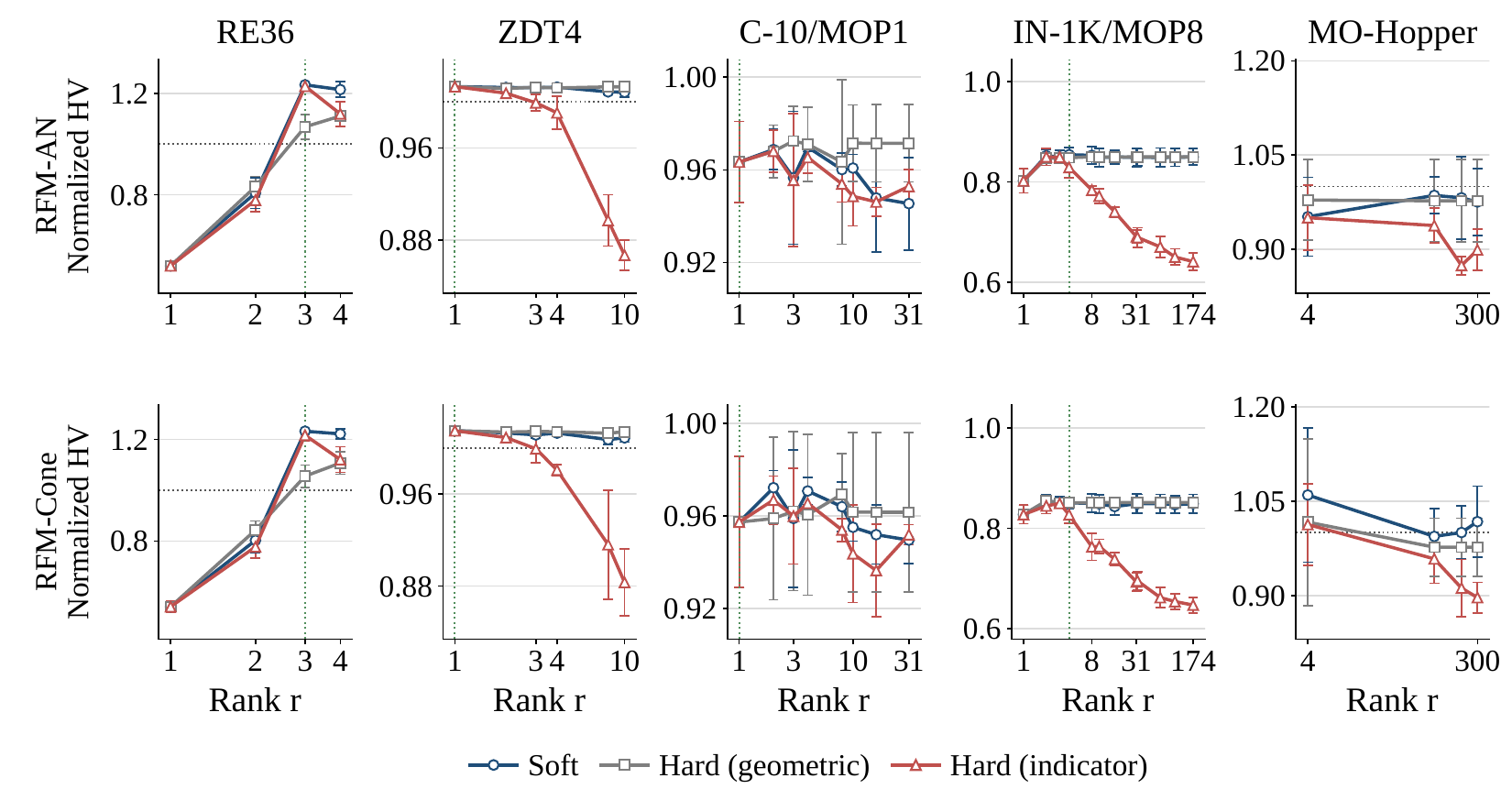}
    \caption{\textbf{Retained rank sweep for RFM-AN (top) and RFM-Cone (bottom)}, mean $\pm$ s.d.\ over five seeds. The vertical line marks $r^\ast$, the anchor rank of Table~\ref{tab:tier2_anchor}, and the horizontal dotted line the offline data. Indicator weights lose HV as $r$ grows on every task with room to add directions, while soft and geometric weights stay flat.}
    \label{fig:abl_rank}
\end{figure}

\textbf{Results.}
One direction is not enough. On RE36 every weighting scores about half the offline data at $r=1$, and soft weighting reaches $1.23\times$ it at $r=3$. Past $r^\ast$ the indicator weights lose HV on every task, by about $23\%$ on IN-1K/MOP8 and $14\%$ on ZDT4. On MO-Hopper they lose $5\%$ to $11\%$ between $r=4$ and $r=300$, a drop that stands out from its seed noise for RFM-Cone. Soft weights stay within $0.02$ of their value at $r^\ast$ up to the full basis on every task. The geometric control is also flat, so any decaying weighting removes the rank sensitivity. The spectrum decides how much is gained beyond that. On RE36 soft weighting beats the geometric control by about $0.16$, while on the NAS tasks the two lie within $0.03$. RFM-Cone reproduces every pattern.

\textbf{Recommendation.}
Soft weighting makes the retained rank a free choice, which supports the full-basis default of the main results. Indicator weights should be used only with a small, well-chosen rank.

\subsubsection{Step length versus shape}
\label{app:tier2_softhard}
Indicator weights, the hard mode of \eqref{eq:steer}, set $\beta_j=1$ for the top $r$ directions and $\beta_j=0$ for the rest, so the step moves equally along every retained direction. They can lose HV at high rank for two reasons. Their step is longer than the soft step at the same $\gamma$, and it is spread evenly over all retained directions instead of concentrated on the strong ones. A longer step pays a larger curvature penalty in Theorem~\ref{thm:ascent}, and a step spread over weak directions collects less gain. Rescaling the hard step to the length of the soft step removes the first effect, so any gap that remains is due to the weighting itself.

\textbf{Setup.}
We rerun the rank sweep of Appendix~\ref{app:tier2_rank} with both hard modes rescaled at every rank to the soft $\|\Delta\vz\|_2$, all other settings unchanged. We report the seed-paired difference in normalized HV between soft weighting and each hard mode, with and without this rescaling, as percentages.

\begin{figure}[htb]
    \centering
    \includegraphics[width=\linewidth]{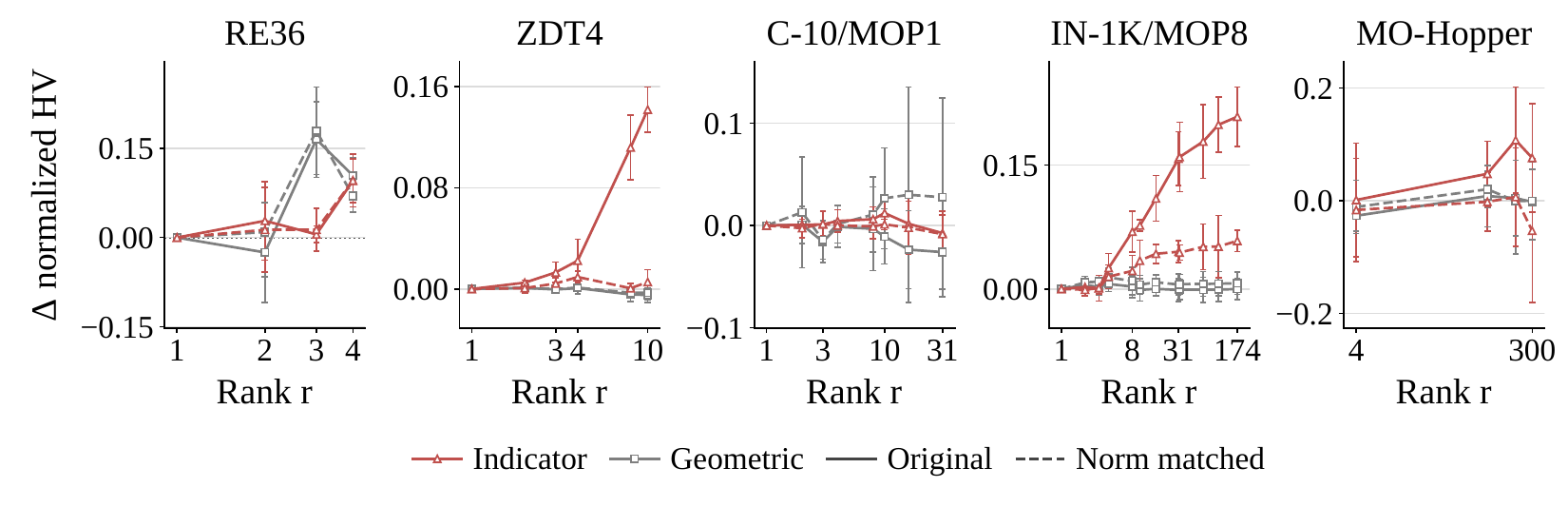}
    \caption{\textbf{Step length versus shape.} Paired difference in normalized HV, soft weighting minus each hard mode, for RFM-AN, mean $\pm$ s.d.\ over five seeds. Positive values favor soft weighting, so a higher curve means that hard mode loses more. Solid lines use the original step, dashed lines rescale the hard step to the length of the soft step.}
    \label{fig:abl_softhard}
\end{figure}

\textbf{Results.}
Figure~\ref{fig:abl_softhard} plots soft minus hard, so a curve above zero means soft weighting wins. Both reasons appear, on different tasks. On ZDT4 the indicator gap of $14\%$ closes to zero once the step is rescaled, so there the indicator mode only steps too far, and the same holds on MO-Hopper. On IN-1K/MOP8 the gap falls from $21\%$ but stays at $2\%$ to $6\%$ for every $r\ge8$, and on RE36 the gap of about $10\%$ at $r=4$ does not change under rescaling, so on both tasks spreading the step evenly costs HV even at equal length. The spectrum also matters beyond decay. On RE36 soft weighting leads the geometric control by about $17\%$ at $r=3$ with and without rescaling. On C-10/MOP1 the geometric control leads by up to $2.6\%$ with the original step, but soft weighting leads by about $3\%$ once the steps are matched, so the geometric edge comes from step length rather than from a better weighting.
RFM-Cone shows the same pattern.

\textbf{Recommendation.}
Soft weighting wins for both reasons. It keeps the step short when weak directions are added, and at equal length it still places the weight where the gradient is, which supports its use as the default throughout.
\subsubsection{Soft exponent}
\label{app:soft_exponent}
Following the comparison of hard and soft weighting, we now study the one hyperparameter of soft steering, the soft exponent $\alpha$ in $\beta_j=(\lambda_j/\lambda_1)^{\alpha}$. The exponent sets how many directions the step uses. As $\alpha\to0$ every weight tends to one and soft weighting becomes the indicator mode, and as $\alpha$ grows the weight concentrates on the leading direction, which is steering at $r=1$. Both limits lost HV in Appendix~\ref{app:tier2_rank}, the first on tasks with many weak directions and the second on tasks where several directions carry the gradient, so HV should peak in between. We count the directions a weighting uses by the participation ratio $r_{\mathrm{eff}}$ of \eqref{eq:preff}.

\textbf{Setup.}
We sweep $\alpha\in\{0.1,0.3,1,1.5,3\}$ at $r=r^\ast$ and at $r=r_{\max}$, with $\gamma=\gamma^\ast$ and all other settings at the anchor. On MO-Hopper we use the fit with $r=300$. We report gaps in normalized HV as percentages.

\begin{figure}[h]
    \centering
    \includegraphics[width=0.9\linewidth]{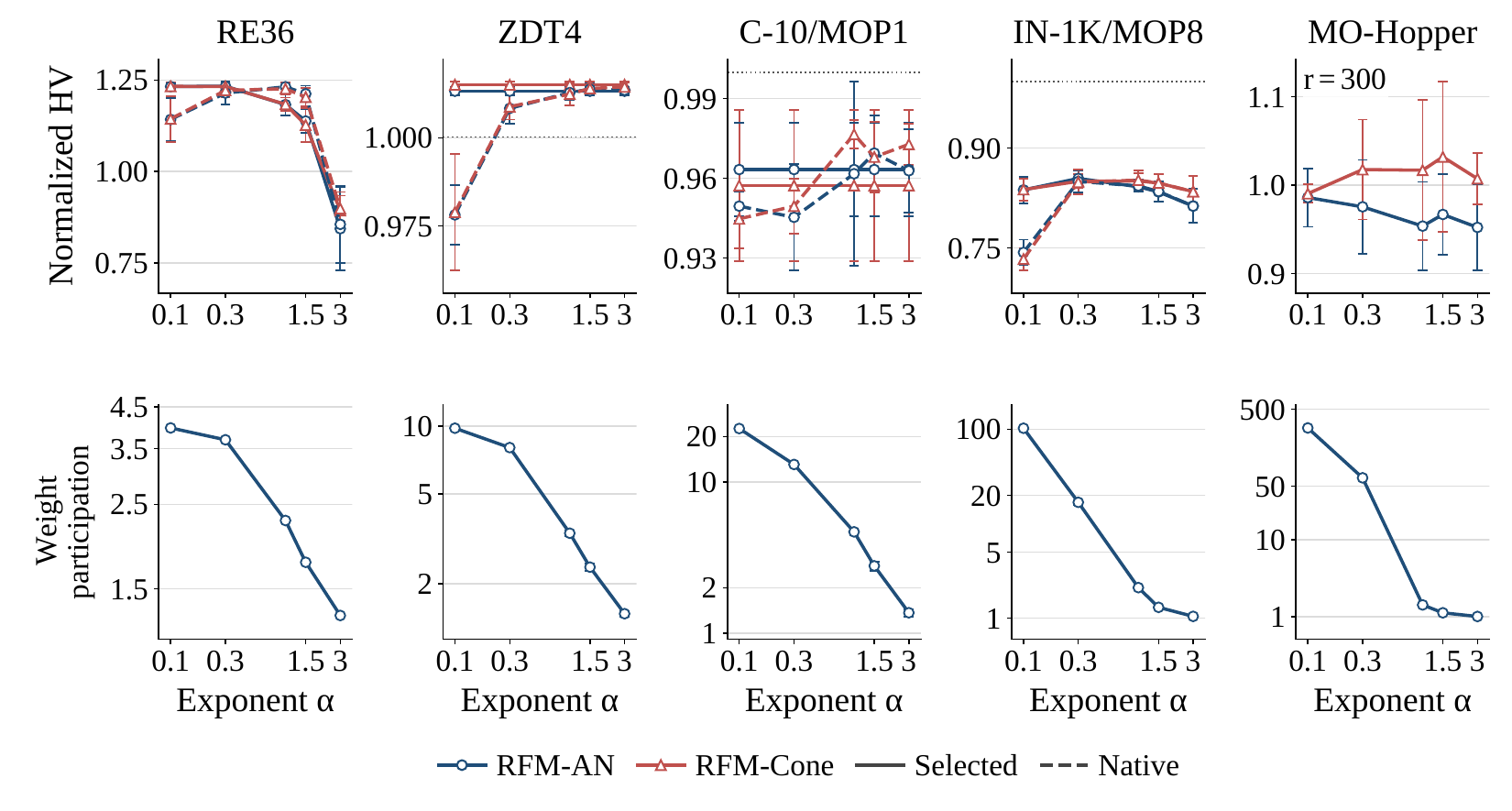}
    \caption{\textbf{Soft exponent.} Top, normalized HV against $\alpha$ at $r=r^\ast$ (solid) and $r=r_{\max}$ (dashed), mean $\pm$ s.d.\ over five seeds. Bottom, the number of directions the weighting uses, $r_{\mathrm{eff}}$ of \eqref{eq:preff}, at $r=r_{\max}$. It falls from nearly the full basis at $\alpha=0.1$ to a single direction at $\alpha=3$.}
    \label{fig:abl_alpha}
\end{figure}

\textbf{Results.}
Both limits fail where predicted. At $\alpha=0.1$ the step spreads over nearly the full basis, and HV at $r=r_{\max}$ drops by $12\%$ on IN-1K/MOP8, $4\%$ on ZDT4, and $2\%$ to $3\%$ on C-10/MOP1, the loss of the indicator mode in Appendix~\ref{app:tier2_rank}. At $\alpha=3$ the step collapses onto one direction, and RE36 loses $34\%$ to $39\%$, the same collapse $r=1$ produced. In between, $\alpha=0.3$ is the best exponent or within $0.1\%$ of it on every task at $r^\ast$, and within $3\%$ of the best at $r_{\max}$. At $r^\ast$, ZDT4 and C-10/MOP1 are flat because $r^\ast=1$ leaves nothing to reweight, and MO-Hopper is flat within its seed noise. RFM-Cone follows RFM-AN throughout. At $\alpha=0.3$ the step uses only a handful of directions, about $17$ of $174$ on IN-1K/MOP8 and $64$ of $300$ on MO-Hopper, consistent with the effective ranks of Appendix~\ref{app:rankscaling}.

\textbf{Recommendation.}
Any $\alpha\in[0.3,1]$ is a good choice, with HV within $3\%$ of the best on every task except RE36 at $r^\ast$, where $\alpha=1$ loses $5\%$. We use $\alpha=0.3$, the elbow at the flat end of this range. It spreads the step over the most directions that still keep HV near the best, which protects tasks such as RE36 that need several directions. Below $0.3$ the full-basis curves fall toward the loss at $\alpha=0.1$ as the step spreads over nearly every direction, and toward the sharp end the losses grow to $39\%$. The elbow therefore gives the best worst case, never more than $3\%$ below the best on any task.

\subsubsection{Estimating the noise-space metric}
\label{app:tier2_agop}
Steering uses the noise-space metric $\mM_\vw$, the AGOP of $H_\vw$ (Section~\ref{sec:geometry}), which we cannot compute directly and replace by its estimate $\widehat\mM_\vw$. The RFM builds this estimate from $n$ decoded noise samples. For each objective it fits a kernel ridge surrogate whose Laplace kernel uses a Mahalanobis metric $\mM_a$, and it alternates the fit with the AGOP update \eqref{eq:agop} for $T$ iterations. The converged $\mM_a$ are the diagonal blocks $\mC_{aa}$ of $\widehat\mM_\vw$ in \eqref{eq:decomp}.

The sample budget $n$, the iterations $T$, the bandwidth $\ell$, and the ridge $\lambda$ therefore set how well $\widehat\mM_\vw$ is estimated, not how the step is taken. They enter only through the surrogate error, and $\widehat\mM_\vw$ converges to the metric of the surrogates as $n$ grows (Proposition~\ref{prop:agop}), so HV should not depend on them over a reasonable range.

\textbf{Setup.}
We sweep $n\in\{5{,}000,10{,}000,20{,}000\}$ and $T\in\{4,7,11\}$ over their full grid, and $\ell\in\{10,20,40\}$ and $\lambda_{\mathrm{stat}}\in\{10^{-8},10^{-7},10^{-6}\}$ one at a time at $n=5{,}000$ and $T=7$, where the ridge is $\lambda=n\lambda_{\mathrm{stat}}$ and the anchor $\lambda_{\mathrm{stat}}=10^{-7}$ gives $\lambda=10^{-3}$ at $n=10{,}000$. On MO-Hopper, $n$ runs from $10{,}000$ to $40{,}000$, the bandwidth and ridge sweeps are at $n=10{,}000$, and every sweep is repeated under both RFM fits, $r=100$ and $r=300$.

\begin{figure}[H]
    \centering
    \includegraphics[width=0.8\linewidth]{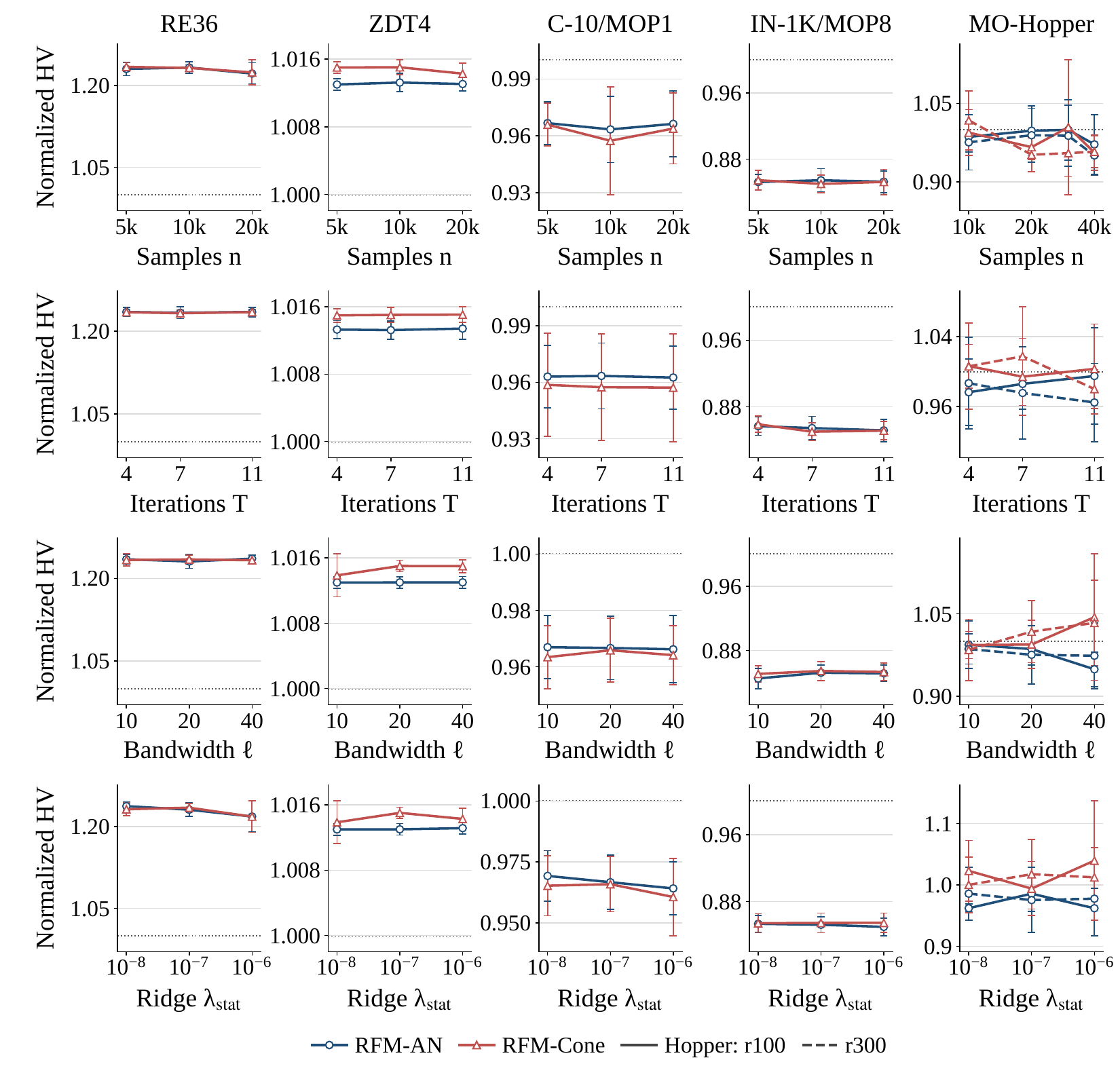}
    \caption{\textbf{Estimating the noise-space metric.} Normalized HV against the sample budget $n$, the AGOP iterations $T$, the bandwidth $\ell$, and the ridge $\lambda_{\mathrm{stat}}$ (rows), mean $\pm$ s.d.\ over five seeds. The dotted line is the offline data. On MO-Hopper, solid and dashed lines are the fits with $r=100$ and $r=300$.}
    \label{fig:abl_agop}
\end{figure}

\textbf{Results.}
Figure~\ref{fig:abl_agop} shows that HV does not depend on any of the four parameters over the tested range. Across a $4\times$ range of $n$, a $3\times$ range of $T$, a $4\times$ range of $\ell$, and a $100\times$ range of $\lambda_{\mathrm{stat}}$, HV changes by at most $1.6\%$ on RE36, ZDT4, C-10/MOP1, and IN-1K/MOP8, and the largest shift is on RE36 at $\lambda_{\mathrm{stat}}=10^{-6}$. MO-Hopper is flat within its seed noise under both fits. RFM-Cone follows RFM-AN throughout.

Note, the flat curves mean the directions are easy to estimate, not that they do not matter, since how the step is spread over them changes HV by up to $23\%$ (Appendix~\ref{app:tier2_rank}). Because the objective depends on only a few directions (Appendix~\ref{app:rankscaling}), the RFM recovers them reliably under any of these settings. The smallest setting tested, $n=5{,}000$ with $T=4$, stays within $1\%$ of the anchor on every task and fits in $0.4$\,s against $0.95$\,s at the anchor, which supports the cheaper precomputation used on the MORL tasks.

\textbf{Recommendation.}
Noise-space metric estimation is robust. The defaults $n=10{,}000$, $T=7$, $\ell=20$, and $\lambda=10^{-3}$ lie inside this robust range, and a smaller budget can be used where the precomputation is expensive.

\subsubsection{Robustness to the choice of scalarization}
\label{app:tier3_scala}
The main results use the weighted sum \eqref{eq:scalarization}, and the middle expression of \eqref{eq:decomp} extends the cache to any differentiable scalarization. This section tests that extension. We ask three questions. (i) Does steering ascend the scalarized objective under scalarizations other than the weighted sum, including non-smooth ones? (ii) Does the one-off precomputation serve them without refitting the RFM or re-running the flow, and at what cost in quality? And (iii) is the steering magnitude at which the gains saturate a property of the sampler, as Remark~\ref{rem:reach} suggests, or of the scalarization?
 
\textit{Scalarizations and setup.}
All maximize, with $\vw\in\Delta^{m-1}$ and the ideal point $\vz^\star$ set per task to the componentwise maximum of the offline objective values. Besides the weighted sum (WS), we test Tchebycheff (Tch), augmented Tchebycheff (aug-Tch, $\zeta=0.05$), and a smoothed Tchebycheff (softmin-Tch), where $\nu$ is fixed per task to $10\%$ of the offline objective range and not tuned, and $h^{\mathrm{soft}}_\vw$ is smooth with $h^{\mathrm{soft}}_\vw \to h^{\mathrm{Tch}}_\vw$ as $\nu\to0$,
\begin{equation}\label{eq:tier3_scalarizations}
\begin{aligned}
h^{\mathrm{Tch}}_\vw(\vx) &= \min_a\, w_a\bigl(\hat f_a(\vx) - z^\star_a\bigr), \qquad
h^{\mathrm{aug}}_\vw(\vx) = h^{\mathrm{Tch}}_\vw(\vx) + \zeta\,\vw^\top\hat\vf(\vx), \\
h^{\mathrm{soft}}_\vw(\vx) &= -\nu \log \textstyle\sum_a \exp\!\bigl(-w_a(\hat f_a(\vx) - z^\star_a)/\nu\bigr).
\end{aligned}
\end{equation}
Each scalarization assembles $\widehat\mM_\vw$ from the cached per-objective gradients by the middle expression of \eqref{eq:decomp}. The precomputation runs once per task and seed and is never repeated for a scalarization, and no scalarization issues a new flow evaluation at any weight. All three studies use ZDT2, seeds 0--4, the main-experiment RFM configuration in Table~\ref{tab:hyperparams} for everything not varied, and RFM-AN only, so that the mechanism is isolated from the guidance operators of Section~\ref{sec:guidance}, whose gradient combination is itself scalarization-specific. Equivalence is judged by the criterion of Appendix~\ref{app:tier2}, paired differences against WS at matched weight and seed with a 95\% confidence interval inside $\pm0.02$ normalized HV.

\paragraph{Study 1. Steering ascends under every scalarization.}
\label{app:tier3_ascent}
 Theorem~\ref{thm:ascent} guarantees ascent for any smooth scalarized objective and never uses linearity in $\vw$. Ascent should therefore hold under all four scalarizations, with the possible exception of Tch and aug-Tch, which are not smooth. The Tchebycheff value follows whichever objective is currently the worst, so its gradient jumps wherever two objectives tie for that role, and at such points the curvature bound of Assumption~\ref{ass:reg} that the proof of Theorem~\ref{thm:ascent} relies on does not exist. Whether these kinks matter in practice is exactly what the comparison against the smooth softmin-Tch decides.
 
\textbf{Setup.}
The matched-norm protocol of Figure~\ref{fig:teaser}b, run once per scalarization at the same $\vw$ and the same base noise. The leading arms displace along the top or the softly weighted directions of that scalarization's $\widehat\mM_\vw$, and three controls of equal norm displace along a trailing eigenvector, a random direction, and the sign-flipped leading direction. We read the true scalarized objective of the scalarization that steered, over $\|\Delta\vz\|/\sqrt{D}\in[0,1]$. All four scalarizations take their directions from the same cached precomputation.
 
\begin{figure}[H]
    \centering
    \includegraphics[width=0.8\linewidth]{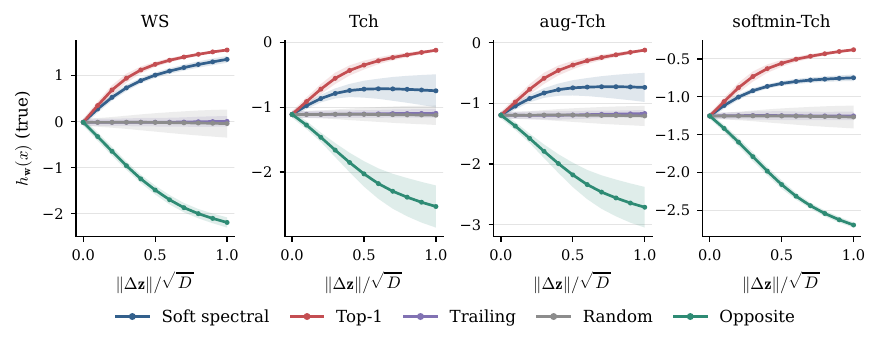}
    \caption{\textbf{Steering ascends under every scalarization.} True scalarized objective against step size on ZDT2 (mean $\pm$ s.d.\ over five seeds, same $\vw$ and base noise). In every panel the leading arms climb, the controls stay flat, and the reversed arm descends, and all four reach the same HV. Only the soft arm under the non-smooth Tch and aug-Tch becomes unreliable at large steps.}
    \label{fig:tier3_ascent}
\end{figure}
 
\textbf{Results.}
All runs use the balanced weight $\vw=(\tfrac12,\tfrac12)$, and every scalarization steers with directions taken from the same cached precomputation, with no refit. Steering improves the objective under all four (Figure~\ref{fig:tier3_ascent}). The top-1 arm climbs steadily, while the trailing and random directions produce no systematic gain and the reversed direction makes the objective worse. The four panels have different vertical scales, because each plots its own scalarized value, so we compare the scalarizations through hypervolume, which is the same metric for all four. Irrespective of the scalarization, steering raises the hypervolume by $3.6\%$ to $4.6\%$ and all four attain the same level, within one percentage point of one another. The precomputation therefore serves every scalarization equally well, and the ascent of Theorem~\ref{thm:ascent} is not specific to the weighted sum.

The one difference concerns smoothness. Theorem~\ref{thm:ascent} guarantees ascent only for a smooth objective (Assumption~\ref{ass:reg}). Tch and aug-Tch violate this assumption, because they take a minimum over the objectives and so have a kink wherever the minimizing objective switches. softmin-Tch replaces the minimum with a smooth approximation and satisfies it. The soft arm shows the consequence. Under Tch and aug-Tch it stops improving once the step size $\|\Delta\vz\|/\sqrt{D}$ reaches about $0.5$, and its results vary much more across runs, with a standard deviation of about $0.24$ at the largest step. Under softmin-Tch it keeps climbing with a standard deviation of $0.05$. Smoothing the scalarization is therefore enough to restore reliable ascent, as the theorem predicts. The top-1 arm stays consistent under all four scalarizations, including the non-smooth ones.
 
\paragraph{Study 2: One precomputation serves every scalarization.}
\label{app:tier3_cache}
Study~1 showed that steering works under every scalarization when its directions come from the cached precomputation. This study asks what the cache costs. The cache is fit once, and for any new weight it assembles $\widehat\mM_\vw$ without refitting and without running the flow. For the weighted sum the assembly is the bilinear combination of \eqref{eq:decomp}. For the Tchebycheff family, whose gradient is not linear in $\vw$, the same cached per-sample gradients are recombined directly. The most accurate alternative is to refit the RFM from scratch at every weight. We compare the two in quality and in cost.

\textbf{Setup.}
ZDT2, five runs, seven weights spread across the simplex. At each weight and for each scalarization we build $\widehat\mM_\vw$ from the cache and from a fresh refit, steer the same starting noise with both, and compare the hypervolume of the results, normalized by that of the best offline designs. For the weighted sum we also refit with the features held fixed to those of the cache, which checks that \eqref{eq:decomp} is implemented exactly.

\begin{figure}[H]
    \centering
    \includegraphics[width=\linewidth]{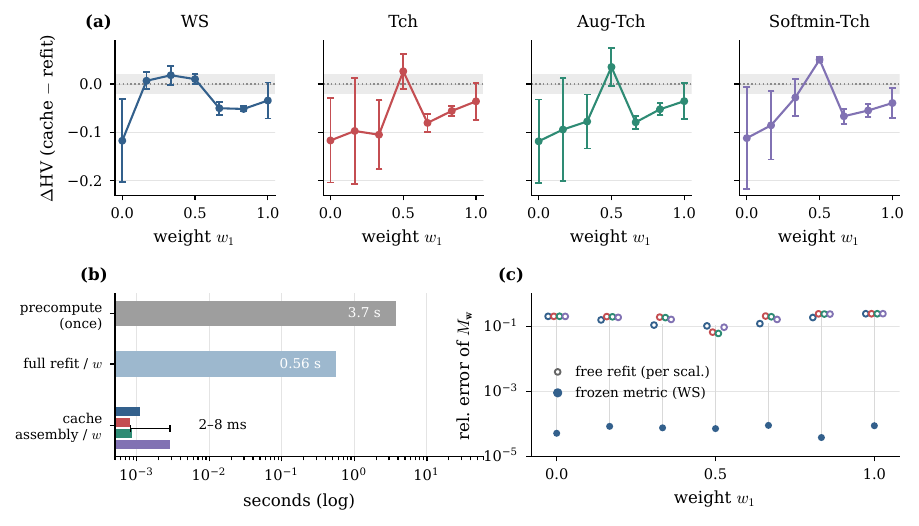}
    \caption{\textbf{One precomputation serves every scalarization at a fraction of the cost of refitting.} \textbf{(a)} HV of cache-based steering minus that of refit-based steering on ZDT2 (mean and 95\% CI over five runs, grey band $\pm0.02$). The cache matches or beats the refit at the balanced weight and trails it by a few percent elsewhere, similarly for all four scalarizations. \textbf{(b)} Cost per weight on a log scale. The precomputation runs once, a refit takes $0.56$\,s per weight, and cache assembly takes about a millisecond. Neither path runs the flow to build $\widehat\mM_\vw$. \textbf{(c)} Difference between the cache and refit matrices. With features held fixed, the refit reproduces the weighted-sum assembly to about $10^{-4}$, so \eqref{eq:decomp} is exact. A free refit differs by $6\%$ to $24\%$ for every scalarization, because it learns new features.}
    \label{fig:tier3_cache}
\end{figure}

\textbf{Results.}
Nonlinear scalarizations run from the cache just as the weighted sum does. They need no refit and no flow evaluations, and assembly costs less than $0.2\%$ of a refit for every scalarization, about a millisecond against $0.56$\,s per weight (Figure~\ref{fig:tier3_cache}b). The weighted sum assembles two to three times faster than the Tchebycheff family. The asymptotic costs of Section~\ref{sec:recovery}, $O(m^2D^2)$ for the weighted sum and $O(nD^2)$ for the Tchebycheff family, suggest a much larger gap, but at this dimension ($D=30$ on ZDT2) fixed overheads dominate both assemblies, so the measured gap is smaller.
The weighted-sum assembly is also exact, matching a fixed-feature refit to within $0.02\%$ at every weight (Figure~\ref{fig:tier3_cache}c).
The cache is not free in quality. A full refit learns features tailored to one weight, while the cache shares one set of features across all objectives, and this makes the refit somewhat better on average. 

Measured as a percentage of the best offline HV and averaged over weights, the refit leads by $3.1\%$ for the weighted sum and by $4.8\%$ to $6.6\%$ for the Tchebycheff family (Figure~\ref{fig:tier3_cache}a). How large the gap is depends mainly on the weight. At the balanced weight the cache matches the refit for the weighted sum and beats it by $3\%$ to $5\%$ for the Tchebycheff family. Toward the ends of the simplex the refit leads, by $5\%$ to $8\%$ at $w_1=\tfrac23$ and $\tfrac56$ and by about $12\%$ at the corner $w_1=0$. This pattern is what shared features predict, since they fit mixtures of the objectives better than any single objective.

\textit{Takeaway.}
One precomputation serves linear and nonlinear scalarizations alike, at a per-weight cost more than $500$ times below refitting. The price is a few percent of hypervolume away from the balanced weight, which a per-weight refit recovers when one specific trade-off matters.

\paragraph{Study 3. Full-sweep equivalence across scalarizations.}
\label{app:tier3_sweep}
Studies 1 and 2 look at one weight at a time. The main results are produced differently, by sweeping many weights and keeping the non-dominated candidates, so we repeat the comparison at that level. If the whole pipeline is run once per scalarization, do the four scalarizations end up with the same front quality?

\textbf{Setup.}
Full Das--Dennis sweeps per scalarization at two steering budgets, $\|\Delta\vz\|/\sqrt{D}\in\{0.3, 1\}$, on ZDT1--3 and DTLZ2, over five seeds. The Tchebycheff variants use the absolute-value convention. We report HV normalized by $\gD(\text{best})$ and IGD against the analytic front, and compare each Tchebycheff variant with the weighted sum at matched budget and seed, counting a difference within $\pm2\%$ HV as equivalent.
\begin{figure}[H]
    \centering
    \includegraphics[width=0.8\linewidth]{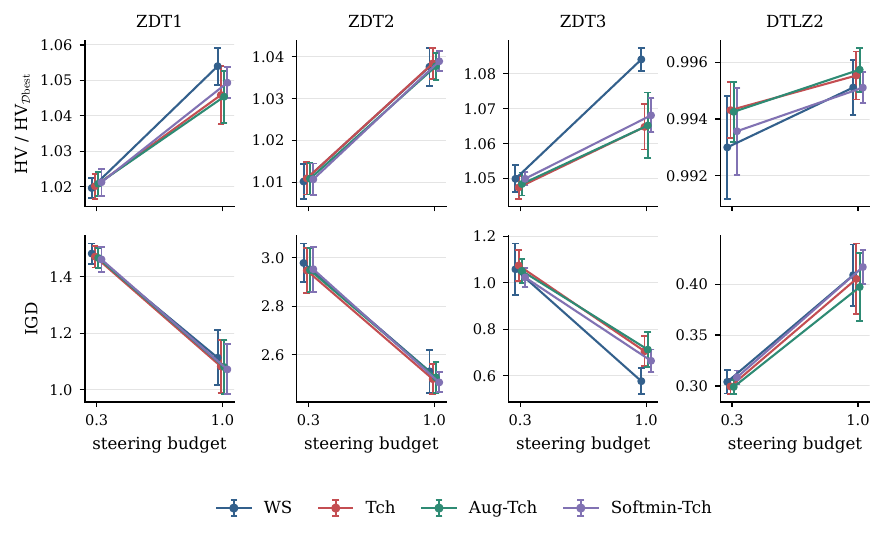}
    \caption{\textbf{Every scalarization reaches the same front quality.} Hypervolume relative to the best offline designs (top) and IGD to the analytic front (bottom) at two steering budgets, mean and 95\% CI over five seeds. The four scalarizations agree within the $2\%$ tolerance on every task. More steering helps on the ZDT tasks, whose optima lie on the design boundary, and hurts IGD on DTLZ2, whose optima lie in its interior.}
    \label{fig:tier3_sweep}
\end{figure}
\textbf{Results.}
Irrespective of the scalarization, the full pipeline reaches the same front quality (Figure~\ref{fig:tier3_sweep}). All $24$ comparisons (four tasks, two budgets, three Tchebycheff variants) fall within the $2\%$ tolerance of the weighted sum in hypervolume, and in $21$ of them the entire confidence interval does. The remaining three are ZDT3 at the larger budget, where the weighted sum leads by $1.6\%$ to $1.9\%$, inside the tolerance on average but with intervals that extend slightly beyond it. On the ZDT tasks every scalarization ends $1\%$ to $8\%$ above the hypervolume of the best offline designs. On DTLZ2 steering improves on the unsteered samples but stays just below the offline best, at $99.3\%$ to $99.6\%$ of it.

A larger steering budget changes the results the same way for every scalarization, but not the same way on every task. It pushes more candidates onto the boundary of the design space, from about $40\%$ to over $90\%$ on the ZDT tasks and from $8\%$ to $31\%$ on DTLZ2. On ZDT the Pareto-optimal designs lie on that boundary, so a larger budget helps, raising hypervolume by $2\%$ to $3\%$ and lowering IGD by $15\%$ to $37\%$. On DTLZ2 they lie in the interior, so a larger budget raises hypervolume by only $0.2\%$ and worsens IGD by about $34\%$, because the generated points crowd toward the edges and cover less of the front. Since this happens identically under all four scalarizations, the complete approach of Section~\ref{sec:experiments} behaves the same whichever scalarization drives it, which is the sweep-level version of Studies~1 and~2.

\subsubsection{Steering beyond the offline Support on ZDT1}
\label{app:zdt1_front}

Figure~\ref{fig:zdt1_front} compares the solutions generated
by RFM-AN with the best offline points and the true Pareto
front on ZDT1. We use the weighted-sum sweep from Study~3
at two steering budgets,
$\|\Delta\vz\|_2/\sqrt{D}\in\{0.3,1\}$.
The offline points are the non-dominated subset used to
compute $\gD(\text{best})$. Both objectives are plotted
in their original minimization convention.

With the smaller step, the generated frontier already
extends below the best offline points. Increasing the
step produces a clearer improvement, reaching lower
$f_2$ values at comparable $f_1$ across much of the
displayed range and moving the frontier closer to the
true Pareto front. This gives a concrete example of
steering beyond the trade-offs observed in the offline
data: changing the initial noise alone is enough to
reach better objective values, without adding guidance
along the sampling trajectory.

\begin{figure}[htb]
    \centering
    \includegraphics[width=0.6\linewidth]{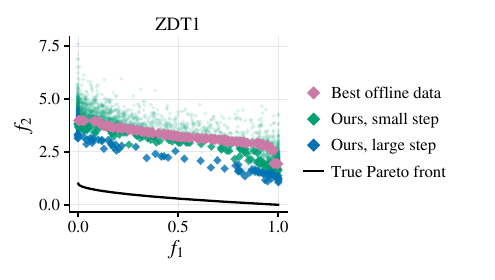}
    \caption{
    \textbf{Steering beyond the offline front on ZDT1.}
    Best offline points, RFM-AN samples under weighted-sum
    scalarization at steering budgets
    $\|\Delta\vz\|_2/\sqrt{D}=0.3$ and $1$, and the true
    Pareto front. Both objectives are minimized.
    The larger step moves the generated frontier further
    beyond the offline points and toward the true front.
    }
    \label{fig:zdt1_front}
\end{figure}

\subsection{Guidance Ablations}
\label{app:guidance_ablations}
Guidance is applied on top of noise-space steering, late in the trajectory (Section~\ref{sec:guidance}). Its main parameter is the guidance strength $\kappa$, which scales the guided velocity relative to the flow, $\eta(t)=\kappa\|\vv_\theta\|/\|\gG_\vw\|$, or the ramp maximum $\eta_{\max}$ for RFM-APG. Theorem~\ref{thm:cone} predicts how HV should respond. For a small enough step the RFM-Cone direction improves every objective to first order, so HV should rise with the strength while the first-order term dominates. At large strengths the second-order term takes over and HV can fall, and where noise-space steering already reaches the front there is little left for guidance to add.

\textbf{Setup.}
We sweep the guidance strength over $\{0,\tfrac14,\tfrac12,1,2,4\}$ times its default for RFM-Cone, RFM-DAMG, and RFM-APG, on ZDT2, ZDT4, DTLZ2, DTLZ6, C-10/MOP8, IN-1K/MOP8, RE41, and MO-Swimmer, with five seeds. Steering is fixed at $\gamma=10$ with soft weighting and $\alpha=0.3$, and guidance starts at $t_{\mathrm{start}}=0.7$. RFM-APG starts at $t_{\mathrm{start}}=0.8$, and in this sweep its steering stage uses indicator weights. Its HV without guidance lies within $3\%$ of that of the other operators on seven of eight tasks, and every result below is measured against each operator's own run without guidance. At strength $0$ each operator reduces to steering alone, and we report the seed-paired change in normalized HV from that point, in percentage points (pp).

\begin{figure}[h]
    \centering
    \includegraphics[width=\linewidth]{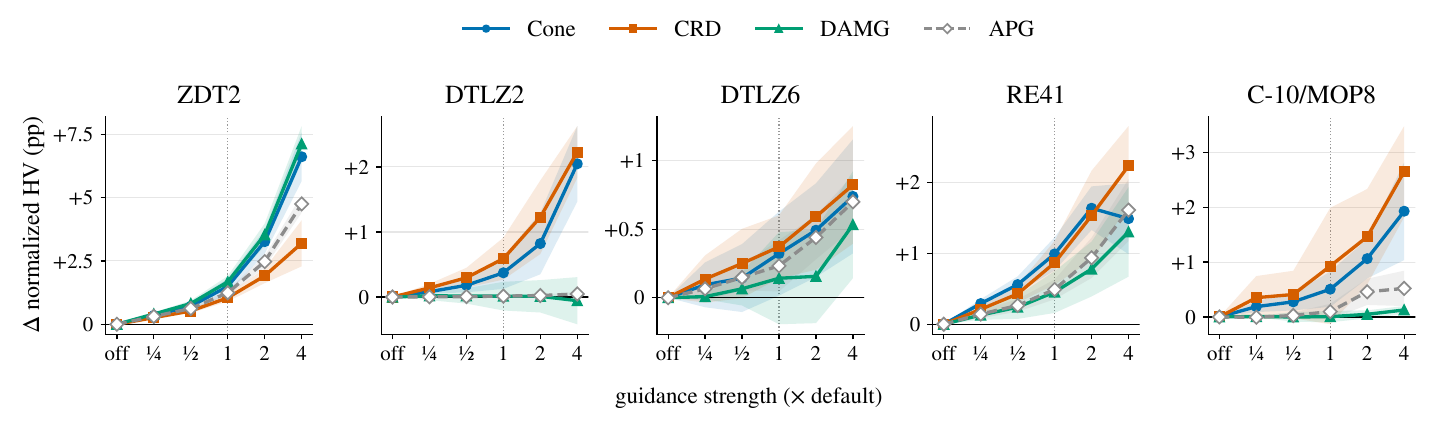}
    \caption{\textbf{Guidance adds to steering and grows with its strength.} Change in normalized HV over steering alone against the guidance strength, relative to its default (dotted line), on the five tasks where steering leaves headroom, mean and 95\% CI over five seeds. The remaining three tasks are in Figure~\ref{fig:guid_all_hv100}.}
    \label{fig:guid_strength}
\end{figure}

\textbf{Results.}
Guidance improves on steering alone wherever steering leaves headroom (Figure~\ref{fig:guid_strength}). At the default strength RFM-Cone raises HV significantly on five of eight tasks, by $0.3$ to $1.5$\,pp, and the gain keeps growing with the strength, reaching $6.6$\,pp on ZDT2 and about $2$\,pp on DTLZ2 and C-10/MOP8 at $4\times$. On the other three tasks RFM-Cone stays within $0.6$\,pp of steering alone up to $2\times$ (Figure~\ref{fig:guid_all_hv100}). Across all eight tasks RFM-Cone never loses HV significantly up to $2\times$ its default, and a loss appears only at $4\times$, on MO-Swimmer ($3.3$\,pp) and IN-1K/MOP8, which is the second-order turnover Theorem~\ref{thm:cone} predicts. RFM-DAMG is the more conservative operator. It gains $1.7$\,pp on ZDT2 and $0.5$\,pp on RE41 at the default and stays within $0.3$\,pp of steering alone on C-10/MOP8, DTLZ2, and IN-1K/MOP8, so it acts where the geometry supports a step and otherwise leaves the steered sample unchanged. RFM-APG gains on the same tasks as RFM-Cone and loses only on ZDT4, by $3.0$\,pp at twice its default.

\begin{figure}[h]
    \centering
    \includegraphics[width=0.8\linewidth]{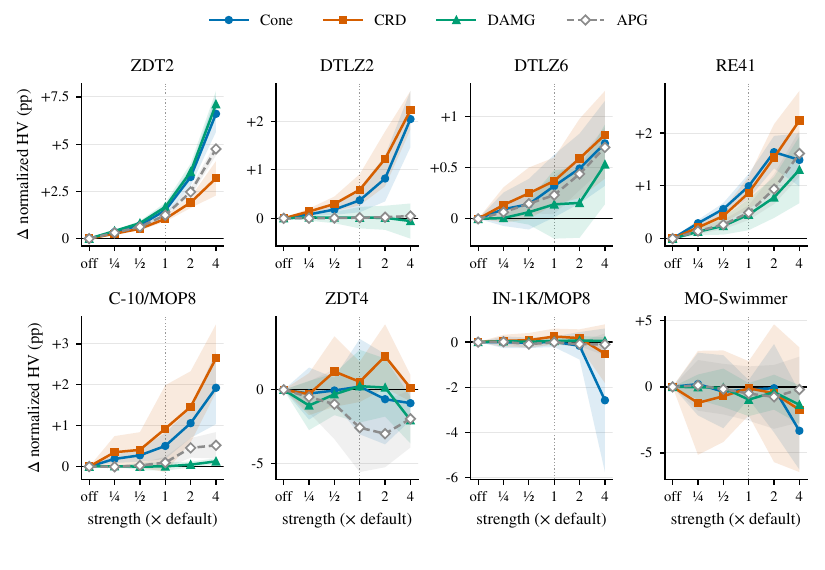}
    \caption{\textbf{Guidance strength on all eight tasks.} As Figure~\ref{fig:guid_strength}, including ZDT4, IN-1K/MOP8, and MO-Swimmer.}
    \label{fig:guid_all_hv100}
\end{figure}

The gains are larger on the better half of the returned set (Figure~\ref{fig:guid_all_hv50}). At the default strength RFM-Cone raises the $50$th-percentile HV significantly on five tasks, by $1.1$ to $2.1$\,pp, and on DTLZ2 by $21$\,pp at $4\times$. Guidance is also consistent across runs (Figure~\ref{fig:guid_winrate}). At the default strength it improves on steering alone in $78\%$ of task and seed pairs for RFM-Cone, $72\%$ for RFM-APG, and $65\%$ for RFM-DAMG, and every operator stays above $60\%$ at every strength.

\begin{figure}[h]
    \centering
    \includegraphics[width=0.8\linewidth]{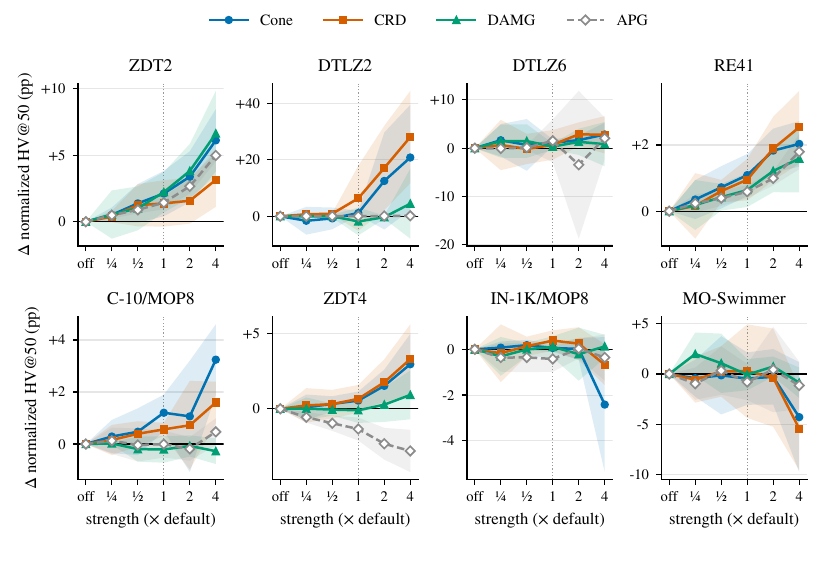}
    \caption{\textbf{Guidance strength at the $50$th percentile.} As Figure~\ref{fig:guid_all_hv100}, for the HV of the better half of the returned set.}
    \label{fig:guid_all_hv50}
\end{figure}

\begin{figure}[h]
    \centering
    \includegraphics[width=0.5\linewidth]{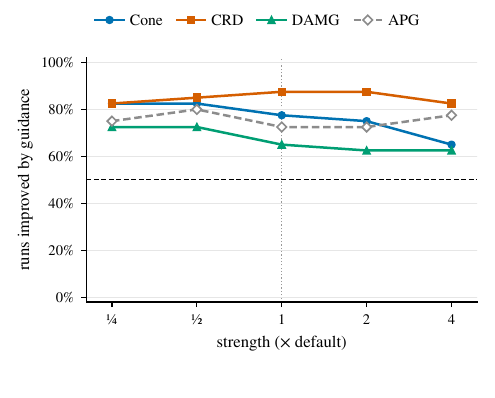}
    \caption{\textbf{Guidance improves most runs.} Fraction of task and seed pairs in which guidance raises HV over steering alone, against the guidance strength. The dashed line is chance.}
    \label{fig:guid_winrate}
\end{figure}

\textbf{Recommendation.}
Guidance is safe over an eightfold range of its strength, and the default is conservative. Doubling it raises HV further on five of eight tasks without a significant loss on any, so the default leaves room for tasks where steering stops short of the front. The main results keep the default so that a single setting serves every task.

\section{Theoretical Results}
\subsection{Notation and setup}
\label{sec:notation}
We maximize $m$ objectives $\vf=(f_1,\dots,f_m):\gX\to\R^m$ over a design space
$\gX\subseteq\R^D$, with access only to an offline dataset $\gD=\{(\vx_i,\vy_i)\}_{i=1}^{N}$, $\vy_i=\vf(\vx_i)$. No queries to $\vf$ are permitted. We train proxies $\hat\vf=(\hat f_1,\dots,\hat f_m)$ on $\gD$. For a scalarization weight $\vw$ in the simplex $\simplex^{m-1}=\{\vw\in\R^m:w_a\ge0,\ \sum_a w_a=1\}$, write the scalarized proxy $h_\vw(\vx)=\vw^\top\hat\vf(\vx)$.

A rectified flow with velocity field $\vv_\theta$ induces a \emph{deterministic}
sampler $\Phi_\theta:\R^D\to\R^D$, $\vz\mapsto\vx$, integrating the unmodified
ODE from a seed $\vz\sim\gN(\vzero,\mI_D)$. We abbreviate $\gN:=\gN(\vzero,\mI_D)$.
Because $\Phi_\theta$ is deterministic, every objective is equally a function of
the seed. We therefore work throughout with the \emph{noise-space} maps,
\begin{equation}
  g_a(\vz):=\hat f_a\bigl(\Phi_\theta(\vz)\bigr),
  \qquad
  H_\vw(\vz):=h_\vw\bigl(\Phi_\theta(\vz)\bigr)=\sum_{a=1}^m w_a\,g_a(\vz),
  \tag{\ref{eq:noise-obj}}\label{app:eq:noise-obj}
\end{equation}
so that $\nabla H_\vw=\sum_a w_a\,\nabla g_a$.

\paragraph{Noise-space geometry.}
Define the per-objective gradient-covariance blocks and the direction-specific
metric
\begin{equation}
 \mC^\star_{ab}:=\E_{\vz\sim\gN}\!\bigl[\nabla g_a(\vz)\,\nabla g_b(\vz)^\top\bigr]\in\R^{D\times D},
  \qquad
  \mM_\vw:=\E_{\vz\sim\gN}\!\bigl[\nabla H_\vw\,\nabla H_\vw^\top\bigr]
        =\sum_{a,b=1}^m w_a w_b\,\mC^\star_{ab}.
  \label{eq:Mw}
\end{equation}
The cached blocks $\mC_{ab}$ of \eqref{eq:decomp} estimate $\mC^\star_{ab}$ through the RFM surrogates (Section~\ref{sec:metric-props}). As in Section~\ref{sec:geometry}, $\mM_\vw$ has eigenpairs $(\lambda_j^{(\vw)},\vu_\vw^{(j)})$ with $\lambda_1^{(\vw)}\ge\cdots\ge\lambda_D^{(\vw)}\ge0$, its top-$r$ eigenvectors form the \emph{active frame} $\mU_\vw\in\R^{D\times r}$, and $\widehat{\mM}_\vw$ is its estimate from the RFM cache.

Note, results in this appendix are stated for the population quantities $\mM_\vw,\lambda_j^{(\vw)},\vu_\vw^{(j)}$. Algorithm~\ref{alg:steer} uses their estimates (hatted), and Section~\ref{sec:metric-props} bounds the difference.

The rank $r$ is not a free input but is read off the spectrum: we keep the
fewest directions carrying all but an $\varepsilon_{\mathrm{rank}}$-fraction of the objective's
sensitivity.

\begin{definition}[Effective rank]\label{def:effrank}
For a tolerance $ \varepsilon_{\mathrm{rank}}\in(0,1)$, the effective rank of $\mM_\vw$ is
\begin{equation}
  k_{\mathrm{eff}}(\vw)=\min\Bigl\{r\in\{1,\dots,D\}:\ \textstyle\sum_{j>r}\lambda_j^{(\vw)}< \varepsilon_{\mathrm{rank}}\sum_{j=1}^D\lambda_j^{(\vw)}\Bigr\}.
  \tag{\ref{eq:effrank}}\label{app:eq:effrank}
\end{equation}
For the hard cutoff we take $r=k_{\mathrm{eff}}(\vw)$; when $k_{\mathrm{eff}}(\vw)\ll D$ the
objective depends on only a few noise directions. The discarded tail
$\sum_{j>r}\lambda_j^{(\vw)}$ is exactly the truncation error bounded in
Proposition~\ref{prop:tail}, so the rank rule and the error bound share one
quantity. (A single fixed $r$ across all $\vw$, or the spectral-gap choice
$r=\arg\min_j\lambda_{j+1}^{(\vw)}/\lambda_j^{(\vw)}$, are drop-in alternatives.)
\end{definition}

\paragraph{Steering.}
Given the eigenbasis, a steering step is
\begin{equation}
  \vz_\vw=\vz+\gamma\sum_{j}\beta_j\,\sigma_j\,\vu_\vw^{(j)},
  \qquad \hat\vx_\vw=\Phi_\theta(\vz_\vw),
  \tag{\ref{eq:steer}}\label{app:eq:steer}
\end{equation}
integrated through the \emph{unmodified} ODE, so $\hat\vx_\vw$ is an ordinary sample of the trained model. Here $\gamma>0$ is the steering strength, $\boldsymbol\beta\ge\vzero$ are the direction weights, and $\vsig\in\{\pm1\}^D$ is a sign pattern fixed across $\vz$ (chosen by the Pearson rule of Algorithm~\ref{alg:steer}). The displacement has norm $R:=\gamma\|\boldsymbol\beta\|_2$. We write
$\vd_\vw(\vz):=\mU_\vw^\top\nabla H_\vw(\vz)$ for the active-subspace projection of the gradient.

\subsection{The steering algorithm}
\label{sec:algorithm}
Algorithm~\ref{alg:steer} runs in two phases. The one-time \emph{precomputation} learns the noise-space geometry of the objectives: it draws seeds from the prior, decodes each through the flow, scores the resulting design with the proxies, and fits an RFM surrogate $\phi_a$ to every objective map $g_a$. The surrogate supplies cheap analytic gradients, so we obtain the gradient-covariance blocks $\mC_{ab}$ without ever differentiating through $\Phi_\theta$. These blocks are the only quantities carried forward.

The \emph{steering} phase produces one candidate per weight $\vw$. We assemble $\widehat\mM_\vw$ from the cached blocks and compute its eigenvectors $\widehat\vu_\vw^{(j)}$, which are the noise directions the scalarized objective is most sensitive to. Since the metric does not determine which way along each direction is uphill, we orient each eigenvector by the sign of its Pearson correlation with the scalarized proxy values on the cache. The noise is then displaced by a weighted steering step along these directions and decoded through the unmodified flow. The precomputation does not depend on $\vw$, so a full sweep over weights reuses it, and each new weight costs one metric assembly, one eigendecomposition, and one ODE solve.

The base method is a single sign-selected steering step whose $r$ directions are scaled by a weight vector $\boldsymbol\beta\ge\vzero$ (line~\ref{line:steer}). The two instances are an adaptive hard cutoff $\beta_j=\mathbf{1}[j\le k_{\mathrm{eff}}(\vw)]$ and a cutoff-free soft weighting $\beta_j=(\lambda_j^{(\vw)}/\lambda_1^{(\vw)})^{\alpha}$. The results of Section~\ref{sec:theory-body} hold for either (Remark~\ref{rem:weights}). Optionally, a \emph{guidance operator} $\mathcal{G}_\vw$ adds a correction $\eta(t)\,\mathcal{G}_\vw$ to the velocity $\vv_\theta$ late in the trajectory, for $t\ge t_{\mathrm{start}}$ (line~\ref{line:guide}). The operators use the precomputed geometry or the per-objective proxy gradients and never differentiate through the ODE solve and they are analyzed in Section~\ref{sec:complementary}. When no operator is supplied, the method reduces to one unmodified ODE solve.

\subsection{First-order guarantee}
\label{sec:theory-body}
We now show that a steering step provably increases the scalarized objective in expectation over the noise, and that the AGOP frame is the best rank-$r$ frame in which to take that steering step. Both results rest on a \emph{single} regularity condition, smoothness of $H_\vw$. Sign selection, by contrast, is part of the algorithm and not an assumption. The theorem holds for any fixed sign pattern, and the Pearson rule of the Algorithm~\ref{alg:steer} targets the pattern that maximizes the resulting bound.

\subsubsection{Assumptions}
\begin{assumption}[Smoothness]\label{ass:reg}
For each $\vw\in \simplex^{m-1}$, the noise-space objective $H_\vw$ is twice continuously differentiable ($C^2$), its gradient is square-integrable under $\gN$, and its curvature is bounded, with $\|\nabla^2 H_\vw(\vz)\|_2\le L_\vw$ for every $\vz \in \R^D$ and some constant $L_\vw > 0$.
\end{assumption}

Note, that this assumption is mild because it follows from the architecture rather than from the data. Since $H_\vw=h_\vw\circ\Phi_\theta$, it is $C^2$ with a bounded Hessian whenever the velocity field $\vv_\theta$ and the proxies $\hat f_a$ have bounded first and second derivatives. Integrating such a velocity field over the finite interval $[0,1]$ yields a flow map with bounded derivatives as well. Our flow model uses SiLU activations, whose first and second derivatives are bounded, so the finite weights the bound holds everywhere without further conditions.

Our proxies use LeakyReLU, which is not twice differentiable at zero. Replacing it by a smooth approximation, such as a softplus-based version, gives a proxy that satisfies Assumption~\ref{ass:reg} and stays uniformly close to the original. \fullref{Theorem}{thm:ascent}{app:thm:ascent} then carries over to the original proxy up to a small error that vanishes as the approximation becomes exact.
Furthermore, square-integrability of $\nabla H_\vw$ makes the metric $\mM_\vw=\E[\nabla H_\vw\nabla H_\vw^\top]$ finite, and by the Gaussian Poincar\'e inequality it also makes $H_\vw$ itself square-integrable under $\gN$, so all expectations below are finite. The constant $L_\vw$ is an \emph{upper} bound on curvature, equivalently the Lipschitz constant of $\nabla H_\vw$. A small $L_\vw$ means the landscape bends gently, so the first-order approximation stays accurate over a longer steering step, while a large $L_\vw$ forces a smaller steering strength. It is exactly what fixes the admissible steering strength in \fullref{Theorem}{thm:ascent}{app:thm:ascent}.

\paragraph{Sign selection, a device, not an assumption.}
The one thing the metric cannot supply is orientation. $\mM_\vw$ is built from outer products $\nabla H_\vw\nabla H_\vw^\top$ and is invariant to flipping the sign of any eigenvector, so it identifies \emph{which} directions the objective is sensitive to but not \emph{which way} along them is uphill. The algorithm resolves this once per direction, independently of the noise vector being steered, by setting $\sigma_j$ to the sign of the Pearson correlation between $\langle\vz_i,\vu_\vw^{(j)}\rangle$ and $\vw^\top\hat\vy_i$ over the cache, which requires no ODE solve. \fullref{Theorem}{thm:ascent}{app:thm:ascent} does not depend on how the signs are chosen. Its inequality holds for every fixed sign pattern, and the signs only determine how large the guaranteed gain is. The Pearson rule is therefore a choice rather than a requirement, and a natural one~\citep{beaglehole2026steering}, since its population value is exactly the pattern that makes the guaranteed gain largest. An empirical sign that disagrees with the population sign costs only the share of the gain carried by that direction (Remark~\ref{rem:caveats}).

\paragraph{What we do not assume.}
Two conditions are deliberately absent. First, we do \emph{not} assume proxy
fidelity ($\hat f\approx f$): \fullref{Theorem}{thm:ascent}{app:thm:ascent} guarantees ascent on the
\emph{proxy} scalarization $H_\vw$, and whether that tracks the true objective is a
separate question, isolated to the recovery result (Theorem~\ref{thm:recovery}).
Second, we do \emph{not} assume the active subspace carries gradient signal at
every seed. The theorem instead \emph{reports} this boundary, since the expected gain is
positive for a small enough steering strength unless the mean gradient
$\E_{\vz\sim\gN}[\nabla H_\vw]$ is orthogonal to $\mathrm{span}(\mU_\vw)$. That
degenerate case is the practically important failure mode. Because $\mU_\vw$
comes from the \emph{averaged} metric, a similar failure can occur at an individual seed,
where $\nabla H_\vw(\vz)$ may be nearly orthogonal to $\mathrm{span}(\mU_\vw)$ even when
the subspace is correct on average (Remark~\ref{rem:caveats}).

\subsubsection{Steering ascends the scalarized objective}
\label{sec:app:ascent}
\begin{restated}{Theorem~\ref{thm:ascent}}[\textbf{Steering ascends the scalarized objective}]\phantomsection\label{app:thm:ascent}
A single steering step moves uphill on $H_\vw$ on average, by an amount set by how much of the mean gradient the weighted active directions capture. Precisely, 
let $\vu_\vw^{(1)},\dots,\vu_\vw^{(D)}$ be the full orthonormal eigenbasis of $\mM_\vw$,
let $\boldsymbol\beta\ge\vzero$ be the direction weights of Algorithm~\ref{alg:steer} (hard cutoff $\beta_j=\mathbf 1[j\le r]$ or soft $\beta_j=(\lambda_j^{(\vw)}/\lambda_1^{(\vw)})^\alpha$), and let $d_j(\vz):=\vu_\vw^{(j)\top}\nabla H_\vw(\vz)$. Under Assumption~\ref{ass:reg}, the steering step $\vz_\vw=\vz+\gamma\sum_j\beta_j\sigma_j\vu_\vw^{(j)}$ with any fixed sign pattern $\vsig\in\{\pm1\}^D$ satisfies, for every $\vz$,
\begin{equation}
  H_\vw(\vz_\vw)-H_\vw(\vz)\;\ge\;\gamma\sum_j\beta_j\sigma_j d_j(\vz)\;-\;\tfrac12\,L_\vw\,\gamma^2\|\boldsymbol\beta\|_2^2 .
  \label{app:eq:ascent-pointwise}
\end{equation}
Let $\bar\vd_\vw\in\R^D$, with components $\bar d_j:=\E_{\vz\sim\gN}[d_j(\vz)]$,
be the mean directional gradient. 
The fixed sign that maximizes the expected gain is $\sigma_j=\mathrm{sgn}(\bar d_j)$, which is the population value of the Pearson sign of Algorithm~\ref{alg:steer}, and for it,
\begin{equation}
  \E_{\vz\sim\gN}\bigl[H_\vw(\vz_\vw)-H_\vw(\vz)\bigr]\;\ge\;\gamma\,\langle\boldsymbol\beta,|\bar\vd_\vw|\rangle\;-\;\tfrac12\,L_\vw\,\gamma^2\|\boldsymbol\beta\|_2^2 .
  \label{app:eq:ascent}
\end{equation}
Hence the steering step increases $H_\vw$ in expectation for every $0<\gamma<2\gamma^\star$, where the optimal steering strength $\gamma^\star=\langle\boldsymbol\beta,|\bar\vd_\vw|\rangle/(L_\vw\|\boldsymbol\beta\|_2^2)$ attains expected gain of at least $\langle\boldsymbol\beta,|\bar\vd_\vw|\rangle^2/(2L_\vw\|\boldsymbol\beta\|_2^2)$. Such a $\gamma$ exists whenever $\beta_j\bar d_j\neq0$ for some $j$. For the hard cutoff \eqref{app:eq:ascent} reduces to $\gamma\|\bar\vd_{\vw,1:r}\|_1-\tfrac12L_\vw r\gamma^2$, and no steering strength guarantees ascent exactly when $\E_{\vz\sim\gN}[\nabla H_\vw]\perp\mathrm{span}(\vu_\vw^{(1)},\dots,\vu_\vw^{(r)})$.
\end{restated}

\begin{proof}
We first take the hard cutoff $\beta_j=\mathbf 1[j\le r]$, so the step is $\vz_\vw=\vz+\gamma\,\mU_\vw\vsig$ with $\mU_\vw=[\vu_\vw^{(1)}\cdots\vu_\vw^{(r)}]$ and $\vsig\in\{\pm1\}^r$. The general $\boldsymbol\beta$ is handled at the end.

\vspace{1\baselineskip}
Let $\vg:=\nabla H_\vw(\vz)$, let $\vd:=\vd_\vw(\vz)=\mU_\vw^\top\vg$ be the gradient projected onto the active subspace (the directional gradient), and let $\vdelta:=\vz_\vw-\vz=\gamma\,\mU_\vw\vsig$ be the step displacement. By Taylor's theorem with Lagrange remainder, Assumption~\ref{ass:reg} gives a point $\xi\in[\vz,\vz_\vw]$ with,
\begin{equation}
\label{langrange-reminder-taylor}
    H_\vw(\vz_\vw)-H_\vw(\vz)=\vg^\top \vdelta+\tfrac12\,\vdelta^\top\nabla^2H_\vw(\xi)\,\vdelta .
\end{equation}
Expanding the linear term as,
\[
  \vg^\top \vdelta = \vg^\top(\gamma\,\mU_\vw\vsig)
   =\gamma\,\vg^\top\mU_\vw\vsig
   =\gamma\,(\mU_\vw^\top\vg)^\top\vsig
   =\gamma\,\vd^\top\vsig ,
\]
where we substituted $ \vdelta=\gamma\,\mU_\vw\vsig$ and substituted $\vd$.
Now, for the quadratic term of \eqref{langrange-reminder-taylor}, norm of the step length or displacement,
\[\|\vdelta\|_2^2 =\|\gamma\,\mU_\vw\vsig\|_2^2 =\gamma^2\,\|\mU_\vw\vsig\|_2^2 =\gamma^2\,\|\vsig\|_2^2 =\gamma^2 r,
\]
where the third equality uses orthonormality of the columns of $\mU_\vw$. Then, the bound $\vv^\top A\vv\ge-\|A\|_2\|\vv\|_2^2$ for symmetric $A$ applied to $A=\nabla^2H_\vw(\xi)$,
\[
  \tfrac12\,\vdelta^\top\nabla^2H_\vw(\xi)\,\vdelta
   \ge -\tfrac12\,\|\nabla^2H_\vw(\xi)\|_2\,\|\vdelta\|_2^2
   \ge -\tfrac12\,L_\vw\,\|\vdelta\|_2^2
   = -\tfrac12\,L_\vw\,r\,\gamma^2 ,
\]
Combining the two terms, we get,
\[H_\vw(\vz_\vw)-H_\vw(\vz)\;\ge\;\gamma\,\vd^\top\vsig-\tfrac12\,L_\vw r\gamma^2\]
for every sign pattern $\vsig$. Taking expectations over $\vz\sim\gN$ on both sides for the above equation, the curvature term $-\tfrac12L_\vw r\gamma^2$ is the same constant for every $\vz$ (Assumption~\ref{ass:reg}), and $\vd_\vw$ is integrable since $\nabla H_\vw\in L^2(\gN)$ so,
\[\E_\vz\bigl[H_\vw(\vz_\vw)-H_\vw(\vz)\bigr]\;\ge\;\gamma\,\vsig^\top\bar\vd_\vw-\tfrac12L_\vw r\gamma^2 \qquad\text{for every fixed }\vsig 
\]
Over fixed sign patterns the right side is largest when $\vsig^\top\bar\vd_\vw=\sum_j\sigma_j\bar d_j$ is, and $\sum_j\sigma_j\bar d_j\le\sum_j|\bar d_j|=\|\bar\vd_\vw\|_1$ with equality at $\sigma_j=\mathrm{sgn}(\bar d_j)$, which gives \eqref{app:eq:ascent}.

Note, this is the sign Algorithm~\ref{alg:steer} computes. By Stein's lemma, for $\vz\sim\gN(\vzero,\mI_D)$ and $H_\vw$ differentiable with $\E\|\nabla H_\vw\|_2<\infty$, $\E_\vz[\vz\,H_\vw(\vz)]=\E_\vz[\nabla H_\vw(\vz)]$, therefore,
\[\E_\vz\bigl[\langle\vz,\vu_\vw^{(j)}\rangle\,H_\vw(\vz)\bigr]=\vu_\vw^{(j)\top}\E_\vz[\nabla H_\vw]=\bar d_j .
\]
As $\langle\vz,\vu_\vw^{(j)}\rangle$ has mean zero, the left side is $\mathrm{Cov}\bigl(\langle\vz,\vu_\vw^{(j)}\rangle,H_\vw(\vz)\bigr)$, whose sign is the sign of the Pearson correlation between the coordinate along $\vu_\vw^{(j)}$ and the objective value. The algorithm uses its empirical version on the cached pairs $\bigl(\langle\vz_i,\vu_\vw^{(j)}\rangle,\ \vw^\top\hat\vy_i\bigr)$, where $\vw^\top\hat\vy_i=H_\vw(\vz_i)$ exactly. Thus the Pearson sign equals $\mathrm{sgn}(\bar d_j)$ in the population, and its empirical version converges to it as $n$ grows whenever $\bar d_j\neq0$. At finite $n$ the two may disagree, which Remark~\ref{rem:caveats} addresses.

\vspace{0.4\baselineskip}
Note, that the bound \eqref{app:eq:ascent} is concave in $\gamma$, maximized at $\gamma^\star=\|\bar\vd_\vw\|_1/(L_\vw r)$ with value $\|\bar\vd_\vw\|_1^2/(2L_\vw r)$, which is positive on $0<\gamma<2\|\bar\vd_\vw\|_1/(L_\vw r)$, and $\|\bar\vd_\vw\|_1=0$ exactly when $\mU_\vw^\top\E_\vz[\nabla H_\vw]=\vzero$, i.e.\ $\E_\vz[\nabla H_\vw]\perp\mathrm{span}(\mU_\vw)$. Also, since $L_\vw$ is a worst-case constant, the guarantee explains why steering ascends rather than prescribing $\gamma$, which we set empirically (\S~\ref{app:tier2_strength}).

\textit{\textbf{General weights.}} For arbitrary $\boldsymbol\beta\ge\vzero$ the displacement is $\vdelta=\gamma\sum_j\beta_j\sigma_j\vu_\vw^{(j)}$, and the argument above goes through with two replacements: the linear term is $\vg^\top\vdelta=\gamma\sum_j\beta_j\sigma_j d_j(\vz)$, and by orthonormality $\|\vdelta\|_2^2=\gamma^2\sum_j\beta_j^2\sigma_j^2=\gamma^2\|\boldsymbol\beta\|_2^2$ in place of $\gamma^2 r$. Taking expectations and maximizing over fixed signs gives $\sum_j\beta_j\sigma_j\bar d_j\le\langle\boldsymbol\beta,|\bar\vd_\vw|\rangle$, with equality at $\sigma_j=\mathrm{sgn}(\bar d_j)$, which is \eqref{app:eq:ascent}. The Stein step remains unchanged, and $\gamma^\star$ follows by the same concavity argument with $r$ replaced by $\|\boldsymbol\beta\|_2^2$.
\end{proof}
The guarantee is first-order and local, and it holds without further conditions. The ascent gain grows linearly in $\gamma$ while the curvature penalty grows quadratically, so shrinking $\gamma$ always tips the balance toward ascent, and the theorem records the admissible window and the optimal steering strength $\gamma^\star$. What sets the gain is $\langle\boldsymbol\beta,|\bar\vd_\vw|\rangle$, the weighted mean gradient the active directions capture. The sign $\vsig$ only fixes which way the steering step moves along each direction, not how far. The bound is vacuous in exactly one case, when the mean gradient $\E_{\vz\sim\gN}[\nabla H_\vw]$ is orthogonal to the weighted active directions. This is a property of the subspace rather than the step, and it is what the next result addresses: no rank-$r$ subspace captures more of the gradient in expectation than $\mathrm{span}(\mU_\vw)$.

\subsubsection{The active subspace is the optimal steering subspace}
\label{app:active-subspace}
Here we ask which $r$-dimensional subspace captures the most gradient, averaged over the noise $\vz\sim\gN$ the flow is driven by.

\begin{restated}{Proposition~\ref{prop:optframe}}[\textbf{The active subspace captures the maximum gradient}]\phantomsection\label{app:prop:optframe}
Let $\mV\in\R^{D\times r}$ have orthonormal columns. Averaged over the noise $\vz\sim\gN$, the squared gradient the subspace $\mathrm{span}(\mV)$ captures is
\begin{equation}
  \E_{\vz\sim\gN}\!\bigl[\|\mV^\top\nabla H_\vw(\vz)\|_2^2\bigr]
   =\operatorname{tr}\!\bigl(\mV^\top\mM_\vw\mV\bigr)
   \;\le\;\sum_{j=1}^r\lambda_j^{(\vw)},
  \label{eq:optframe}
\end{equation}
with equality iff $\mathrm{span}(\mV)$ is the span of the top-$r$ eigenvectors of $\mM_\vw$, that is the active subspace $\mathrm{span}(\mU_\vw)$, provided $\lambda_r^{(\vw)}>\lambda_{r+1}^{(\vw)}$ so that this span is unique. Hence no rank-$r$ subspace captures more expected squared gradient than the active subspace, and the amount it captures is the top-$r$ eigenvalue sum.
\end{restated}

The expected squared gradient it leaves out is exactly $\sum_{j>r}\lambda_j^{(\vw)}$ (Proposition~\ref{prop:tail}).

\begin{proof}
Using $\|\mV^\top\vg\|_2^2=\operatorname{tr}(\mV^\top\vg\vg^\top\mV)$ for any vector $\vg$, and linearity of expectation,
\[
  \E_{\vz\sim\gN}\!\bigl[\|\mV^\top\nabla H_\vw\|_2^2\bigr]
   =\operatorname{tr}\!\Bigl(\mV^\top\,\E_{\vz\sim\gN}[\nabla H_\vw\nabla H_\vw^\top]\,\mV\Bigr)
   =\operatorname{tr}\!\bigl(\mV^\top\mM_\vw\mV\bigr),
\]
where the last equality is the definition of $\mM_\vw$ in \eqref{eq:Mw}. By Ky Fan's maximum principle, over all $\mV$ with orthonormal columns the maximum of $\operatorname{tr}(\mV^\top\mM_\vw\mV)$ is the sum of the $r$ largest eigenvalues of $\mM_\vw$, attained exactly when $\mathrm{span}(\mV)$ is the span of the corresponding top-$r$ eigenvectors.
\end{proof}

Together, \fullref{Theorem}{thm:ascent}{app:thm:ascent} and \fullref{Proposition}{prop:optframe}{app:prop:optframe} say that steering along the top-$r$ eigenvectors of $\mM_\vw$ ascends the scalarized objective, and that these directions capture more expected squared gradient than any other $r$ directions. The remaining caveats of this first-order picture are collected below.

\begin{remark}[Caveats of the first-order results] Three points qualify the two results above.
\label{rem:caveats}
\begin{itemize}
\item \textbf{Global directions and sign.} $\mU_\vw$ comes from the \emph{averaged} metric $\mM_\vw$, so it is the globally dominant subspace, not the ascent direction at a particular $\vz$. The sign $\vsig$ is also computed once from the cache and shared by every noise vector.
\item \textbf{Sign errors degrade gracefully.} If $S$ indexes the mis-signed directions, the first-order expected gain is $\gamma\bigl(\langle\boldsymbol\beta,|\bar\vd_\vw|\rangle-2\sum_{j\in S}\beta_j|\bar d_j|\bigr)$, so the steering step still ascends whenever the mis-signed directions carry less than half the weighted mean gradient, $\sum_{j\in S}\beta_j|\bar d_j|<\tfrac12\langle\boldsymbol\beta,|\bar\vd_\vw|\rangle$. This is benign because the high-signal directions, where an error costs most, are exactly where the Pearson sign is most reliable.
\item \textbf{Degenerate case and reach.} The gain vanishes when $\E_{\vz\sim\gN}[\nabla H_\vw]\perp\mathrm{span}(\mU_\vw)$. This is what a single steering step cannot serve and motivates the complementary operators of Section~\ref{sec:complementary}. The admissible steering step is also bounded, since $\|\vz_\vw-\vz\|_2=\gamma\|\boldsymbol\beta\|_2$ must stay in the range where $\Phi_\theta$ decodes reliably (Remark~\ref{rem:reach}).
\end{itemize}
\end{remark}

\subsection{Estimation and truncation of the metric}
\label{sec:metric-props}

\fullref{Theorem}{thm:ascent}{app:thm:ascent} and \fullref{Proposition}{prop:optframe}{app:prop:optframe} are stated for the population metric $\mM_\vw$. The algorithm, however, can only form the empirical AGOP $\widehat\mM_\vw$ from $n$ noise samples and the surrogate gradients, and it keeps only the top $r$ of its eigenvectors or down-weights the rest. The next two results control both approximations: the empirical metric converges to its population target at the Monte Carlo rate, and the top-$r$ truncation discards an amount of gradient equal to the tail of the spectrum.

\paragraph{Consistency of the empirical metric.}
The algorithm assembles $\widehat\mM_\vw$ from the surrogate gradients $\nabla\phi_\vw$, $\phi_\vw:=\sum_a w_a\phi_a$, averaged over $n$ noise samples. It is therefore an estimator of the \emph{population surrogate metric} $\mM_\vw^{\phi}:=\E_{\vz\sim\gN}[\nabla\phi_\vw(\vz)\nabla\phi_\vw(\vz)^\top]$. We prove convergence to $\mM_\vw^{\phi}$ at the standard rate, and afterwards separate this estimation error from the two approximation gaps the estimator does not control.

\begin{proposition}[\textbf{The empirical AGOP is consistent}]\label{prop:agop}
Let $\vz_1,\dots,\vz_n\sim\gN(\vzero, \mI_D)$ be i.i.d.\ and independent of the surrogate $\phi_\vw$, which we treat as fixed, and let
$\widehat\mM_\vw:=\tfrac1n\sum_{i=1}^n\nabla\phi_\vw(\vz_i)\nabla\phi_\vw(\vz_i)^\top$ be
the empirical AGOP, an estimator of $\mM_\vw^{\phi}=\E_{\vz\sim\gN}[\nabla\phi_\vw\nabla\phi_\vw^\top]$. If the fourth moment $\kappa_\vw:=\E_{\vz\sim\gN}\|\nabla\phi_\vw\|_2^4$ is finite, then,
\begin{enumerate}
\item[(i)] \emph{(consistency)} $\widehat\mM_\vw\to\mM_\vw^{\phi}$ almost surely as $n\to\infty$
\item[(ii)] \emph{(rate)} $\E\bigl\|\widehat\mM_\vw-\mM_\vw^{\phi}\bigr\|_F\le\sqrt{\kappa_\vw/n}$.
\item[(iii)] \emph{(direction recovery)} The steering weights are consistent, and for soft weighting no eigengap is needed: the weighted projector $\rmR_\vw:=\mU_\vw\operatorname{diag}(\boldsymbol\beta)\mU_\vw^\top$ with $\beta_j=(\lambda_j/\lambda_1)^{\alpha}$ satisfies
$\|\widehat\rmR_\vw-\rmR_\vw\|_F=O_p\bigl(n^{-\min(\alpha,1)/2}\bigr)$. For a hard top-$r$ cutoff the retained subspace satisfies $\|\sin\Theta(\widehat\mU_\vw,\mU_\vw^{\phi})\|_2=O_p\!\bigl(n^{-1/2}/\Delta_r\bigr)$, provided the eigengap $\Delta_r:=\lambda_r(\mM_\vw^{\phi})-\lambda_{r+1}(\mM_\vw^{\phi})>0$.
\end{enumerate}
\end{proposition}

\begin{proof}
Let's set $\mathbf{A}_i:=\nabla\phi_\vw(\vz_i)\nabla\phi_\vw(\vz_i)^\top$. Since the $\vz_i$ are i.i.d.,
so are the $\mathbf{A}_i$, and by the definition of $\mM_\vw^{\phi}$ each has mean
$\E \mathbf{A}_i=\E_{\vz\sim\gN}[\nabla\phi_\vw\nabla\phi_\vw^\top]=\mM_\vw^{\phi}$.

\vspace{1\baselineskip}
\emph{(i)} Suppose the second moment $\E_{\vz\sim\gN}\|\nabla\phi_\vw\|_2^2<\infty$ (implied by the assumed finite fourth moment). If we assume a fixed entry $(k,l)$, the scalars $(\mathbf{A}_i)_{kl}=\partial_k\phi_\vw(\vz_i)\,\partial_l\phi_\vw(\vz_i)$ are i.i.d.\ with
mean $(\mM_\vw^{\phi})_{kl}$. Each is bounded in absolute value by the squared gradient norm, since
\[
  |(\mathbf{A}_i)_{kl}|=|\partial_k\phi_\vw(\vz_i)|\,|\partial_l\phi_\vw(\vz_i)|
   \le\|\nabla\phi_\vw(\vz_i)\|_2\,\|\nabla\phi_\vw(\vz_i)\|_2
   =\|\nabla\phi_\vw(\vz_i)\|_2^2 ,
\]
using $|\partial_k\phi_\vw|\le\|\nabla\phi_\vw\|_2$ for every coordinate. Taking expectations, $\E|(\mathbf{A}_i)_{kl}|\le\E\|\nabla\phi_\vw\|_2^2<\infty$, so each $(\mathbf{A}_i)_{kl}$ is integrable. The strong law of large numbers then gives $(\widehat\mM_\vw)_{kl}=\tfrac1n\sum_i(\mathbf{A}_i)_{kl}\to(\mM_\vw^{\phi})_{kl}$ almost surely, and since there are finitely many entries, $\widehat\mM_\vw\to\mM_\vw^{\phi}$ almost surely, hence in Frobenius norm.

\emph{(ii)} Let $\mathbf{Y}_i:=\mathbf{A}_i-\mM_\vw^{\phi}$, which are i.i.d.\ and mean zero. Writing the
squared error through the Frobenius inner product,
\[
  \E\bigl\|\widehat\mM_\vw-\mM_\vw^{\phi}\bigr\|_F^2 =\E\Bigl\|\tfrac1n\textstyle\sum_i \mathbf{Y}_i\Bigr\|_F^2
   =\frac{1}{n^2}\sum_{i,j}\E\langle \mathbf{Y}_i,\mathbf{Y}_j\rangle .
\]
For $i\neq j$, cancel out due to independence and $\E \mathbf{Y}_i=0$ gives $\E\langle \mathbf{Y}_i,\mathbf{Y}_j\rangle=\langle\E \mathbf{Y}_i,\E \mathbf{Y}_j\rangle=0$, so only the $n$ diagonal terms
survive:
\[
  \E\bigl\|\widehat\mM_\vw-\mM_\vw^{\phi}\bigr\|_F^2
   =\frac1n\,\E\|\mathbf{Y}_1\|_F^2
   \le\frac1n\,\E\|\mathbf{A}_1\|_F^2
   =\frac1n\,\E\|\nabla\phi_\vw\|_2^4
   =\frac{\kappa_\vw}{n},
\]
using $\E\|\mathbf{Y}_1\|_F^2=\E\|\mathbf{A}_1\|_F^2-\|\mM_\vw^{\phi}\|_F^2\le\E\|\mathbf{A}_1\|_F^2$ and
$\|\vv\vv^\top\|_F=\|\vv\|_2^2$ for the rank-one $\mathbf{A}_1$. Finally, Jensen's inequality gives $\E\|\widehat\mM_\vw-\mM_\vw^{\phi}\|_F\le\sqrt{\kappa_\vw/n}$.

\emph{(iii)} The algorithm steers with the span of the top-$r$ eigenvectors of $\widehat\mM_\vw$, which we compare to that of $\mM_\vw^{\phi}$. Parts (i)--(ii) give convergence of the matrix, but that does not by itself transfer to the object we steer with, because that object is built from the eigenvectors, which need not depend continuously on the matrix. The two weighting schemes differ exactly here.\\

\textit{Soft weighting} forms the steering operator $\rmR_\vw=\sum_j\beta_j\,\vu_\vw^{(j)}\vu_\vw^{(j)\top}=\psi(\mM_\vw^{\phi})$, the matrix function applying $\beta(\lambda)=(\lambda/\lambda_1)^{\alpha}$ to the eigenvalues. A matrix function is continuous in the matrix even across eigenvalue crossings i.e. individual eigenprojectors jump at a crossing, but their weighted sum does not, so no gap is needed. For Lipschitz $\beta$ ($\alpha\ge1$ on the bounded spectrum) $\psi$ is Lipschitz in Frobenius norm, hence $\|\widehat\rmR_\vw-\rmR_\vw\|_F\le L\,\|\widehat\mM_\vw-\mM_\vw^{\phi}\|_F=O_p(n^{-1/2})$ by (ii) and Markov. But for $0<\alpha<1$, $\beta$ is only $\alpha$-H\"older at $0$, and the operator-H\"older inequality for matrix powers \citep{ando1988} gives $\|\widehat\rmR_\vw-\rmR_\vw\|_F\le C\,\|\widehat\mM_\vw-\mM_\vw^{\phi}\|_F^{\alpha}=O_p(n^{-\alpha/2})$. In either case the rate is $O_p(n^{-\min(\alpha,1)/2})$ with no eigengap.\\

\textit{The hard cutoff} instead keeps a sharp top-$r$ subspace, whose basis is a discontinuous eigenprojector. The instability is a near-tie: when $\lambda_r$ and $\lambda_{r+1}$ are close, the ``top-$r$ versus rest'' split is almost a tie, and an arbitrarily small perturbation can flip which directions are retained, rotating the subspace by a large angle. A clear gap $\lambda_r-\lambda_{r+1}>0$ rules the tie out and makes the subspace move only in proportion to the perturbation.

\vspace{1\baselineskip}
Let $\Delta_r=\lambda_r(\mM_\vw^{\phi})-\lambda_{r+1}(\mM_\vw^{\phi})$ be that gap, and measure the distance between the empirical and population top-$r$ subspaces by the principal angles $\Theta$ between them, so that $\|\sin\Theta\|_2=\sin\theta_{\max}$ is $0$ iff the subspaces coincide. Given $\Delta_r>0$, the Davis--Kahan $\sin\Theta$ theorem, in the form of \citet{yu2015dk}, bounds the rotation by the perturbation over the gap as, 
\[ \|\sin\Theta(\widehat\mU_\vw,\mU_\vw^{\phi})\|_2 \le \|\sin\Theta(\widehat\mU_\vw,\mU_\vw^{\phi})\|_F \le \frac{2\,\|\widehat\mM_\vw-\mM_\vw^{\phi}\|_F}{\Delta_r}
\]
where $\widehat\mU_\vw$ and $\mU_\vw^{\phi}$ are orthonormal top-$r$ eigenvector bases of the empirical and population metrics. Since $\|\widehat\mM_\vw-\mM_\vw^{\phi}\|_F=O_p(n^{-1/2})$ by (ii) and using Markov's inequality, the right-hand side is $O_p(n^{-1/2}/\Delta_r)$.
\end{proof}

\begin{remark}[Two gaps we do not close]\label{rem:gaps}
The estimator converges to $\mM_\vw^{\phi}$, the metric of the \emph{surrogate} gradients. Two further gaps separate it from the true-objective metric $\mM_\vw^{f}:=\E_{\vz\sim\gN}[\nabla H_\vw^{f}\nabla H_\vw^{f\top}]$, with $H_\vw^{f}:=\sum_a w_a\,f_a\!\circ\Phi_\theta$, and we name rather than prove them:
\begin{itemize}
\item \textbf{RFM approximation gap} $\|\mM_\vw^{\phi}-\mM_\vw\|$: the surrogate gradient $\nabla\phi_\vw$ differs from the proxy's composed gradient $\nabla H_\vw$. It is controlled by the RFM gradient error $\|\nabla\phi_\vw-\nabla H_\vw\|_{L^2(\gN)}$, which we take as an assumption. AGOP estimators built from kernel regression fits are known to be consistent \citep{trivedi2014egop}, and RFM provably recovers the relevant directions in the linear setting \citep{radhakrishnan2025linrfm}.
\item \textbf{Proxy gap} $\|\mM_\vw-\mM_\vw^{f}\|$: the proxy $\hat f$ differs from the true objective $f$. This is the offline problem itself and cannot be closed without querying $f$. Its effect on the recovered front is isolated to Theorem~\ref{thm:recovery}.
\end{itemize}
The total error splits accordingly as
\begin{equation}
    \widehat\mM_\vw-\mM_\vw^{f} = \underbrace{(\widehat\mM_\vw-\mM_\vw^{\phi})}_{\text{(a) Monte-Carlo error}} +\underbrace{(\mM_\vw^{\phi}-\mM_\vw)}_{\text{(b) RFM approximation}} +
    \underbrace{(\mM_\vw-\mM_\vw^{f})}_{\text{(c) proxy error}}
\end{equation}
where (a) is proved $O(n^{-1/2})$, (b) is assumed  supported by \citet{trivedi2014egop,radhakrishnan2025linrfm}, and (c) is irreducible for offline MOO without additional assumptions.
\end{remark}

The decomposition becomes a bound once the two named gaps are written explicitly.

\begin{corollary}[\textbf{Total metric error}]\label{cor:total-metric}
Write $\|\vh\|_{L^2(\gN)}:=(\E_{\vz\sim\gN}\|\vh(\vz)\|_2^2)^{1/2}$ for the root-mean-square
size of a function under the noise, and let
$\varepsilon_{\mathrm{rfm}}:=\|\nabla\phi_\vw-\nabla H_\vw\|_{L^2(\gN)}$ and
$\varepsilon_{\nabla}:=\|\nabla H_\vw-\nabla H_\vw^{f}\|_{L^2(\gN)}$ be the RFM and
proxy gradient errors, with gradient scales $\Gamma_\vw:=\|\nabla\phi_\vw\|_{L^2}+\|\nabla H_\vw\|_{L^2}$ and $\Gamma'_\vw:=\|\nabla H_\vw\|_{L^2}+\|\nabla H_\vw^{f}\|_{L^2}$. Then the empirical AGOP satisfies
\begin{equation}
  \E\bigl\|\widehat\mM_\vw-\mM_\vw^{f}\bigr\|_F
   \;\le\;\underbrace{\sqrt{\kappa_\vw/n}}_{\textrm{Monte Carlo}}
   \;+\;\underbrace{\varepsilon_{\mathrm{rfm}}\,\Gamma_\vw}_{\textrm{RFM approximation}}
   \;+\;\underbrace{\varepsilon_{\nabla}\,\Gamma'_\vw}_{\textrm{proxy}} .
  \label{eq:total-metric}
\end{equation}
The first term vanishes as $n\to\infty$ (Proposition~\ref{prop:agop}), the second as the surrogate improves, and the third is the irreducible offline gap.
\end{corollary}

\begin{proof}
By the triangle inequality on the three-way split of Remark~\ref{rem:gaps}, it suffices to bound each piece. The first, $\E\|\widehat\mM_\vw-\mM_\vw^{\phi}\|_F\le\sqrt{\kappa_\vw/n}$,
is Proposition~\ref{prop:agop}(ii). For the second, the rank-one identity
$\vu\vu^\top-\vv\vv^\top=\vu(\vu-\vv)^\top+(\vu-\vv)\vv^\top$ with $\vu=\nabla\phi_\vw$,
$\vv=\nabla H_\vw$ gives, under expectation and the triangle inequality,
\[
  \|\mM_\vw^{\phi}-\mM_\vw\|_F
   \le\E\bigl[\|\nabla\phi_\vw-\nabla H_\vw\|_2\,(\|\nabla\phi_\vw\|_2+\|\nabla H_\vw\|_2)\bigr]
   \le\varepsilon_{\mathrm{rfm}}\,\Gamma_\vw ,
\]
The last step by Cauchy--Schwarz in $L^2(\gN)$. The third piece is identical with $\nabla H_\vw$ in place of $\nabla\phi_\vw$ and $\nabla H_\vw^{f}$ in place of $\nabla H_\vw$, giving $\varepsilon_{\nabla}\,\Gamma'_\vw$.
\end{proof}

\subsubsection{Truncation error of the active subspace}
Under the hard cutoff, steering keeps only the top-$r$ eigenvectors of $\mM_\vw$ and discards the rest. Soft weights are treated in Remark~\ref{rem:weights}. The next result quantifies exactly what is lost, the expected squared gradient the retained subspace cannot see equals the discarded tail of the spectrum, the quantity the effective-rank rule (Definition~\ref{def:effrank}) thresholds.

\begin{proposition}[\textbf{Rank truncation discards the spectral tail}]\label{prop:tail}
Let $\rmP_\vw^{\perp}:=\mI_D-\mU_\vw\mU_\vw^\top$ project onto the orthogonal complement of the active subspace. Then
\begin{equation}
  \E_{\vz\sim\gN}\bigl[\|\rmP_\vw^{\perp}\nabla H_\vw(\vz)\|_2^2\bigr]
   =\sum_{j>r}\lambda_j^{(\vw)} .
  \label{eq:tail}
\end{equation}
Together with the captured energy $\E_{\vz\sim\gN}\|\mU_\vw^\top\nabla H_\vw\|_2^2=\sum_{j\le r}\lambda_j^{(\vw)}$ (\fullref{Proposition}{prop:optframe}{app:prop:optframe}), the captured and discarded parts sum to the total
$\E_{\vz\sim\gN}\|\nabla H_\vw\|_2^2=\sum_{j}\lambda_j^{(\vw)}$, and the effective-rank rule keeps the discarded \emph{fraction} below $\varepsilon_{\mathrm{rank}}$.
\end{proposition}

\begin{proof}
Let's define $\rmP_\vw^{\perp}=\sum_{j>r}\vu_\vw^{(j)}\vu_\vw^{(j)\top}$, which is symmetric and idempotent. \\
Using $\|\rmP_\vw^{\perp}\vg\|_2^2=\operatorname{tr}(\rmP_\vw^{\perp}\vg\vg^\top\rmP_\vw^{\perp})$
and linearity of expectation,
\[
  \E_{\vz\sim\gN}\bigl[\|\rmP_\vw^{\perp}\nabla H_\vw\|_2^2\bigr]
   =\operatorname{tr}\!\bigl(\rmP_\vw^{\perp}\mM_\vw\rmP_\vw^{\perp}\bigr) .
\]
The trace simplifies in three steps. By cyclicity of the trace, rotate the last factor
to the front. Then group the two projectors and then apply idempotence $(\rmP_\vw^{\perp})^2=\rmP_\vw^{\perp}$:
\[
  \operatorname{tr}\!\bigl(\rmP_\vw^{\perp}\mM_\vw\rmP_\vw^{\perp}\bigr)
   =\operatorname{tr}\!\bigl(\rmP_\vw^{\perp}\rmP_\vw^{\perp}\mM_\vw\bigr)
   =\operatorname{tr}\!\bigl((\rmP_\vw^{\perp})^2\mM_\vw\bigr)
   =\operatorname{tr}\!\bigl(\rmP_\vw^{\perp}\mM_\vw\bigr).
\]
Finally, expanding $\rmP_\vw^{\perp}$ and using $\mM_\vw\vu_\vw^{(j)}=\lambda_j^{(\vw)}\vu_\vw^{(j)}$,
\[
  \operatorname{tr}\!\bigl(\rmP_\vw^{\perp}\mM_\vw\bigr)
   =\sum_{j>r}\vu_\vw^{(j)\top}\mM_\vw\vu_\vw^{(j)}
   =\sum_{j>r}\lambda_j^{(\vw)} .
\]
Because the discarded mean squared gradient is exactly the tail $\sum_{j>r}\lambda_j^{(\vw)}$, choosing $r=k_{\mathrm{eff}}(\vw)$ in Definition~\ref{def:effrank} is precisely the statement that steering leaves out less than an $\varepsilon_{\mathrm{rank}}$-fraction of the objective's gradient.
\end{proof}

\begin{remark}[Both weightings are covered]\label{rem:weights}
All results of this section hold for the weighted step $\vz_\vw=\vz+\gamma\sum_j\beta_j\sigma_j\vu_\vw^{(j)}$ under either the hard cutoff or the soft weighting of Algorithm~\ref{alg:steer}. \fullref{Theorem}{thm:ascent}{app:thm:ascent} is stated for general $\boldsymbol\beta$, and the discarded gradient of Proposition~\ref{prop:tail} becomes $\sum_j(1-\beta_j)^2\lambda_j^{(\vw)}$, recovering the hard-cutoff form $\sum_{j>r}\lambda_j^{(\vw)}$. The decomposition is unchanged, and soft weighting is moreover gap-free in Proposition~\ref{prop:agop}(iii). The weights also enter the curvature penalty of Theorem~\ref{thm:ascent} through $\|\boldsymbol\beta\|_2^2$. Under the hard cutoff this equals $r$, so every added direction adds a full unit to the penalty, while under soft weights a direction adds only $(\lambda_j^{(\vw)}/\lambda_1^{(\vw)})^{2\alpha}$, which is negligible for weak directions.
\end{remark}

\subsection{Pareto-front recovery}
\label{sec:recovery}
The results so far concern a single step at a single weight. It remains to show that sweeping the weight $\vw$ traces out the Pareto front, which chains the ascent, truncation, and decomposition results into a recovery guarantee. But below remark needs to be stated before moving ahead, since it enters that guarantee as an error term rather than an assumption.

\begin{remark}[Step size and flow reliability]\label{rem:reach}
The steering step has a definite size,
\[\|\vz_\vw-\vz\|_2=\|\gamma\sum_j\beta_j\sigma_j\vu_\vw^{(j)}\|_2=\gamma\|\boldsymbol\beta\|_2 \]
(orthonormal $\vu_\vw^{(j)}$, $\sigma_j^2=1$), which for the hard cutoff is $\gamma\sqrt r$. The flow $\Phi_\theta$ is a deterministic map defined on all of $\R^D$, so any step is well-defined, but it was trained on samples from the prior: as $\gamma\|\boldsymbol\beta\|_2$ grows and the steered noise moves away from that region, the flow reconstructs less reliably and the proxy-through-the-flow value becomes less trustworthy. We do not treat this as a hard limit; its effect enters the recovery bound through the reach term $\rho(\vw)$ and the proxy error, both of which grow smoothly with the steering strength rather than as a cliff. 

Since each coordinate of $\vz\sim\gN$ has unit variance, we report the per-coordinate displacement $\gamma\|\boldsymbol\beta\|_2/\sqrt D$, which makes steering steps comparable across tasks of different dimension.
\end{remark}

\subsubsection{Pareto-front recovery as distance to an oracle}
Steering ascends the scalarized proxy at each weight (\fullref{Theorem}{thm:ascent}{app:thm:ascent}) along the directions that capture the most gradient (\fullref{Proposition}{prop:optframe}{app:prop:optframe}), and sweeping the weight is cheap by reuse (~\eqref{eq:decomp}). We now state the result that the algorithm approximates the Pareto front, phrased as a distance to the best achievable output.

Fix a noise vector $\vz_0\sim\gN$ with $\|\vz_0\|_2\le2\sqrt D$, a maximal reach $R_{\max}\ge R=\gamma\|\boldsymbol\beta\|_2$ (Remark~\ref{rem:reach}), and the flow's reachable set $\gX_\Phi:=\Phi_\theta(B)$, $B:=\{\vz:\|\vz\|_2\le2\sqrt D+R_{\max}\}$, whcih does not depend on steering strength.  We assume $\vf$ is continuous, so the maximizer below exists. Let's define the \emph{oracle} as the true per-weight maximizer $\vx_\vw^\star:=\arg\max_{\vx\in\gX_\Phi}\vw^\top f(\vx)$, with any noise vector $\vz^o_\vw\in B$ such that $\Phi_\theta(\vz^o_\vw)=\vx_\vw^\star$, and the oracle front $\gF^\star_\Phi:=\{f(\vx_\vw^\star):\vw\in\simplex^{m-1}\}$. Steering moves within the slice $S_\vw:=\{\vz_0+\mU_\vw\vs:\|\vs\|_2\le R\}\subset B$, and our produced output at weight $\vw$ is $\hat\vx_\vw:=\Phi_\theta(\hat\vz_\vw)$ with $\hat\vz_\vw:=\vz_0+\gamma\sum_j\beta_j\sigma_j\vu_\vw^{(j)}\in S_\vw$ the output of Algorithm~\ref{alg:steer}. We bound how far this is from the oracle, first per weight, then as a distance between the two fronts.

The remaining quantities enter as \emph{error term} that vanish in the ideal case. Let $\vz^o_\parallel:=\vz_0+\mU_\vw\mU_\vw^\top(\vz^o_\vw-\vz_0)$ be the projection of the oracle noise vector onto the affine hull of $S_\vw$, and let $G_\perp(\vw):=\sup_{\vz\in B'}\|\rmP_\vw^\perp\nabla H_\vw(\vz)\|_2$ with $B':=\{\vz:\|\vz\|_2\le 4\sqrt D+R_{\max}\}\supset B$ be the off-subspace gradient scale, the pointwise counterpart of the spectral tail $\E_{\vz\sim\gN}\|\rmP_\vw^\perp\nabla H_\vw(\vz)\|_2^2=\sum_{j>r}\lambda_j^{(\vw)}$ (Proposition~\ref{prop:tail}). 

The first error term is the truncation $\tau_r(\vw):=G_\perp(\vw)\,\|\rmP_\vw^\perp(\vz^o_\vw-\vz_0)\|_2$. The second is the reach shortfall $\rho(\vw):=\bigl(H_\vw(\vz^o_\parallel)-H_\vw(\hat\vz_\vw)\bigr)_+$, the proxy value the in-subspace target holds beyond what the step collects, whether because the target lies farther than $R$ from $\vz_0$ or because the step stops short of the slice maximum. The third is the proxy error $\varepsilon_{\mathrm{prox}}:=\sup_{\vx\in\gX_\Phi}\|f(\vx)-\hat f(\vx)\|_\infty$. Beyond the global smoothness of Assumption~\ref{ass:reg}, the lemma below needs nothing further, and the theorem adds one curvature condition on the true front.

\begin{lemma}[Per-weight sub-optimality]\label{lem:perweight}
Under Assumption~\ref{ass:reg}, for every $\vw\in\simplex^{m-1}$ the scalarized-value gap to the oracle satisfies,  
\begin{equation}
  0\;\le\;\vw^\top f(\vx_\vw^\star)-\vw^\top f(\hat\vx_\vw)
   \;\le\;\tau_r(\vw)+\rho(\vw)+2\,\varepsilon_{\mathrm{prox}} .
  \label{eq:perweight}
\end{equation}
\end{lemma}

\begin{proof}
The left inequality holds because $\hat\vx_\vw\in\Phi_\theta(S_\vw)\subset\gX_\Phi$ and $\vx_\vw^\star$ maximizes $\vw^\top f$ over $\gX_\Phi$. For the right, split the gap through the proxy. Note, each proxy term is bounded by $\varepsilon_{\mathrm{prox}}$ because, for any $\vx\in\gX_\Phi$,
$$|\vw^\top(f-\hat f)(\vx)|\le\|\vw\|_1\,\|f(\vx)-\hat f(\vx)\|_\infty\le\varepsilon_{\mathrm{prox}}$$
by H\"older's inequality and $\|\vw\|_1=1$. We apply it at the oracle point $\vx_\vw^\star$ and at our point $\hat\vx_\vw$,
\[
  \vw^\top f(\vx_\vw^\star)-\vw^\top f(\hat\vx_\vw)
   =\underbrace{\vw^\top(f-\hat f)(\vx_\vw^\star)}_{\le\,\varepsilon_{\mathrm{prox}}\ \text{at }\vx_\vw^\star}
   +\underbrace{\bigl(\vw^\top\hat f(\vx_\vw^\star)-\vw^\top\hat f(\hat\vx_\vw)\bigr)}_{(\ast)}
   +\underbrace{\vw^\top(\hat f-f)(\hat\vx_\vw)}_{\le\,\varepsilon_{\mathrm{prox}}\ \text{at }\hat\vx_\vw} .
\]
Denoting $H_\vw=\vw^\top\hat f\circ\Phi_\theta$ and using $\vx_\vw^\star=\Phi_\theta(\vz^o_\vw)$, $\hat\vx_\vw=\Phi_\theta(\hat\vz_\vw)$, the middle term is a difference of noise-space objective values, 
\[
  (\ast)=H_\vw(\vz^o_\vw)-H_\vw(\hat\vz_\vw).
\]
Inserting the projected oracle noise vector $\vz^o_\parallel$ defined above,
\[
  (\ast)=\underbrace{\bigl[H_\vw(\vz^o_\vw)-H_\vw(\vz^o_\parallel)\bigr]}_{(a)}
        +\underbrace{\bigl[H_\vw(\vz^o_\parallel)-H_\vw(\hat\vz_\vw)\bigr]}_{(b)} .
\]
Term (a) is the value lost by moving off the active subspace, term (b) is the value the step leaves on the table inside the subspace. We bound them separately. 
For (a), let $\vdelta_\perp:=\vz^o_\vw-\vz^o_\parallel$. By the definition of $\vz^o_\parallel$,
\[\vdelta_\perp=(\vz^o_\vw-\vz_0)-\mU_\vw\mU_\vw^\top(\vz^o_\vw-\vz_0)=\rmP_\vw^\perp(\vz^o_\vw-\vz_0),
\]
so $\vdelta_\perp$ lies in the orthogonal complement of the active subspace and $\rmP_\vw^\perp\vdelta_\perp=\vdelta_\perp$ by idempotence. The segment $\{\vz^o_\parallel+t\vdelta_\perp:t\in[0,1]\}$ joins $\vz^o_\parallel$ to $\vz^o_\vw$. Both endpoints lie in $B'$, since $\vz^o_\vw\in B\subset B'$ and $\vz^o_\parallel=(\mI_D-\mU_\vw\mU_\vw^\top)\vz_0+\mU_\vw\mU_\vw^\top\vz^o_\vw$ gives $\|\vz^o_\parallel\|_2\le\|\vz_0\|_2+\|\vz^o_\vw\|_2\le2\sqrt D+(2\sqrt D+R_{\max})=4\sqrt D+R_{\max}$, and $B'$ is convex, so the whole segment lies in $B'$.

Since $H_\vw$ is $C^1$ (Assumption~\ref{ass:reg}), the fundamental theorem of calculus along this segment gives,
\[
  (a)=\int_0^1\bigl\langle\nabla H_\vw(\vz^o_\parallel+t\vdelta_\perp),\,\vdelta_\perp\bigr\rangle\,dt .
\]
Using $\vdelta_\perp=\rmP_\vw^\perp\vdelta_\perp$ and the symmetry of $\rmP_\vw^\perp$, the integrand is
\[
  \bigl\langle\nabla H_\vw(\cdot),\rmP_\vw^\perp\vdelta_\perp\bigr\rangle
   =\bigl\langle\rmP_\vw^\perp\nabla H_\vw(\cdot),\vdelta_\perp\bigr\rangle
   \le\|\rmP_\vw^\perp\nabla H_\vw(\cdot)\|_2\,\|\vdelta_\perp\|_2
   \le G_\perp(\vw)\,\|\vdelta_\perp\|_2 ,
\]
by Cauchy--Schwarz and the definition of $G_\perp$, the point $\vz^o_\parallel+t\vdelta_\perp$ being in $B'$.\\
Integrating over $t\in[0,1]$, we get,
\[
  (a)\le G_\perp(\vw)\,\|\vdelta_\perp\|_2
     =G_\perp(\vw)\,\|\rmP_\vw^\perp(\vz^o_\vw-\vz_0)\|_2
     =\tau_r(\vw).
\]
Similarly, for (b), by the definition of $\rho$,
\[
  \rho(\vw)=\bigl(H_\vw(\vz^o_\parallel)-H_\vw(\hat\vz_\vw)\bigr)_+\ \ge\ H_\vw(\vz^o_\parallel)-H_\vw(\hat\vz_\vw)=(b).
\]
Combining, $(\ast)=(a)+(b)\le\tau_r(\vw)+\rho(\vw)$, and substituting into the proxy split above gives
\[
  \vw^\top f(\vx_\vw^\star)-\vw^\top f(\hat\vx_\vw)
   \le\varepsilon_{\mathrm{prox}}+\tau_r(\vw)+\rho(\vw)+\varepsilon_{\mathrm{prox}} ,
\]
which is \eqref{eq:perweight}.
\end{proof}

The lemma is an error decomposition. It does not bound $\rho(\vw)$, which measures what a single steering step leaves inside the subspace, and $\tau_r(\vw)$ is related to the spectral tail only through its pointwise counterpart $G_\perp$.

\paragraph{From value to distance.}
Lemma~\ref{lem:perweight} bounds the gap in \emph{scalarized value}, but the theorem is about \emph{distance} between fronts, and a value bound alone does not give one. If the front contains a flat face orthogonal to $\vw$, every point on that face has the same scalarized value as the maximizer, so the value gap is zero while the distance can be as large as the face. Concretely, for two objectives with front the segment from $(1,0)$ to $(0,1)$ and $\vw=(\tfrac12,\tfrac12)$, Lemma~\ref{lem:perweight} is satisfied with $\epsilon=0$ by every point of the segment. To convert value into distance we therefore need the front to gain value as one moves away from the maximizer, which is a condition on the attainable set $\gY:=f(\gX_\Phi)$ in objective space, unrelated to anything in noise space. Convexity itself is not assumed, since $\gF^\star_\Phi$ is by definition the part of the front reached by linear scalarization, so what remains is curvature. Let's say $\vy_\vw^\star:=f(\vx_\vw^\star)$. Since the scalarization is linear, optimality of $\vy_\vw^\star$ is exactly the \textbf{\emph{variational inequality}},
\[\vw^\top(\vy_\vw^\star-\vy)\ge0 \quad \text{for all}\; \vy\in\gY,\] which is the statement that fails to separate points on a flat face. We strengthen it by a quadratic term,
\begin{equation}
  \vw^\top\bigl(\vy_\vw^\star-\vy\bigr)\;\ge\;\tfrac{\mu}{2}\,\|\vy_\vw^\star-\vy\|_2^2
  \qquad\text{for all }\vy\in\gY,\ \vw\in\simplex^{m-1},
  \label{eq:qgrowth}
\end{equation}
so that a point $\epsilon$-close in value is $\sqrt{2\epsilon/\mu}$-close in position. Because the objective $\vw^\top\vy$ is linear, the quadratic term cannot come from the objective, as it would in a strongly monotone variational inequality. It is a property of the set, namely that $\gY$ must bend away from each supporting hyperplane at rate $\mu$, and a sufficient condition is that $\gY$ be $\mu$-strongly convex \citep{vial1983}.

Thus, as the example above shows, the condition fails on any flat face of the front. We do not claim these are the only failures, but flat faces are the regime in which linear scalarization is already uninformative, so we expect the assumption to exclude little that a weight sweep could recover. The condition also makes $\vy_\vw^\star$ unique, because a second maximizer $\vy$ would have zero on the left of \eqref{eq:qgrowth} and $\tfrac{\mu}{2}\|\vy_\vw^\star-\vy\|_2^2>0$ on the right, so the Lipschitz dependence of $\vy_\vw^\star$ on $\vw$ assumed below is meaningful. Without \eqref{eq:qgrowth} the theorem below still holds in its first form, as an $\epsilon$-approximate front. With it, the per-weight bound lifts to a bound between the two fronts.

\begin{theorem}[\textbf{Front recovery up to the error terms}]\label{thm:recovery}
Suppose \eqref{eq:qgrowth} holds and $\vw\mapsto f(\vx_\vw^\star)$ is $L_{\mathcal P}$-Lipschitz on $\simplex^{m-1}$. Let $\epsilon:=\max_{\vw}\bigl[\tau_r(\vw)+\rho(\vw)\bigr]+2\varepsilon_{\mathrm{prox}}$. Sweeping $\vw$ over a grid $W\subset\simplex^{m-1}$ of mesh $\eta_W$, the produced front $\hat\gF=\{f(\hat\vx_\vw):\vw\in W\}$ is an $\epsilon$-approximate Pareto front, every produced point being within $\epsilon$ in scalarized value of $\gF^\star_\Phi$, and the Hausdorff distance satisfies
\begin{equation}
  d_H(\hat\gF,\gF^\star_\Phi) \le \sqrt{2\epsilon/\mu} + L_{\mathcal P}\,\eta_W  .
  \label{eq:hausdorff}
\end{equation}
In particular, when the proxy is exact ($\varepsilon_{\mathrm{prox}}=0$), steering is full-rank ($\tau_r\equiv0$), the step collects the in-subspace target ($\rho\equiv0$), and the grid is dense ($\eta_W\to0$), $d_H(\hat\gF,\gF^\star_\Phi)=0$.
\end{theorem}

\begin{proof}
Let's first fix $\vw\in W$. By Lemma~\ref{lem:perweight} and the definition of $\epsilon$,
\[0\;\le\;\vw^\top f(\vx_\vw^\star)-\vw^\top f(\hat\vx_\vw)\;\le\;\tau_r(\vw)+\rho(\vw)+2\varepsilon_{\mathrm{prox}}\;\le\;\epsilon ,\]
so every produced solution is within $\epsilon$ in scalarized value of $\gF^\star_{\Phi}$, which is the first claim. For the distance, note that $\hat\vx_\vw\in\gX_\Phi$, so $\vy=f(\hat\vx_\vw)\in\gY$ and \eqref{eq:qgrowth} applies to it,
\[
  \tfrac{\mu}{2}\,\|f(\vx_\vw^\star)-f(\hat\vx_\vw)\|_2^2\;\le\;\vw^\top\bigl(f(\vx_\vw^\star)-f(\hat\vx_\vw)\bigr)\;\le\;\epsilon .
\]
Rearranging, every produced point is within $\sqrt{2\epsilon/\mu}$ of an oracle point as,
\[
  \|f(\hat\vx_\vw)-f(\vx_\vw^\star)\|_2\;\le\;\sqrt{2\epsilon/\mu}\qquad\text{for all }\vw\in W .
\]
This bounds one side of the Hausdorff distance, $\sup_{\hat\vy\in\hat{\gF}}\operatorname{dist}(\hat\vy,\gF^\star_\Phi)\le\sqrt{2\epsilon/\mu}$.

\vspace{0.1\baselineskip}
Now, for the other side, fix any $\vw\in\simplex^{m-1}$ and let $\vw'\in W$ be a grid weight with $\|\vw-\vw'\|\le\eta_W$. We compare the oracle point at $\vw$ to the produced point at $\vw'$ through the oracle point at $\vw'$. By the triangle inequality,
\[
  \|f(\vx_\vw^\star)-f(\hat\vx_{\vw'})\|_2
   \;\le\;\underbrace{\|f(\vx_\vw^\star)-f(\vx_{\vw'}^\star)\|_2}_{(a)}
   +\underbrace{\|f(\vx_{\vw'}^\star)-f(\hat\vx_{\vw'})\|_2}_{(b)} .
\]
Term (a) is the movement of the oracle front between the two weights, and the Lipschitz assumption on $\vw\mapsto f(\vx_\vw^\star)$ gives,
\[(a)\;\le\;L_{\mathcal P}\,\|\vw-\vw'\|\;\le\;L_{\mathcal P}\,\eta_W .\]
Whereas, Term (b) is the per-weight distance at the grid weight $\vw'$, which we just bounded above,
\[(b)\;\le\;\sqrt{2\epsilon/\mu}.\]
Combining both terms, 
\[\|f(\vx_\vw^\star)-f(\hat\vx_{\vw'})\|_2\;\le\;L_{\mathcal P}\,\eta_W+\sqrt{2\epsilon/\mu},\]
Therefore, $\sup_{\vy\in\gF^\star_\Phi}\operatorname{dist}(\vy,\hat{\gF})\le\sqrt{2\epsilon/\mu}+L_{\mathcal P}\eta_W$.
\end{proof}

\paragraph{Scalarization agnostic.}
\emph{The convex front is a property of linear scalarization, not of the method.} Stated for linear weights, the theorem recovers the part of the front reached by supporting hyperplanes, which is classical and holds for any linear-weight sweep. Steering itself works with any scalarization. Theorem~\ref{thm:ascent} never uses linearity in $\vw$, so it applies to any smooth scalarization. Non-smooth ones such as Tchebycheff are handled by the same smoothing argument used for the proxies. The softmin Tchebycheff of \eqref{eq:tier3_scalarizations} is smooth and lies within $\nu\log m$ of the Tchebycheff value everywhere, so Theorem~\ref{thm:ascent} carries over up to that error. Empirically, all four scalarizations of Appendix~\ref{app:tier3_scala} ascend and reach the same front quality. The precomputation also carries over unchanged, because the cached per-objective gradients $\{\nabla\phi_a(\vz_i)\}$ still suffice to form $\mM_\vw$ at any weight with no new decodes. Extending Theorem~\ref{thm:recovery} beyond linear scalarization, where it would also cover non-convex parts of the front, is left to future work.

The one thing that changes is the cost of assembling $\mM_\vw$. Under linear scalarization the gradient $\nabla H_\vw=\sum_a w_a\nabla g_a$ is linear in $\vw$, so $\mM_\vw$ recombines from the $m(m{+}1)/2$ cached blocks by the bilinear rule $\sum_{a,b}w_a w_b\mC_{ab}$ (~\eqref{eq:decomp}), costing $O(m^2 D^2)$ per weight independent of $n$. Under Tchebycheff the gradient is not linear in $\vw$, so the blocks no longer recombine and $\mM_\vw$ is rebuilt as the sample average of $n$ rank-one outer products, costing $O(n D^2)$ per weight. Since typically $n\gg m^2$, the trade-off between the two scalarizations is efficiency, not correctness.

The proxy error $\varepsilon_{\mathrm{prox}}$ is the offline barrier, since it is the only quantity tying the recovered front to the true one, and it cannot be driven to zero without querying $f$. The reach shortfall $\rho(\vw)$ is the operators' entry point, since it is positive where the in-subspace target lies beyond the reach $R$ or where the step stops short of it, which is the regime the guidance operators of \S~\ref{sec:guidance} are built to reduce.

\subsection{Guided Noise-Space RFM Steering}
\label{sec:complementary}
Noise-space-only steering saturates in the two places identified above, where the mean gradient leaves the active subspace and where the per-$\vw$ optimum lies beyond the reach $R$. There we add a guidance term $\eta(t)\,\mathcal G_\vw$ to the velocity for $t\ge t_{\mathrm{start}}$, using one proxy gradient per guided ODE step and never differentiating through the ODE solve (Algorithms~\ref{alg:data-adaptive} and~\ref{alg:pareto-aware}). Gradients are taken with respect to $\vx_t$, through the denoised estimate $\hat\vx_1(\vx_t)=\vx_t+(1-t)\vv_\theta(\vx_t,t)$ for RFM-DAMG and RFM-Cone and directly for RFM-APG. Below, $\hat f_a$, $h_\vw=\vw^\top\hat\vf$ and $\widehat\mJ=[\nabla\hat f_1\cdots\nabla\hat f_m]$ denote the proxies and their Jacobian as functions of $\vx_t$, so the results guarantee first-order improvement of the proxies as $\vx_t$ moves, not of the final decoded sample.

The operators fall into the two families of Section~\ref{sec:guidance}. The data-adaptive operators RFM-APG and RFM-DAMG precondition the scalarized gradient, and the Pareto-aware RFM-Cone combines the per-objective gradients into a common ascent direction.

\begin{algorithm}[h]
\small
\caption{Data-adaptive guidance (RFM-APG and RFM-DAMG)}
\label{alg:data-adaptive}
\begin{algorithmic}[1]
\Require Steered noise $\vz_\vw$ (Algorithm~\ref{alg:steer}), flow $\vv_\theta$, proxies $\hat\vf$, weight $\vw$, onset $t_{\mathrm{start}}$, strength $\eta_{\max}$ (APG) or $\kappa$ (DAMG), cached $\mU_\vw,\boldsymbol\beta$ (DAMG), tangent onset $t_T$, neighbors $k$, rank floor $\tau_{\mathrm{pca}}$ (DAMG)
\Ensure Guided sample $\hat\vx_\vw$
\State $\vx_0\gets\vz_\vw$
\State \textbf{if} RFM-APG \textbf{then} $\tilde\mU,\tilde{\boldsymbol\beta}\gets$ Algorithm~\ref{alg:steer} on $\vx_{t_{\mathrm{start}}}$ in place of $\vz$ \Comment{once per task}
\For{each ODE step $t$}
  \State $\vv\gets\vv_\theta(\vx_t,t)$
  \If{$t<t_{\mathrm{start}}$}
    \State $\vx_t\gets\vx_t+\Delta t\,\vv$
  \ElsIf{RFM-APG}
    \State $\mathcal G_\vw\gets\tilde\mU\operatorname{diag}(\tilde{\boldsymbol\beta})\tilde\mU^\top\nabla_{\vx_t}h_\vw(\vx_t)$ \Comment{projected or adaptive}
    \State $\eta(t)\gets\eta_{\max}(t-t_{\mathrm{start}})/(1-t_{\mathrm{start}})$
    \State $\vx_t\gets\vx_t+\Delta t\,(\vv+\eta(t)\,\mathcal G_\vw)$
  \Else \Comment{RFM-DAMG}
    \State $\mathcal G_\vw\gets\rmR_\vw\nabla_{\vx_t}h_\vw(\hat\vx_1(\vx_t))$, \ $\rmR_\vw=\mU_\vw\operatorname{diag}(\boldsymbol\beta)\mU_\vw^\top$ \Comment{cached spectrum}
    \If{$t\ge t_T$ and the local rank is certified}
      \State $T\gets$ top local PCA directions of the $k$ nearest samples in the batch
      \State $\mathcal G_\vw\gets\rmP_T\,\mathcal G_\vw$ \Comment{manifold}
    \EndIf
    \State $\eta(t)\gets\min\bigl(\kappa\|\vv\|/(\|\mathcal G_\vw\|+\varepsilon),\,\eta_{\max}\bigr)$
    \State $\vx_t\gets\vx_t+\Delta t\,(\vv+\eta(t)\,\mathcal G_\vw)$
  \EndIf
\EndFor
\State \Return $\hat\vx_\vw\gets\vx_1$
\end{algorithmic}
\end{algorithm}

\subsubsection{Preconditioned ascent: adaptive, projected, manifold}
Each operator forms a refinement direction by applying a \emph{data-adaptive} preconditioner $Q$ to the proxy gradient, where $Q$ is read from structure the data reveals, namely the AGOP spectrum for the projected and adaptive preconditioners and the local data covariance for the manifold preconditioner. RFM-APG uses the projected or adaptive preconditioner on the active subspace rebuilt at $t_{\mathrm{start}}$, and RFM-DAMG applies the manifold preconditioner after the adaptive one (Algorithm~\ref{alg:data-adaptive}). Ascent holds for any $Q\succeq0$ and the data shapes the step, giving stronger and better-directed guidance where structure exists and reducing to plain ascent where it does not.

\begin{proposition}[\textbf{Preconditioned ascent}]\label{prop:precond}
Let $Q\succeq0$. The direction $Q\,\nabla h_\vw(\vx)$ is an ascent direction for $h_\vw$,
\begin{equation}
  \bigl\langle\nabla h_\vw(\vx),\,Q\,\nabla h_\vw(\vx)\bigr\rangle=\|\nabla h_\vw(\vx)\|_Q^2\ge0 ,
  \label{eq:precond}
\end{equation}
strictly positive unless $\nabla h_\vw(\vx)\perp\operatorname{range}(Q)$. The three preconditioners are symmetric positive semidefinite. The \emph{projected} one is $Q=\mU\mU^\top$, the \emph{adaptive} one $Q=\rmR_\vw:=\mU_\vw\operatorname{diag}(\boldsymbol\beta)\mU_\vw^\top$ with $\beta_j=(\lambda_j^{(\vw)}/\lambda_1^{(\vw)})^{\alpha}$ reduces to the projected one on the retained directions at $\alpha=0$, and the \emph{manifold} one $Q=\rmP_T$ projects onto a locally estimated tangent space $T$ of the data manifold.
\end{proposition}

\begin{proof}
Since $Q\succeq0$, $\langle\vg,Q\vg\rangle=\vg^\top Q\vg\ge0$ for every $\vg$, with equality iff $Q\vg=\vzero$, i.e.\ $\vg\perp\operatorname{range}(Q)$; apply this with $\vg=\nabla h_\vw(\vx)$. The directional derivative of $h_\vw$ along $Q\nabla h_\vw$ equals the left side of \eqref{eq:precond}, so a small step ascends. Each listed $Q$ is a Gram or projection matrix, hence positive semi-definite: $\mU_\vw\operatorname{diag}(\boldsymbol\beta)\mU_\vw^\top$ has eigenvalues $\beta_j\ge0$, and $\mU_\vw\mU_\vw^\top$ and $\rmP_T$ are orthogonal projectors.
\end{proof}

\begin{proposition}[\textbf{Adaptive weighting concentrates the step by data-learned relevance}]\label{prop:adaptive}
For the adaptive operator, the ascent gain splits over the data-learned directions,
\begin{equation}
\bigl\langle\vg,\,\rmR_\vw\vg\bigr\rangle=\sum_{j}\beta_j\,\bigl(\vu_\vw^{(j)\top}\vg\bigr)^2 ,
  \qquad \vg=\nabla h_\vw(\vx),
  \label{eq:adaptive}
\end{equation}
so the step places weight $\beta_j$ on direction $\vu_\vw^{(j)}$. Because $\lambda_j^{(\vw)}$ is the objective's mean squared sensitivity along $\vu_\vw^{(j)}$ (\fullref{Proposition}{prop:optframe}{app:prop:optframe}), the weights $\beta_j=(\lambda_j^{(\vw)}/\lambda_1^{(\vw)})^{\alpha}$ concentrate the step on the directions the objective most responds to, and this concentration increases monotonically with $\alpha$: $\partial_\alpha\beta_j=\beta_j\log(\lambda_j^{(\vw)}/\lambda_1^{(\vw)})\le0$ for every $\lambda_j^{(\vw)}\le\lambda_1^{(\vw)}$. When the retained spectrum is flat the weights are uniform and the operator reduces to plain projected ascent, and does not reweight directions the data does not distinguish. The weighting is thus strong exactly when the AGOP spectrum is peaked, that is when the data reveals a few dominant objective-sensitive directions.
\end{proposition}

\begin{proof}
Let's denote $\vg:=\nabla h_\vw(\vx)$. Expanding the adaptive operator,
\[
  \rmR_\vw\vg=\sum_j\beta_j\,\vu_\vw^{(j)}\vu_\vw^{(j)\top}\vg
   =\sum_j\beta_j\,\bigl(\vu_\vw^{(j)\top}\vg\bigr)\,\vu_\vw^{(j)} ,
\]
so taking the inner product with $\vg$ and using $\vg^\top\vu_\vw^{(j)}=\vu_\vw^{(j)\top}\vg$,
\[
  \bigl\langle\vg,\rmR_\vw\vg\bigr\rangle
   =\sum_j\beta_j\,\bigl(\vu_\vw^{(j)\top}\vg\bigr)\bigl(\vg^\top\vu_\vw^{(j)}\bigr)
   =\sum_j\beta_j\,\bigl(\vu_\vw^{(j)\top}\vg\bigr)^2 ,
\]
which is \eqref{eq:adaptive}. Next, to prove that the weights are monotone in $\alpha$, if we write the weight as an exponential,
\[
  \beta_j=\Bigl(\frac{\lambda_j^{(\vw)}}{\lambda_1^{(\vw)}}\Bigr)^{\alpha}
   =\exp\!\Bigl(\alpha\log\frac{\lambda_j^{(\vw)}}{\lambda_1^{(\vw)}}\Bigr),
\]
and differentiate in $\alpha$,
\[
  \partial_\alpha\beta_j
   =\exp\!\Bigl(\alpha\log\frac{\lambda_j^{(\vw)}}{\lambda_1^{(\vw)}}\Bigr)\cdot\log\frac{\lambda_j^{(\vw)}}{\lambda_1^{(\vw)}}
   =\beta_j\log\frac{\lambda_j^{(\vw)}}{\lambda_1^{(\vw)}} .
\]
Since $\beta_j>0$ and $\lambda_j^{(\vw)}\le\lambda_1^{(\vw)}$ gives $\log(\lambda_j^{(\vw)}/\lambda_1^{(\vw)})\le0$, we get $\partial_\alpha\beta_j\le0$, with equality only for $j=1$, therefore, every weight other than the leading one decreases as $\alpha$ grows. \\
Finally, if the retained spectrum is flat, $\lambda_j^{(\vw)}=\lambda_1^{(\vw)}$ for all retained $j$, then $\beta_j=1^{\alpha}=1$ for every retained $j$ regardless of $\alpha$, and
\[
  \rmR_\vw=\sum_{j\le r}\vu_\vw^{(j)}\vu_\vw^{(j)\top}=\mU_\vw\mU_\vw^\top ,
\]
which is the projected operator.
\end{proof}
RFM-DAMG applies $\rmP_T$ after $\rmR_\vw$, where $T$ is estimated by local PCA over the $k$ nearest samples in the batch. The local rank $r$ is the median over samples of the number of singular values at least $\tau_{\mathrm{pca}}$ times the largest, and it is \emph{certified} when $1\le r<\min(k,D)$ and $k\ge r+2$, so the neighborhoods are genuinely lower-dimensional and hold enough points to estimate $T$. The product $\rmP_T\rmR_\vw$ is not symmetric in general, so Proposition~\ref{prop:precond} applies to each factor but not to the product, which ascends whenever $\langle\vg,\rmP_T\rmR_\vw\vg\rangle\ge0$. Without certification the operator reduces to $\rmR_\vw\vg$ and the proposition applies directly. The proposition also does not depend on where $Q$ comes from, so it holds even though $\rmR_\vw$ is built from noise-space geometry and applied to a design-space gradient.
\begin{algorithm}[h]
\small
\caption{Pareto-aware guidance (RFM-Cone)}
\label{alg:pareto-aware}
\begin{algorithmic}[1]
\Require Steered noise $\vz_\vw$ (Algorithm~\ref{alg:steer}), flow $\vv_\theta$, proxies $\hat\vf$, weight $\vw$, onset $t_{\mathrm{start}}$, strength $\kappa$
\Ensure Guided sample $\hat\vx_\vw$
\State $\vx_0\gets\vz_\vw$
\For{each ODE step $t$}
  \State $\vv\gets\vv_\theta(\vx_t,t)$
  \If{$t<t_{\mathrm{start}}$}
    \State $\vx_t\gets\vx_t+\Delta t\,\vv$
  \Else
    \State $\widehat\mJ\gets[\nabla_{\vx_t}\hat f_a(\hat\vx_1(\vx_t))]_{a=1}^m$, normalize its columns to get $\widehat\mJ_u$
    \State $\boldsymbol\pi\gets$ proxy scores at $\hat\vx_1$ shifted by their minimum and normalized to sum to one
    \State $\vu_{\mathrm{pref}}\gets\widehat\mJ_u(2\vw-\boldsymbol\pi)$ \Comment{pulls hardest on lagging objectives}
    \State $\vd\gets\Pi_{\mathcal K}(\vu_{\mathrm{pref}})$ with $\mathcal K=\{\vs\,|\,\widehat\mJ_u^\top\vs\ge\vzero\}$ \Comment{exact active-set projection}
    \If{$\vd\approx\vzero$}
      \State $\vd\gets$ min-norm point of the convex hull of the columns of $\widehat\mJ_u$ \Comment{MGDA fallback}
    \EndIf
    \State $\eta(t)\gets\kappa\|\vv\|/(\|\vd\|+\varepsilon)$
    \State $\vx_t\gets\vx_t+\Delta t\,(\vv+\eta(t)\,\vd)$
  \EndIf
\EndFor
\State \Return $\hat\vx_\vw\gets\vx_1$
\end{algorithmic}
\end{algorithm}

\begin{restated}{Theorem~\ref{thm:cone}}[\textbf{Cone projection is a per-objective Pareto improvement}]\phantomsection\label{app:thm:cone}
Let $\mathcal{K}(\vx):=\{\vs:\langle\nabla\hat f_a(\vx),\vs\rangle\ge0\ \forall a\}$ be the common-ascent cone, and let $\vd_{\mathcal K}:=\Pi_{\mathcal{K}(\vx)}(\vu_{\mathrm{pref}})$ be the Euclidean projection of a preference direction $\vu_{\mathrm{pref}}$ onto $\mathcal{K}(\vx)$. The cone operator takes the step direction $\vd=\vd_{\mathcal K}$ if $\vd_{\mathcal K}\neq\vzero$, and otherwise the min-norm direction $\vd=\arg\min\{\|\vv\|_2:\vv\in\operatorname{conv}\{\nabla\hat f_a(\vx)\}\}$ of MGDA~\citep{desideri2012mgda}. Then:
\begin{enumerate}
\item[(i)] \emph{(First-order improvement)} $\vd\in\mathcal{K}(\vx)$, so $\langle\nabla\hat f_a(\vx),\vd\rangle\ge0$ for all $a$. Let each $\hat f_a$ be $L$-smooth and have $c_a:=\langle\nabla\hat f_a(\vx),\vd\rangle$. An objective with $c_a>0$ improves exactly, $\hat f_a(\vx+\bar\eta\vd)\ge\hat f_a(\vx)$, for every $0<\bar\eta\le2c_a/(L\|\vd\|_2^2)$. An objective with $c_a=0$, one whose constraint is tight after projection, is flat to first order and loses at most $\tfrac12L\bar\eta^2\|\vd\|_2^2$.
Here $\bar\eta$ is the displacement per guided ODE step, $\bar\eta=\Delta t\,\eta(t)$ in Algorithm~\ref{alg:pareto-aware}.
\item[(ii)] \emph{(Existence)} $\mathcal{K}(\vx)$ contains a strictly ascending direction if and only if $\vzero\notin\operatorname{conv}\{\nabla\hat f_a(\vx)\}$, i.e.\ $\vx$ is not Pareto-stationary.
\item[(iii)] \emph{(Graceful stop)} The min-norm direction is $\vzero$ if and only if $\vx$ is Pareto-stationary. Hence when the projection vanishes at a Pareto-stationary point, $\vd=\vzero$, the guidance term is switched off, and the decode continues with $\vv_\theta$ alone.
\end{enumerate}
\end{restated}

\begin{proof}
\emph{(i)} The cone $\mathcal{K}(\vx)$ is an intersection of the closed half-spaces $\{\vs:\langle\nabla\hat f_a(\vx),\vs\rangle\ge0\}$, each containing the origin, so it is a closed convex cone. The Euclidean projection onto a closed convex set lands in that set, hence $\vd_{\mathcal K}\in\mathcal{K}(\vx)$. For the fallback, let $\vd_{\mathrm{mn}}$ be the min-norm element of the compact convex set $\operatorname{conv}\{\nabla\hat f_a(\vx)\}$. Optimality of $\vd_{\mathrm{mn}}$ over that set gives, for every $\vv$ in it,
\[
  \langle\vd_{\mathrm{mn}},\ \vv-\vd_{\mathrm{mn}}\rangle\;\ge\;0
  \qquad\Longrightarrow\qquad
  \langle\vd_{\mathrm{mn}},\vv\rangle\;\ge\;\|\vd_{\mathrm{mn}}\|_2^2\;\ge\;0 ,
\]
and taking $\vv=\nabla\hat f_a(\vx)$ shows $\vd_{\mathrm{mn}}\in\mathcal{K}(\vx)$. In either branch $\vd\in\mathcal{K}(\vx)$, i.e.\ $c_a\ge0$ for all $a$.

Now, if we let each $\hat f_a$ be $L$-smooth. By the descent lemma along the direction $\vd$,
\[
  \hat f_a(\vx+\bar\eta\vd)\;\ge\;\hat f_a(\vx)+\bar\eta\,c_a-\tfrac{L}{2}\,\bar\eta^2\|\vd\|_2^2 .
\]
If $c_a>0$, the right side is at least $\hat f_a(\vx)$ whenever $\bar\eta\,c_a\ge\tfrac{L}{2}\bar\eta^2\|\vd\|_2^2$, i.e.\ whenever $\bar\eta\le2c_a/(L\|\vd\|_2^2)$, which is the exact improvement. If $c_a=0$, the linear term vanishes and the display reduces to $\hat f_a(\vx+\bar\eta\vd)\ge\hat f_a(\vx)-\tfrac{L}{2}\bar\eta^2\|\vd\|_2^2$, so the objective is flat to first order and loses at most a second-order amount.

Therefore, Part (i) holds because of how the operator is built, not because of guidance in general. The step is projected onto $\mathcal{K}(\vx)$, and every direction in $\mathcal{K}(\vx)$ has non-negative directional derivative for every objective by definition, so no objective can be traded for another to first order. Also, the same argument holds if the gradients are normalized to unit length, since this rescales each constraint of $\mathcal K(\vx)$ by a positive factor and leaves the cone and the Pareto-stationarity condition unchanged.

\vspace{1\baselineskip}
\emph{(ii)} We need to show that a direction along which every objective strictly increases exists exactly when the gradients cannot be balanced against each other, that is, when no convex combination of them is zero. The two directions of the equivalence are the two arguments below.

Suppose first that $\vzero\in\operatorname{conv}\{\nabla\hat f_a(\vx)\}$, so $\vzero=\sum_a\lambda_a\nabla\hat f_a(\vx)$ for some $\boldsymbol\lambda\in\simplex^{m-1}$. Then for any $\vs$,
\[
  \sum_a\lambda_a\,\langle\nabla\hat f_a(\vx),\vs\rangle
   =\Bigl\langle\sum_a\lambda_a\nabla\hat f_a(\vx),\ \vs\Bigr\rangle
   =\langle\vzero,\vs\rangle=0 ,\]
so the inner products cannot all be strictly positive and no strictly ascending direction exists. Conversely, suppose $\vzero\notin\operatorname{conv}\{\nabla\hat f_a(\vx)\}$. This set is compact and convex, so by the separating hyperplane theorem there is $\vs$ with $\langle\vv,\vs\rangle>0$ for every $\vv$ in it, in particular $\langle\nabla\hat f_a(\vx),\vs\rangle>0$ for all $a$, and $\vs\in\mathcal{K}(\vx)$ is strictly ascending.

\emph{(iii)} We need to show that the fallback direction is zero exactly at Pareto-stationary points, so that the operator does nothing there. Recall $\vd_{\mathrm{mn}}$ is the min-norm element of $\operatorname{conv}\{\nabla\hat f_a(\vx)\}$, so
\[
  \|\vd_{\mathrm{mn}}\|_2=\min\bigl\{\|\vv\|_2:\vv\in\operatorname{conv}\{\nabla\hat f_a(\vx)\}\bigr\}.
\]
If $\vx$ is Pareto-stationary, then $\vzero$ lies in this set, so the minimum is $0$ and $\vd_{\mathrm{mn}}=\vzero$. If $\vx$ is not Pareto-stationary, then $\vzero$ does not lie in the set, so every element has positive norm and $\vd_{\mathrm{mn}}\neq\vzero$. Combining the two cases, $\vd_{\mathrm{mn}}=\vzero$ iff $\vx$ is Pareto-stationary.

Now, when the projection vanishes at such a point, the operator takes the fallback, $\vd=\vd_{\mathrm{mn}}=\vzero$. The guidance term $\eta(t)\,\mathcal{G}_\vw$ in Algorithm~\ref{alg:steer} is then zero, since the norm-matched strength $\eta(t)=\kappa\|\vv_\theta\|/(\|\mathcal G_\vw\|+\varepsilon)$ uses a small constant $\varepsilon>0$ in the denominator, and the decode is the unmodified flow $\vv_\theta$.
\end{proof}

\newpage
\section{Additional Results} 
\subsection{100th Percentile Results}
\label{app:hv100}
Tables~\ref{tab:hv_morl_proxy}--\ref{tab:hv_zdt_proxy} report the 100th-percentile hypervolume of the full returned set, from which the average ranks of Table~\ref{tab:avg_rank} are computed, for MORL (Table~\ref{tab:hv_morl_proxy}), C-10/MOP (Table~\ref{tab:hv_monas1_proxy}), IN-1K/MOP (Table~\ref{tab:hv_monas2_proxy}), RE (Tables~\ref{tab:hv_re1_proxy} and~\ref{tab:hv_re2_proxy}), DTLZ (Table~\ref{tab:hv_dtlz_proxy}), and ZDT (Table~\ref{tab:hv_zdt_proxy}).
\begin{table}[H]\centering
\caption{100th-percentile hypervolume results for MORL. For each task, methods within one standard deviation (or $5\times10^{-4}$) of the best are bolded and the best of the remaining methods is underlined; the shaded row marks the RFM variant with the most task-best results overall. }
\label{tab:hv_morl_proxy}
\small\setlength{\tabcolsep}{6pt}
\begin{tabular}{lcc}
\toprule
Methods & MO-Hopper & MO-Swimmer \\ \midrule
$\mathcal{D}$(best) & 6.33 & 3.70 \\ \midrule
E2E & 5.56 $\pm$ 0.46 & 2.99 $\pm$ 0.10 \\
E2E + GradNorm & 6.17 $\pm$ 0.15 & \textbf{3.55 $\pm$ 0.30} \\
E2E + PcGrad & 5.65 $\pm$ 0.53 & 3.31 $\pm$ 0.17 \\
MH & 6.03 $\pm$ 0.22 & 3.43 $\pm$ 0.14 \\
MH + GradNorm & \textbf{6.07 $\pm$ 0.57} & \textbf{3.61 $\pm$ 0.15} \\
MH + PcGrad & 5.81 $\pm$ 0.15 & 3.41 $\pm$ 0.11 \\
MM & 5.78 $\pm$ 0.26 & 3.50 $\pm$ 0.20 \\
MM + RoMA & \textbf{6.40 $\pm$ 0.20} & 3.40 $\pm$ 0.18 \\
MM + IOM & 5.97 $\pm$ 0.38 & 3.46 $\pm$ 0.18 \\
MM + ICT & \textbf{6.46 $\pm$ 0.27} & \textbf{3.74 $\pm$ 0.07} \\
MM + Tri-Mentor & 5.88 $\pm$ 0.33 & \textbf{3.73 $\pm$ 0.12} \\
\midrule
PCD & \textbf{6.27 $\pm$ 0.20} & 3.52 $\pm$ 0.13 \\
ParetoFlow & 5.88 $\pm$ 0.02 & 3.45 $\pm$ 0.19 \\
ParetoFlow+ & 5.79 $\pm$ 0.13 & 3.31 $\pm$ 0.28 \\
\midrule
RFM-AN (ours) & \underline{6.18 $\pm$ 0.17} & 3.58 $\pm$ 0.07 \\
RFM-APG (ours) & \textbf{6.40 $\pm$ 0.21} & \underline{3.59 $\pm$ 0.12} \\
RFM-DAMG (ours) & 6.12 $\pm$ 0.21 & 3.54 $\pm$ 0.08 \\
\rowcolor{oursbest} RFM-Cone (ours) & 6.17 $\pm$ 0.23 & 3.57 $\pm$ 0.06 \\
\bottomrule\end{tabular}
\end{table}
\begin{table}[H]\centering
\caption{100th-percentile hypervolume results for MO-NAS (Part 1, C-10/MOP). For each task, methods within one standard deviation (or $5\times10^{-4}$) of the best are bolded and the best of the remaining methods is underlined; the shaded row marks the RFM variant with the most task-best results overall.}
\label{tab:hv_monas1_proxy}
\scriptsize\setlength{\tabcolsep}{2.5pt}\resizebox{\textwidth}{!}{%
\begin{tabular}{lccccccccc}
\toprule
Methods & C-10/MOP1 & C-10/MOP2 & C-10/MOP3 & C-10/MOP4 & C-10/MOP5 & C-10/MOP6 & C-10/MOP7 & C-10/MOP8 & C-10/MOP9 \\ \midrule
$\mathcal{D}$(best) & 1.55 & 1.41 & 9.58 & 19.74 & 45.09 & 103.73 & 446.07 & 5.68 & 15.02 \\ \midrule
E2E & 0.00 $\pm$ 0.00 & 0.00 $\pm$ 0.00 & 10.50 $\pm$ 0.14 & 22.27 $\pm$ 0.24 & 40.90 $\pm$ 2.56 & 104.16 $\pm$ 1.87 & 408.71 $\pm$ 28.08 & 5.79 $\pm$ 0.04 & \textbf{15.47 $\pm$ 0.19} \\
E2E + GradNorm & 1.22 $\pm$ 0.11 & 1.16 $\pm$ 0.08 & 10.48 $\pm$ 0.15 & 21.93 $\pm$ 0.38 & 39.60 $\pm$ 2.35 & 101.06 $\pm$ 3.63 & 402.82 $\pm$ 23.99 & 5.70 $\pm$ 0.09 & \textbf{14.08 $\pm$ 1.66} \\
E2E + PcGrad & 0.69 $\pm$ 0.63 & 0.46 $\pm$ 0.64 & 10.66 $\pm$ 0.09 & 22.29 $\pm$ 0.46 & 43.90 $\pm$ 0.15 & 102.28 $\pm$ 2.54 & 412.49 $\pm$ 33.85 & 5.79 $\pm$ 0.03 & 14.85 $\pm$ 0.41 \\
MH & 1.15 $\pm$ 0.10 & 0.99 $\pm$ 0.20 & 10.43 $\pm$ 0.17 & 21.71 $\pm$ 0.67 & 38.84 $\pm$ 0.12 & 104.04 $\pm$ 2.22 & 440.62 $\pm$ 13.07 & \underline{5.82 $\pm$ 0.02} & \textbf{15.24 $\pm$ 0.41} \\
MH + GradNorm & 0.83 $\pm$ 0.56 & \textbf{0.94 $\pm$ 0.53} & 10.29 $\pm$ 0.21 & 21.64 $\pm$ 0.56 & 40.19 $\pm$ 2.21 & 97.93 $\pm$ 2.09 & 438.70 $\pm$ 16.65 & 5.36 $\pm$ 0.30 & 14.07 $\pm$ 0.75 \\
MH + PcGrad & 1.29 $\pm$ 0.07 & 1.10 $\pm$ 0.05 & 10.44 $\pm$ 0.16 & 22.14 $\pm$ 0.33 & 42.71 $\pm$ 2.19 & 101.37 $\pm$ 1.83 & 412.10 $\pm$ 21.71 & 5.78 $\pm$ 0.06 & \textbf{15.45 $\pm$ 0.19} \\
\bottomrule\end{tabular}}
\end{table}

\begin{table}[H]\centering
\caption{100th-percentile hypervolume results for MO-NAS (Part 2, IN-1K/MOP). For each task, methods within one standard deviation (or $5\times10^{-4}$) of the best are bolded and the best of the remaining methods is underlined; the shaded row marks the RFM variant with the most task-best results overall.}
\label{tab:hv_monas2_proxy}
\scriptsize\setlength{\tabcolsep}{2.5pt}\resizebox{\textwidth}{!}{%
\begin{tabular}{lccccccccc}
\toprule
Methods & IN-1K/MOP1 & IN-1K/MOP2 & IN-1K/MOP3 & IN-1K/MOP4 & IN-1K/MOP5 & IN-1K/MOP6 & IN-1K/MOP7 & IN-1K/MOP8 & IN-1K/MOP9 \\ \midrule
$\mathcal{D}$(best) & 6.60 & 7.44 & 17.59 & 4.23 & 4.57 & 11.36 & 5.73 & 14.04 & 19.66 \\ \midrule
E2E & 5.54 $\pm$ 0.22 & 5.67 $\pm$ 0.32 & 13.96 $\pm$ 0.20 & \textbf{4.51 $\pm$ 0.20} & \underline{4.85 $\pm$ 0.05} & \textbf{12.49 $\pm$ 0.10} & 4.86 $\pm$ 0.09 & 9.37 $\pm$ 0.46 & 11.40 $\pm$ 0.92 \\
E2E + GradNorm & 5.33 $\pm$ 0.05 & 6.18 $\pm$ 0.08 & 13.30 $\pm$ 0.43 & 4.12 $\pm$ 0.13 & 4.36 $\pm$ 0.05 & 9.94 $\pm$ 0.20 & 4.75 $\pm$ 0.31 & 8.57 $\pm$ 0.23 & 10.37 $\pm$ 0.37 \\
E2E + PcGrad & 5.70 $\pm$ 0.13 & 6.27 $\pm$ 0.02 & 14.10 $\pm$ 0.29 & 4.48 $\pm$ 0.05 & 4.65 $\pm$ 0.06 & 11.86 $\pm$ 0.14 & 4.95 $\pm$ 0.10 & 9.54 $\pm$ 0.17 & 11.63 $\pm$ 0.14 \\
MH & 5.59 $\pm$ 0.04 & 5.96 $\pm$ 0.12 & 13.67 $\pm$ 0.08 & \underline{4.52 $\pm$ 0.04} & 4.76 $\pm$ 0.05 & 11.53 $\pm$ 0.12 & 4.87 $\pm$ 0.04 & 9.43 $\pm$ 0.07 & 11.96 $\pm$ 0.31 \\
MH + GradNorm & 5.44 $\pm$ 0.22 & 6.28 $\pm$ 0.03 & 13.72 $\pm$ 0.45 & 3.30 $\pm$ 0.38 & \textbf{4.45 $\pm$ 0.49} & 8.68 $\pm$ 0.34 & 4.06 $\pm$ 0.15 & 8.71 $\pm$ 0.97 & 13.56 $\pm$ 1.91 \\
MH + PcGrad & 5.49 $\pm$ 0.09 & 6.33 $\pm$ 0.03 & 13.70 $\pm$ 0.29 & 4.45 $\pm$ 0.05 & 4.76 $\pm$ 0.02 & 11.71 $\pm$ 0.17 & 4.96 $\pm$ 0.08 & 9.45 $\pm$ 0.22 & 11.47 $\pm$ 0.41 \\
\bottomrule\end{tabular}}
\end{table}
\begin{table}[H]\centering
\caption{100th-percentile hypervolume results for RE (Part 1). For each task, methods within one standard deviation (or $5\times10^{-4}$) of the best are bolded and the best of the remaining methods is underlined; the shaded row marks the RFM variant with the most task-best results overall.}
\label{tab:hv_re1_proxy}
\scriptsize\setlength{\tabcolsep}{2.5pt}\resizebox{\textwidth}{!}{%
\begin{tabular}{lcccccccc}
\toprule
Methods & RE21 & RE22 & RE23 & RE24 & RE25 & RE31 & RE32 & RE33 \\ \midrule
$\mathcal{D}$(best) & 4.18 & 4.81 & 4.75 & 4.60 & 4.79 & 10.61 & 10.61 & 10.53 \\ \midrule
E2E & 4.26 $\pm$ 0.32 & \textbf{4.84 $\pm$ 0.00} & \textbf{4.84 $\pm$ 0.00} & 4.28 $\pm$ 0.02 & 4.35 $\pm$ 0.00 & \textbf{10.65 $\pm$ 0.00} & 10.63 $\pm$ 0.01 & \underline{10.62 $\pm$ 0.00} \\
E2E + GradNorm & 4.47 $\pm$ 0.07 & \textbf{4.84 $\pm$ 0.00} & \textbf{4.81 $\pm$ 0.03} & \textbf{4.69 $\pm$ 0.15} & \textbf{4.84 $\pm$ 0.00} & \textbf{10.61 $\pm$ 0.06} & 10.62 $\pm$ 0.01 & 10.57 $\pm$ 0.04 \\
E2E + PcGrad & \textbf{4.60 $\pm$ 0.00} & \textbf{4.84 $\pm$ 0.00} & \textbf{4.84 $\pm$ 0.00} & \textbf{4.74 $\pm$ 0.16} & \textbf{4.84 $\pm$ 0.00} & \textbf{10.65 $\pm$ 0.00} & 10.64 $\pm$ 0.00 & 10.60 $\pm$ 0.01 \\
MH & \textbf{4.60 $\pm$ 0.00} & \textbf{4.84 $\pm$ 0.00} & \textbf{4.84 $\pm$ 0.00} & \textbf{4.63 $\pm$ 0.25} & \textbf{4.84 $\pm$ 0.00} & \textbf{10.65 $\pm$ 0.00} & 10.63 $\pm$ 0.00 & 10.62 $\pm$ 0.00 \\
MH + GradNorm & \textbf{4.41 $\pm$ 0.25} & \textbf{4.84 $\pm$ 0.00} & 4.80 $\pm$ 0.03 & \textbf{4.49 $\pm$ 0.76} & \textbf{4.84 $\pm$ 0.00} & \textbf{10.65 $\pm$ 0.01} & 7.88 $\pm$ 1.23 & 10.60 $\pm$ 0.02 \\
MH + PcGrad & \textbf{4.60 $\pm$ 0.00} & \textbf{4.84 $\pm$ 0.00} & \textbf{4.39 $\pm$ 0.71} & 4.79 $\pm$ 0.03 & \textbf{4.84 $\pm$ 0.00} & \textbf{10.65 $\pm$ 0.00} & 10.63 $\pm$ 0.00 & 10.62 $\pm$ 0.01 \\
\bottomrule\end{tabular}}
\end{table}

\begin{table}[H]\centering
\caption{100th-percentile hypervolume results for DTLZ. For each task, methods within one standard deviation (or $5 \times 10^{-4}$) of the best are bolded and the best of the remaining methods is underlined; the shaded row marks the RFM variant with the most task-best results overall.}
\label{tab:hv_dtlz_proxy}
\scriptsize\setlength{\tabcolsep}{2.5pt}\resizebox{\textwidth}{!}{%
\begin{tabular}{lccccccc}
\toprule
Methods & DTLZ1 & DTLZ2 & DTLZ3 & DTLZ4 & DTLZ5 & DTLZ6 & DTLZ7 \\ \midrule
$\mathcal{D}$(best) & 10.61 & 12.43 & 9.90 & 17.50 & 20.98 & 15.35 & 8.99 \\ \midrule
E2E & \textbf{10.65 $\pm$ 0.00} & \textbf{12.45 $\pm$ 0.00} & \textbf{9.86 $\pm$ 0.07} & \textbf{17.66 $\pm$ 0.03} & \underline{21.06 $\pm$ 0.00} & \textbf{15.74 $\pm$ 0.01} & \textbf{10.80 $\pm$ 0.00} \\
E2E + GradNorm & 10.65 $\pm$ 0.00 & \textbf{12.44 $\pm$ 0.01} & 9.89 $\pm$ 0.00 & 17.53 $\pm$ 0.07 & 21.05 $\pm$ 0.01 & 15.32 $\pm$ 0.14 & 10.69 $\pm$ 0.02 \\
E2E + PcGrad & 10.65 $\pm$ 0.00 & \textbf{12.44 $\pm$ 0.00} & 9.89 $\pm$ 0.01 & \underline{17.62 $\pm$ 0.04} & \textbf{21.06 $\pm$ 0.00} & 15.67 $\pm$ 0.01 & 10.77 $\pm$ 0.00 \\
MH & 10.64 $\pm$ 0.00 & \textbf{12.44 $\pm$ 0.00} & 9.89 $\pm$ 0.00 & \textbf{17.61 $\pm$ 0.11} & 21.06 $\pm$ 0.00 & 15.66 $\pm$ 0.01 & 10.78 $\pm$ 0.01 \\
MH + GradNorm & 10.64 $\pm$ 0.00 & 12.44 $\pm$ 0.00 & 9.88 $\pm$ 0.00 & 17.56 $\pm$ 0.03 & 20.94 $\pm$ 0.05 & 15.55 $\pm$ 0.04 & 10.77 $\pm$ 0.01 \\
MH + PcGrad & 10.64 $\pm$ 0.00 & \textbf{12.45 $\pm$ 0.00} & 9.89 $\pm$ 0.00 & 17.55 $\pm$ 0.06 & 21.06 $\pm$ 0.00 & 15.58 $\pm$ 0.01 & 10.70 $\pm$ 0.07 \\
\bottomrule\end{tabular}}
\end{table}
\begin{table}[H]\centering
\caption{100th-percentile hypervolume results for RE (Part 2). For each task, methods within one standard deviation (or $5\times10^{-4}$) of the best are bolded and the best of the remaining methods is underlined; the shaded row marks the RFM variant with the most task-best results overall.}
\label{tab:hv_re2_proxy}
\scriptsize\setlength{\tabcolsep}{2.5pt}\resizebox{\textwidth}{!}{%
\begin{tabular}{lccccccc}
\toprule
Methods & RE34 & RE35 & RE36 & RE37 & RE41 & RE42 & RE61 \\ \midrule
$\mathcal{D}$(best) & 9.65 & 10.22 & 7.76 & 5.92 & 18.99 & 15.28 & 105.13 \\ \midrule
E2E & \textbf{10.12 $\pm$ 0.00} & 10.52 $\pm$ 0.02 & 10.02 $\pm$ 0.12 & \textbf{6.75 $\pm$ 0.00} & \textbf{20.85 $\pm$ 0.05} & 22.38 $\pm$ 0.36 & \textbf{108.89 $\pm$ 0.03} \\
E2E + GradNorm & 9.85 $\pm$ 0.22 & 10.38 $\pm$ 0.02 & 9.83 $\pm$ 0.24 & 6.65 $\pm$ 0.02 & 17.52 $\pm$ 0.57 & 21.53 $\pm$ 0.26 & 102.69 $\pm$ 2.91 \\
E2E + PcGrad & \textbf{10.12 $\pm$ 0.00} & \textbf{10.58 $\pm$ 0.01} & \textbf{10.17 $\pm$ 0.22} & \textbf{6.75 $\pm$ 0.00} & \textbf{20.89 $\pm$ 0.03} & \textbf{22.71 $\pm$ 0.09} & 104.63 $\pm$ 0.85 \\
MH & \underline{10.11 $\pm$ 0.00} & \textbf{10.56 $\pm$ 0.02} & 10.07 $\pm$ 0.02 & \textbf{6.75 $\pm$ 0.01} & 20.76 $\pm$ 0.04 & 22.51 $\pm$ 0.14 & \textbf{108.77 $\pm$ 0.16} \\
MH + GradNorm & 9.79 $\pm$ 0.07 & 10.41 $\pm$ 0.07 & 7.39 $\pm$ 1.88 & 6.68 $\pm$ 0.03 & 20.75 $\pm$ 0.06 & 22.24 $\pm$ 0.14 & 105.61 $\pm$ 2.98 \\
MH + PcGrad & 10.11 $\pm$ 0.00 & \textbf{10.57 $\pm$ 0.01} & 10.08 $\pm$ 0.15 & \underline{6.74 $\pm$ 0.01} & \underline{20.78 $\pm$ 0.02} & 22.53 $\pm$ 0.05 & 107.66 $\pm$ 0.15 \\
\bottomrule\end{tabular}}
\end{table}

\begin{table}[H]\centering
\caption{100th-percentile hypervolume results for ZDT. For each task, methods within one standard deviation (or $5\times10^{-4}$) of the best are bolded and the best of the remaining methods is underlined; the shaded row marks the RFM variant with the most task-best results overall.}
\label{tab:hv_zdt_proxy}
\small\setlength{\tabcolsep}{6pt}
\begin{tabular}{lccccc}
\toprule
Methods & ZDT1 & ZDT2 & ZDT3 & ZDT4 & ZDT6 \\ \midrule
$\mathcal{D}$(best) & 4.17 & 4.68 & 5.15 & 5.46 & 4.61 \\ \midrule
E2E & \textbf{4.86 $\pm$ 0.00} & \textbf{5.71 $\pm$ 0.00} & 5.71 $\pm$ 0.03 & 4.68 $\pm$ 0.24 & \textbf{4.80 $\pm$ 0.00} \\
E2E + GradNorm & 4.16 $\pm$ 0.07 & 5.50 $\pm$ 0.14 & 5.02 $\pm$ 0.06 & 5.09 $\pm$ 0.09 & 4.41 $\pm$ 0.08 \\
E2E + PcGrad & \underline{4.86 $\pm$ 0.00} & \textbf{5.71 $\pm$ 0.00} & 5.74 $\pm$ 0.09 & 4.11 $\pm$ 0.12 & \textbf{3.90 $\pm$ 1.07} \\
MH & 4.85 $\pm$ 0.00 & 5.63 $\pm$ 0.01 & \textbf{5.80 $\pm$ 0.04} & 5.14 $\pm$ 0.18 & \underline{4.78 $\pm$ 0.00} \\
MH + GradNorm & \textbf{4.82 $\pm$ 0.04} & \textbf{5.63 $\pm$ 0.08} & 5.42 $\pm$ 0.15 & 3.82 $\pm$ 0.08 & 3.32 $\pm$ 1.05 \\
MH + PcGrad & 4.83 $\pm$ 0.01 & 5.53 $\pm$ 0.06 & 5.69 $\pm$ 0.08 & 3.78 $\pm$ 0.35 & \textbf{4.31 $\pm$ 0.74} \\
\bottomrule\end{tabular}
\end{table}

\subsection{75th Percentile Results}
\label{app:hv75}
Tables~\ref{tab:hv_dtlz_proxy_hv75}--\ref{tab:hv_re2_proxy_hv75} report the 75th-percentile hypervolume for DTLZ (Table~\ref{tab:hv_dtlz_proxy_hv75}), ZDT (Table~\ref{tab:hv_zdt_proxy_hv75}), C-10/MOP (Table~\ref{tab:hv_monas1_proxy_hv75}), IN-1K/MOP (Table~\ref{tab:hv_monas2_proxy_hv75}), MORL (Table~\ref{tab:hv_morl_proxy_hv75}), and RE (Tables~\ref{tab:hv_re1_proxy_hv75} and~\ref{tab:hv_re2_proxy_hv75}). $\mathcal{D}(\text{best})$ is reported only at the 100th percentile.

\begin{table}[H]\centering
\caption{75th-percentile hypervolume results for DTLZ. For each task, methods within one standard deviation (or $5\times10^{-4}$) of the best are bolded and the best of the remaining methods is underlined; the shaded row marks the RFM variant with the most task-best results overall.}
\label{tab:hv_dtlz_proxy_hv75}
\scriptsize\setlength{\tabcolsep}{2.5pt}\resizebox{\textwidth}{!}{%
\begin{tabular}{lccccccc}
\toprule
Methods & DTLZ1 & DTLZ2 & DTLZ3 & DTLZ4 & DTLZ5 & DTLZ6 & DTLZ7 \\ \midrule
$\mathcal{D}$(best) & -- & -- & -- & -- & -- & -- & -- \\ \midrule
E2E & \textbf{10.65 $\pm$ 0.00} & \textbf{12.44 $\pm$ 0.00} & \textbf{9.83 $\pm$ 0.09} & \textbf{17.65 $\pm$ 0.03} & \underline{21.06 $\pm$ 0.00} & \textbf{15.72 $\pm$ 0.01} & \textbf{10.80 $\pm$ 0.00} \\
E2E + GradNorm & 10.64 $\pm$ 0.00 & \textbf{12.43 $\pm$ 0.02} & 9.87 $\pm$ 0.00 & 17.35 $\pm$ 0.22 & \textbf{21.04 $\pm$ 0.02} & 15.31 $\pm$ 0.14 & 10.67 $\pm$ 0.04 \\
E2E + PcGrad & 10.64 $\pm$ 0.00 & \textbf{12.44 $\pm$ 0.00} & \textbf{9.88 $\pm$ 0.01} & 17.47 $\pm$ 0.08 & \textbf{21.06 $\pm$ 0.00} & \underline{15.66 $\pm$ 0.01} & 10.77 $\pm$ 0.01 \\
MH & 10.64 $\pm$ 0.00 & \textbf{12.44 $\pm$ 0.00} & 9.87 $\pm$ 0.01 & \textbf{17.57 $\pm$ 0.13} & \textbf{21.05 $\pm$ 0.02} & \textbf{15.01 $\pm$ 0.78} & 10.77 $\pm$ 0.01 \\
MH + GradNorm & 10.64 $\pm$ 0.00 & 12.43 $\pm$ 0.00 & 9.86 $\pm$ 0.01 & 17.47 $\pm$ 0.03 & 20.89 $\pm$ 0.14 & 15.54 $\pm$ 0.04 & 10.76 $\pm$ 0.01 \\
MH + PcGrad & 10.64 $\pm$ 0.00 & \textbf{12.44 $\pm$ 0.00} & 9.87 $\pm$ 0.01 & \textbf{17.24 $\pm$ 0.44} & 21.05 $\pm$ 0.01 & 15.57 $\pm$ 0.01 & 10.69 $\pm$ 0.08 \\
\bottomrule\end{tabular}}
\end{table}
 
\begin{table}[H]\centering
\caption{75th-percentile hypervolume results for ZDT. For each task, methods within one standard deviation (or $5\times10^{-4}$) of the best are bolded and the best of the remaining methods is underlined; the shaded row marks the RFM variant with the most task-best results overall.}
\label{tab:hv_zdt_proxy_hv75}
\small\setlength{\tabcolsep}{6pt}
\begin{tabular}{lccccc}
\toprule
Methods & ZDT1 & ZDT2 & ZDT3 & ZDT4 & ZDT6 \\ \midrule
$\mathcal{D}$(best) & -- & -- & -- & -- & -- \\ \midrule
E2E & \textbf{4.86 $\pm$ 0.00} & \textbf{5.71 $\pm$ 0.00} & 5.65 $\pm$ 0.04 & 4.52 $\pm$ 0.26 & \textbf{4.79 $\pm$ 0.00} \\
E2E + GradNorm & 3.29 $\pm$ 0.03 & \textbf{5.21 $\pm$ 0.53} & 4.34 $\pm$ 0.06 & 4.80 $\pm$ 0.05 & 3.33 $\pm$ 1.08 \\
E2E + PcGrad & \underline{4.86 $\pm$ 0.00} & \textbf{5.71 $\pm$ 0.00} & \textbf{5.55 $\pm$ 0.36} & 3.92 $\pm$ 0.19 & 1.97 $\pm$ 0.29 \\
MH & 4.85 $\pm$ 0.00 & 5.62 $\pm$ 0.01 & \textbf{5.77 $\pm$ 0.05} & 4.88 $\pm$ 0.11 & 4.76 $\pm$ 0.00 \\
MH + GradNorm & \textbf{4.82 $\pm$ 0.05} & \textbf{5.62 $\pm$ 0.08} & 5.29 $\pm$ 0.17 & 3.48 $\pm$ 0.05 & 3.28 $\pm$ 1.04 \\
MH + PcGrad & 4.83 $\pm$ 0.01 & 5.47 $\pm$ 0.05 & 5.64 $\pm$ 0.09 & 3.57 $\pm$ 0.34 & \textbf{4.26 $\pm$ 0.75} \\
\bottomrule\end{tabular}
\end{table}
 
\begin{table}[H]\centering
\caption{75th-percentile hypervolume results for MO-NAS (Part 1, C-10/MOP). For each task, methods within one standard deviation (or $5\times10^{-4}$) of the best are bolded and the best of the remaining methods is underlined; the shaded row marks the RFM variant with the most task-best results overall.}
\label{tab:hv_monas1_proxy_hv75}
\scriptsize\setlength{\tabcolsep}{2.5pt}\resizebox{\textwidth}{!}{%
\begin{tabular}{lccccccccc}
\toprule
Methods & C-10/MOP1 & C-10/MOP2 & C-10/MOP3 & C-10/MOP4 & C-10/MOP5 & C-10/MOP6 & C-10/MOP7 & C-10/MOP8 & C-10/MOP9 \\ \midrule
$\mathcal{D}$(best) & -- & -- & -- & -- & -- & -- & -- & -- & -- \\ \midrule
E2E & 0.00 $\pm$ 0.00 & 0.00 $\pm$ 0.00 & 10.35 $\pm$ 0.21 & 22.10 $\pm$ 0.22 & 40.22 $\pm$ 2.23 & \underline{103.01 $\pm$ 2.57} & 405.35 $\pm$ 29.34 & 5.75 $\pm$ 0.05 & \textbf{15.38 $\pm$ 0.18} \\
E2E + GradNorm & 1.20 $\pm$ 0.10 & 1.16 $\pm$ 0.08 & 10.02 $\pm$ 0.12 & 21.87 $\pm$ 0.35 & 38.07 $\pm$ 2.14 & 98.91 $\pm$ 3.63 & 402.47 $\pm$ 23.92 & 5.51 $\pm$ 0.09 & \textbf{13.78 $\pm$ 1.85} \\
E2E + PcGrad & 0.68 $\pm$ 0.62 & 0.26 $\pm$ 0.58 & 10.48 $\pm$ 0.24 & 22.08 $\pm$ 0.28 & 42.50 $\pm$ 0.80 & 101.72 $\pm$ 2.49 & 411.63 $\pm$ 33.26 & \underline{5.77 $\pm$ 0.02} & 14.58 $\pm$ 0.32 \\
MH & 0.91 $\pm$ 0.52 & 0.98 $\pm$ 0.20 & 10.22 $\pm$ 0.32 & 21.36 $\pm$ 1.02 & 38.41 $\pm$ 0.62 & 101.89 $\pm$ 1.48 & 435.49 $\pm$ 6.14 & \textbf{5.77 $\pm$ 0.05} & \textbf{15.02 $\pm$ 0.42} \\
MH + GradNorm & 0.82 $\pm$ 0.55 & \textbf{0.91 $\pm$ 0.51} & 9.94 $\pm$ 0.43 & 21.22 $\pm$ 0.65 & 37.69 $\pm$ 1.26 & 96.99 $\pm$ 1.79 & 408.37 $\pm$ 24.70 & 5.27 $\pm$ 0.26 & 12.19 $\pm$ 1.39 \\
MH + PcGrad & 1.23 $\pm$ 0.09 & 1.09 $\pm$ 0.04 & 10.32 $\pm$ 0.15 & 21.88 $\pm$ 0.63 & 41.39 $\pm$ 2.08 & 100.84 $\pm$ 2.33 & 401.48 $\pm$ 19.89 & 5.71 $\pm$ 0.07 & \textbf{15.26 $\pm$ 0.24} \\
\bottomrule\end{tabular}}
\end{table}
 
\begin{table}[H]\centering
\caption{75th-percentile hypervolume results for MO-NAS (Part 2, IN-1K/MOP). For each task, methods within one standard deviation (or $5\times10^{-4}$) of the best are bolded and the best of the remaining methods is underlined; the shaded row marks the RFM variant with the most task-best results overall.}
\label{tab:hv_monas2_proxy_hv75}
\scriptsize\setlength{\tabcolsep}{2.5pt}\resizebox{\textwidth}{!}{%
\begin{tabular}{lccccccccc}
\toprule
Methods & IN-1K/MOP1 & IN-1K/MOP2 & IN-1K/MOP3 & IN-1K/MOP4 & IN-1K/MOP5 & IN-1K/MOP6 & IN-1K/MOP7 & IN-1K/MOP8 & IN-1K/MOP9 \\ \midrule
$\mathcal{D}$(best) & -- & -- & -- & -- & -- & -- & -- & -- & -- \\ \midrule
E2E & 5.52 $\pm$ 0.23 & 5.66 $\pm$ 0.33 & 13.91 $\pm$ 0.20 & \textbf{4.49 $\pm$ 0.21} & \underline{4.84 $\pm$ 0.04} & \underline{11.83 $\pm$ 0.37} & 4.75 $\pm$ 0.15 & 9.31 $\pm$ 0.48 & 11.33 $\pm$ 0.91 \\
E2E + GradNorm & 5.26 $\pm$ 0.02 & 6.16 $\pm$ 0.09 & 13.23 $\pm$ 0.36 & 4.09 $\pm$ 0.12 & 4.31 $\pm$ 0.08 & 9.66 $\pm$ 0.14 & 4.54 $\pm$ 0.36 & 8.56 $\pm$ 0.23 & 10.29 $\pm$ 0.40 \\
E2E + PcGrad & 5.68 $\pm$ 0.13 & 6.25 $\pm$ 0.02 & 14.06 $\pm$ 0.26 & 4.44 $\pm$ 0.09 & 4.61 $\pm$ 0.04 & 11.77 $\pm$ 0.24 & 4.85 $\pm$ 0.11 & 9.52 $\pm$ 0.18 & 11.58 $\pm$ 0.15 \\
MH & 5.56 $\pm$ 0.05 & 5.87 $\pm$ 0.25 & 13.58 $\pm$ 0.04 & 4.48 $\pm$ 0.03 & 4.72 $\pm$ 0.03 & 10.98 $\pm$ 0.90 & 4.76 $\pm$ 0.03 & 9.33 $\pm$ 0.09 & 11.81 $\pm$ 0.44 \\
MH + GradNorm & 5.43 $\pm$ 0.23 & 6.22 $\pm$ 0.06 & 13.70 $\pm$ 0.43 & 3.22 $\pm$ 0.39 & \textbf{4.41 $\pm$ 0.50} & 8.52 $\pm$ 0.34 & 3.83 $\pm$ 0.24 & 8.66 $\pm$ 0.96 & 13.52 $\pm$ 1.90 \\
MH + PcGrad & 5.44 $\pm$ 0.09 & \underline{6.27 $\pm$ 0.03} & 13.69 $\pm$ 0.30 & 4.38 $\pm$ 0.05 & 4.68 $\pm$ 0.05 & 11.38 $\pm$ 0.24 & 4.84 $\pm$ 0.10 & 9.42 $\pm$ 0.21 & 11.43 $\pm$ 0.47 \\
\bottomrule\end{tabular}}
\end{table}
 
\begin{table}[H]\centering
\caption{75th-percentile hypervolume results for MORL. For each task, methods within one standard deviation (or $5\times10^{-4}$) of the best are bolded and the best of the remaining methods is underlined; the shaded row marks the RFM variant with the most task-best results overall.}
\label{tab:hv_morl_proxy_hv75}
\small\setlength{\tabcolsep}{6pt}
\begin{tabular}{lcc}
\toprule
Methods & MO-Hopper & MO-Swimmer \\ \midrule
$\mathcal{D}$(best) & -- & -- \\ \midrule
E2E & 4.98 $\pm$ 0.19 & 2.55 $\pm$ 0.29 \\
E2E + GradNorm & 5.20 $\pm$ 0.36 & 3.17 $\pm$ 0.30 \\
E2E + PcGrad & 4.95 $\pm$ 0.21 & 2.79 $\pm$ 0.06 \\
MH & \underline{5.35 $\pm$ 0.20} & 2.95 $\pm$ 0.14 \\
MH + GradNorm & \textbf{5.57 $\pm$ 0.46} & 2.96 $\pm$ 0.09 \\
MH + PcGrad & 4.88 $\pm$ 0.15 & 2.93 $\pm$ 0.16 \\
MM & 5.00 $\pm$ 0.19 & 2.81 $\pm$ 0.09 \\
MM + RoMA & \textbf{5.30 $\pm$ 0.35} & 2.83 $\pm$ 0.09 \\
MM + IOM & 4.97 $\pm$ 0.12 & 2.86 $\pm$ 0.09 \\
MM + ICT & 5.33 $\pm$ 0.08 & \textbf{3.61 $\pm$ 0.07} \\
MM + Tri-Mentor & 5.12 $\pm$ 0.22 & \underline{3.41 $\pm$ 0.13} \\
\midrule
PCD & 5.08 $\pm$ 0.09 & 2.84 $\pm$ 0.08 \\
ParetoFlow & \textbf{5.58 $\pm$ 0.03} & 2.76 $\pm$ 0.49 \\
ParetoFlow+ & \textbf{5.59 $\pm$ 0.03} & 2.87 $\pm$ 0.27 \\
\midrule
RFM-AN (ours) & 5.06 $\pm$ 0.14 & 2.85 $\pm$ 0.15 \\
RFM-APG (ours) & 5.12 $\pm$ 0.17 & 2.83 $\pm$ 0.08 \\
RFM-DAMG (ours) & 5.21 $\pm$ 0.15 & 2.86 $\pm$ 0.07 \\
\rowcolor{oursbest} RFM-Cone (ours) & 5.05 $\pm$ 0.06 & 2.88 $\pm$ 0.14 \\
\bottomrule\end{tabular}
\end{table}
 
\begin{table}[H]\centering
\caption{75th-percentile hypervolume results for RE (Part 1). For each task, methods within one standard deviation (or $5\times10^{-4}$) of the best are bolded and the best of the remaining methods is underlined; the shaded row marks the RFM variant with the most task-best results overall.}
\label{tab:hv_re1_proxy_hv75}
\scriptsize\setlength{\tabcolsep}{2.5pt}\resizebox{\textwidth}{!}{%
\begin{tabular}{lcccccccc}
\toprule
Methods & RE21 & RE22 & RE23 & RE24 & RE25 & RE31 & RE32 & RE33 \\ \midrule
$\mathcal{D}$(best) & -- & -- & -- & -- & -- & -- & -- & -- \\ \midrule
E2E & 4.25 $\pm$ 0.31 & \textbf{4.75 $\pm$ 0.21} & \textbf{4.72 $\pm$ 0.27} & 4.27 $\pm$ 0.03 & 4.23 $\pm$ 0.14 & \textbf{10.61 $\pm$ 0.06} & 10.63 $\pm$ 0.01 & 10.61 $\pm$ 0.01 \\
E2E + GradNorm & 4.46 $\pm$ 0.07 & \textbf{4.84 $\pm$ 0.00} & 4.16 $\pm$ 0.03 & 4.31 $\pm$ 0.02 & 4.55 $\pm$ 0.02 & 9.93 $\pm$ 0.03 & 10.18 $\pm$ 0.33 & 9.30 $\pm$ 0.74 \\
E2E + PcGrad & \textbf{4.59 $\pm$ 0.00} & \textbf{4.84 $\pm$ 0.00} & \textbf{4.84 $\pm$ 0.00} & 3.97 $\pm$ 0.06 & 4.55 $\pm$ 0.06 & \textbf{10.63 $\pm$ 0.03} & \underline{10.64 $\pm$ 0.00} & 9.70 $\pm$ 0.85 \\
MH & \textbf{4.59 $\pm$ 0.00} & \textbf{4.84 $\pm$ 0.00} & \textbf{4.84 $\pm$ 0.00} & \textbf{4.52 $\pm$ 0.39} & \textbf{4.54 $\pm$ 0.67} & \textbf{10.65 $\pm$ 0.00} & 10.63 $\pm$ 0.00 & 10.60 $\pm$ 0.01 \\
MH + GradNorm & \textbf{4.35 $\pm$ 0.33} & \textbf{4.40 $\pm$ 0.60} & 4.75 $\pm$ 0.04 & \textbf{4.21 $\pm$ 0.96} & 4.82 $\pm$ 0.02 & \textbf{10.49 $\pm$ 0.35} & 7.66 $\pm$ 1.09 & 10.58 $\pm$ 0.02 \\
MH + PcGrad & \textbf{4.59 $\pm$ 0.00} & \textbf{4.73 $\pm$ 0.23} & \textbf{4.37 $\pm$ 0.73} & \underline{4.78 $\pm$ 0.04} & \textbf{4.82 $\pm$ 0.04} & \textbf{10.60 $\pm$ 0.05} & 10.63 $\pm$ 0.00 & 10.59 $\pm$ 0.02 \\
\bottomrule\end{tabular}}
\end{table}
 
\begin{table}[H]\centering
\caption{75th-percentile hypervolume results for RE (Part 2). For each task, methods within one standard deviation (or $5\times10^{-4}$) of the best are bolded and the best of the remaining methods is underlined; the shaded row marks the RFM variant with the most task-best results overall.}
\label{tab:hv_re2_proxy_hv75}
\scriptsize\setlength{\tabcolsep}{2.5pt}\resizebox{\textwidth}{!}{%
\begin{tabular}{lccccccc}
\toprule
Methods & RE34 & RE35 & RE36 & RE37 & RE41 & RE42 & RE61 \\ \midrule
$\mathcal{D}$(best) & -- & -- & -- & -- & -- & -- & -- \\ \midrule
E2E & \textbf{10.09 $\pm$ 0.01} & 10.52 $\pm$ 0.02 & 9.87 $\pm$ 0.12 & \textbf{6.71 $\pm$ 0.03} & \textbf{20.74 $\pm$ 0.10} & \textbf{22.27 $\pm$ 0.39} & \textbf{108.63 $\pm$ 0.10} \\
E2E + GradNorm & \textbf{9.57 $\pm$ 0.73} & 10.38 $\pm$ 0.02 & 6.06 $\pm$ 0.68 & 6.62 $\pm$ 0.02 & 17.41 $\pm$ 0.57 & 21.43 $\pm$ 0.20 & 100.45 $\pm$ 3.18 \\
E2E + PcGrad & \textbf{10.08 $\pm$ 0.02} & \textbf{10.58 $\pm$ 0.00} & \textbf{10.03 $\pm$ 0.22} & \textbf{6.72 $\pm$ 0.01} & \textbf{20.79 $\pm$ 0.04} & \textbf{22.66 $\pm$ 0.10} & 103.26 $\pm$ 1.44 \\
MH & \underline{10.08 $\pm$ 0.01} & \textbf{10.55 $\pm$ 0.02} & \underline{9.93 $\pm$ 0.04} & \textbf{6.70 $\pm$ 0.02} & 20.62 $\pm$ 0.05 & 22.44 $\pm$ 0.08 & \textbf{108.31 $\pm$ 0.47} \\
MH + GradNorm & 9.73 $\pm$ 0.03 & 10.35 $\pm$ 0.03 & 6.54 $\pm$ 1.27 & 6.65 $\pm$ 0.02 & 20.63 $\pm$ 0.06 & 22.13 $\pm$ 0.16 & 105.13 $\pm$ 3.20 \\
MH + PcGrad & 10.08 $\pm$ 0.01 & \textbf{10.57 $\pm$ 0.01} & 9.76 $\pm$ 0.21 & \textbf{6.71 $\pm$ 0.03} & \underline{20.65 $\pm$ 0.05} & \underline{22.47 $\pm$ 0.08} & 107.25 $\pm$ 0.11 \\
\bottomrule\end{tabular}}
\end{table}

\subsection{50th Percentile Results}
\label{app:hv50}
Tables~\ref{tab:hv_dtlz_proxy_hv50}--\ref{tab:hv_re2_proxy_hv50} report the 50th-percentile hypervolume for DTLZ (Table~\ref{tab:hv_dtlz_proxy_hv50}), ZDT (Table~\ref{tab:hv_zdt_proxy_hv50}), C-10/MOP (Table~\ref{tab:hv_monas1_proxy_hv50}), IN-1K/MOP (Table~\ref{tab:hv_monas2_proxy_hv50}), MORL (Table~\ref{tab:hv_morl_proxy_hv50}), and RE (Tables~\ref{tab:hv_re1_proxy_hv50} and~\ref{tab:hv_re2_proxy_hv50}). $\mathcal{D}(\text{best})$ is reported only at the 100th percentile.
\begin{table}[H]\centering
\caption{50th-percentile hypervolume results for DTLZ. For each task, methods within one standard deviation (or $5\times10^{-4}$) of the best are bolded and the best of the remaining methods is underlined; the shaded row marks the RFM variant with the most task-best results overall.}
\label{tab:hv_dtlz_proxy_hv50}
\scriptsize\setlength{\tabcolsep}{2.5pt}\resizebox{\textwidth}{!}{%
\begin{tabular}{lccccccc}
\toprule
Methods & DTLZ1 & DTLZ2 & DTLZ3 & DTLZ4 & DTLZ5 & DTLZ6 & DTLZ7 \\ \midrule
$\mathcal{D}$(best) & -- & -- & -- & -- & -- & -- & -- \\ \midrule
E2E & \textbf{10.64 $\pm$ 0.00} & 12.43 $\pm$ 0.01 & 9.73 $\pm$ 0.12 & \textbf{17.61 $\pm$ 0.03} & \textbf{21.05 $\pm$ 0.01} & \textbf{15.70 $\pm$ 0.02} & \textbf{10.79 $\pm$ 0.00} \\
E2E + GradNorm & 10.62 $\pm$ 0.01 & \textbf{12.40 $\pm$ 0.08} & 9.83 $\pm$ 0.02 & 17.16 $\pm$ 0.30 & \textbf{20.83 $\pm$ 0.40} & 15.16 $\pm$ 0.22 & 10.59 $\pm$ 0.08 \\
E2E + PcGrad & 10.62 $\pm$ 0.01 & \textbf{12.44 $\pm$ 0.00} & 9.85 $\pm$ 0.02 & 17.25 $\pm$ 0.17 & \textbf{21.06 $\pm$ 0.00} & \underline{15.62 $\pm$ 0.02} & 10.61 $\pm$ 0.08 \\
MH & 10.62 $\pm$ 0.01 & \textbf{12.44 $\pm$ 0.01} & 9.84 $\pm$ 0.03 & \textbf{17.38 $\pm$ 0.42} & \textbf{21.03 $\pm$ 0.04} & 14.31 $\pm$ 0.49 & 10.77 $\pm$ 0.02 \\
MH + GradNorm & 10.63 $\pm$ 0.01 & 12.42 $\pm$ 0.01 & 9.80 $\pm$ 0.06 & 17.25 $\pm$ 0.14 & 20.83 $\pm$ 0.16 & 15.46 $\pm$ 0.06 & 10.69 $\pm$ 0.03 \\
MH + PcGrad & 10.62 $\pm$ 0.01 & \textbf{12.44 $\pm$ 0.00} & 9.83 $\pm$ 0.01 & 16.98 $\pm$ 0.53 & \textbf{21.04 $\pm$ 0.03} & 15.49 $\pm$ 0.03 & 10.64 $\pm$ 0.09 \\
\bottomrule\end{tabular}}
\end{table}
 
\begin{table}[H]\centering
\caption{50th-percentile hypervolume results for ZDT. For each task, methods within one standard deviation (or $5\times10^{-4}$) of the best are bolded and the best of the remaining methods is underlined; the shaded row marks the RFM variant with the most task-best results overall.}
\label{tab:hv_zdt_proxy_hv50}
\small\setlength{\tabcolsep}{6pt}
\begin{tabular}{lccccc}
\toprule
Methods & ZDT1 & ZDT2 & ZDT3 & ZDT4 & ZDT6 \\ \midrule
$\mathcal{D}$(best) & -- & -- & -- & -- & -- \\ \midrule
E2E & \textbf{4.86 $\pm$ 0.00} & \textbf{5.70 $\pm$ 0.00} & \textbf{5.18 $\pm$ 0.64} & 4.40 $\pm$ 0.32 & \textbf{4.78 $\pm$ 0.01} \\
E2E + GradNorm & 3.08 $\pm$ 0.10 & 5.07 $\pm$ 0.61 & 4.21 $\pm$ 0.07 & 4.65 $\pm$ 0.02 & 3.02 $\pm$ 0.90 \\
E2E + PcGrad & \underline{4.86 $\pm$ 0.00} & \textbf{5.70 $\pm$ 0.00} & 5.06 $\pm$ 0.50 & 3.80 $\pm$ 0.20 & 1.82 $\pm$ 0.22 \\
MH & 4.84 $\pm$ 0.01 & 5.52 $\pm$ 0.11 & \textbf{5.68 $\pm$ 0.07} & 4.72 $\pm$ 0.11 & \underline{4.75 $\pm$ 0.00} \\
MH + GradNorm & \textbf{4.49 $\pm$ 0.79} & \underline{5.56 $\pm$ 0.10} & 4.94 $\pm$ 0.26 & 3.22 $\pm$ 0.13 & 3.21 $\pm$ 0.99 \\
MH + PcGrad & 4.82 $\pm$ 0.01 & 5.39 $\pm$ 0.01 & 5.49 $\pm$ 0.07 & 3.43 $\pm$ 0.34 & \textbf{4.12 $\pm$ 0.85} \\
\bottomrule\end{tabular}
\end{table}
 
\begin{table}[H]\centering
\caption{50th-percentile hypervolume results for MO-NAS (Part 1, C-10/MOP). For each task, methods within one standard deviation (or $5\times10^{-4}$) of the best are bolded and the best of the remaining methods is underlined; the shaded row marks the RFM variant with the most task-best results overall.}
\label{tab:hv_monas1_proxy_hv50}
\scriptsize\setlength{\tabcolsep}{2.5pt}\resizebox{\textwidth}{!}{%
\begin{tabular}{lccccccccc}
\toprule
Methods & C-10/MOP1 & C-10/MOP2 & C-10/MOP3 & C-10/MOP4 & C-10/MOP5 & C-10/MOP6 & C-10/MOP7 & C-10/MOP8 & C-10/MOP9 \\ \midrule
$\mathcal{D}$(best) & -- & -- & -- & -- & -- & -- & -- & -- & -- \\ \midrule
E2E & 0.00 $\pm$ 0.00 & 0.00 $\pm$ 0.00 & 10.10 $\pm$ 0.25 & \underline{21.02 $\pm$ 0.32} & 38.49 $\pm$ 2.38 & 99.40 $\pm$ 2.67 & 400.60 $\pm$ 24.57 & 5.72 $\pm$ 0.05 & \textbf{15.22 $\pm$ 0.22} \\
E2E + GradNorm & 1.20 $\pm$ 0.10 & 1.12 $\pm$ 0.08 & 9.66 $\pm$ 0.14 & 20.43 $\pm$ 0.69 & 37.02 $\pm$ 2.45 & 94.45 $\pm$ 4.46 & 377.97 $\pm$ 19.27 & 4.88 $\pm$ 0.39 & \textbf{13.50 $\pm$ 1.79} \\
E2E + PcGrad & 0.66 $\pm$ 0.61 & 0.26 $\pm$ 0.57 & 9.99 $\pm$ 0.17 & 20.50 $\pm$ 0.49 & 40.11 $\pm$ 2.06 & 97.21 $\pm$ 2.04 & 396.72 $\pm$ 30.39 & \textbf{5.72 $\pm$ 0.06} & 14.41 $\pm$ 0.32 \\
MH & 1.10 $\pm$ 0.04 & 0.95 $\pm$ 0.16 & 9.74 $\pm$ 0.31 & 19.56 $\pm$ 1.18 & 37.19 $\pm$ 0.65 & 99.48 $\pm$ 2.29 & 409.60 $\pm$ 20.37 & \textbf{5.74 $\pm$ 0.07} & \textbf{14.87 $\pm$ 0.50} \\
MH + GradNorm & 0.77 $\pm$ 0.55 & \textbf{0.90 $\pm$ 0.50} & 9.64 $\pm$ 0.53 & 20.41 $\pm$ 0.55 & 35.25 $\pm$ 3.28 & 94.56 $\pm$ 1.45 & 382.77 $\pm$ 13.02 & 5.17 $\pm$ 0.27 & 10.98 $\pm$ 1.23 \\
MH + PcGrad & 1.19 $\pm$ 0.14 & 1.09 $\pm$ 0.03 & 9.91 $\pm$ 0.33 & 20.47 $\pm$ 0.45 & 37.77 $\pm$ 0.76 & 95.53 $\pm$ 3.81 & 375.68 $\pm$ 7.65 & 5.67 $\pm$ 0.07 & \underline{15.00 $\pm$ 0.18} \\
\bottomrule\end{tabular}}
\end{table}
 
\begin{table}[H]\centering
\caption{50th-percentile hypervolume results for MO-NAS (Part 2, IN-1K/MOP). For each task, methods within one standard deviation (or $5\times10^{-4}$) of the best are bolded and the best of the remaining methods is underlined; the shaded row marks the RFM variant with the most task-best results overall.}
\label{tab:hv_monas2_proxy_hv50}
\scriptsize\setlength{\tabcolsep}{2.5pt}\resizebox{\textwidth}{!}{%
\begin{tabular}{lccccccccc}
\toprule
Methods & IN-1K/MOP1 & IN-1K/MOP2 & IN-1K/MOP3 & IN-1K/MOP4 & IN-1K/MOP5 & IN-1K/MOP6 & IN-1K/MOP7 & IN-1K/MOP8 & IN-1K/MOP9 \\ \midrule
$\mathcal{D}$(best) & -- & -- & -- & -- & -- & -- & -- & -- & -- \\ \midrule
E2E & 5.49 $\pm$ 0.24 & 5.63 $\pm$ 0.30 & 13.84 $\pm$ 0.20 & \textbf{4.44 $\pm$ 0.21} & 4.64 $\pm$ 0.19 & 9.89 $\pm$ 0.71 & 4.58 $\pm$ 0.14 & 9.03 $\pm$ 0.45 & 10.89 $\pm$ 0.83 \\
E2E + GradNorm & 5.22 $\pm$ 0.03 & 6.13 $\pm$ 0.08 & 12.91 $\pm$ 0.30 & 3.96 $\pm$ 0.11 & 4.28 $\pm$ 0.10 & 9.42 $\pm$ 0.32 & 4.39 $\pm$ 0.36 & 8.24 $\pm$ 0.23 & 10.06 $\pm$ 0.74 \\
E2E + PcGrad & 5.55 $\pm$ 0.24 & \underline{6.17 $\pm$ 0.05} & 13.94 $\pm$ 0.34 & 4.27 $\pm$ 0.14 & 4.43 $\pm$ 0.25 & 10.64 $\pm$ 1.21 & 4.75 $\pm$ 0.12 & 9.42 $\pm$ 0.22 & 11.49 $\pm$ 0.09 \\
MH & 5.51 $\pm$ 0.07 & 5.79 $\pm$ 0.22 & 13.37 $\pm$ 0.10 & 4.30 $\pm$ 0.19 & 4.46 $\pm$ 0.21 & 9.44 $\pm$ 0.30 & 4.63 $\pm$ 0.10 & 9.21 $\pm$ 0.09 & 11.74 $\pm$ 0.42 \\
MH + GradNorm & 5.37 $\pm$ 0.25 & 6.15 $\pm$ 0.05 & 13.40 $\pm$ 0.48 & 3.19 $\pm$ 0.37 & \textbf{4.38 $\pm$ 0.49} & 8.28 $\pm$ 0.19 & 3.65 $\pm$ 0.33 & 8.28 $\pm$ 0.88 & 13.28 $\pm$ 2.17 \\
MH + PcGrad & 5.30 $\pm$ 0.05 & \textbf{6.19 $\pm$ 0.11} & 13.43 $\pm$ 0.24 & 4.11 $\pm$ 0.26 & 4.31 $\pm$ 0.24 & 10.52 $\pm$ 0.93 & 4.72 $\pm$ 0.09 & 9.27 $\pm$ 0.16 & 11.04 $\pm$ 0.48 \\
\bottomrule\end{tabular}}
\end{table}
 
\begin{table}[H]\centering
\caption{50th-percentile hypervolume results for MORL. For each task, methods within one standard deviation (or $5\times10^{-4}$) of the best are bolded and the best of the remaining methods is underlined; the shaded row marks the RFM variant with the most task-best results overall.}
\label{tab:hv_morl_proxy_hv50}
\small\setlength{\tabcolsep}{6pt}
\begin{tabular}{lcc}
\toprule
Methods & MO-Hopper & MO-Swimmer \\ \midrule
$\mathcal{D}$(best) & -- & -- \\ \midrule
E2E & 4.87 $\pm$ 0.13 & 2.20 $\pm$ 0.43 \\
E2E + GradNorm & 4.93 $\pm$ 0.27 & 2.89 $\pm$ 0.14 \\
E2E + PcGrad & 4.85 $\pm$ 0.21 & 2.58 $\pm$ 0.06 \\
MH & 5.00 $\pm$ 0.15 & 2.68 $\pm$ 0.11 \\
MH + GradNorm & \textbf{5.41 $\pm$ 0.39} & 2.66 $\pm$ 0.18 \\
MH + PcGrad & 4.77 $\pm$ 0.03 & 2.67 $\pm$ 0.11 \\
MM & 4.82 $\pm$ 0.04 & 2.57 $\pm$ 0.10 \\
MM + RoMA & 4.97 $\pm$ 0.24 & 2.58 $\pm$ 0.17 \\
MM + IOM & 4.83 $\pm$ 0.06 & 2.62 $\pm$ 0.10 \\
MM + ICT & \underline{5.14 $\pm$ 0.04} & \textbf{3.33 $\pm$ 0.18} \\
MM + Tri-Mentor & 4.86 $\pm$ 0.19 & \underline{3.06 $\pm$ 0.15} \\
\midrule
PCD & 4.84 $\pm$ 0.04 & 2.45 $\pm$ 0.10 \\
ParetoFlow & \textbf{5.23 $\pm$ 0.44} & 2.34 $\pm$ 0.45 \\
ParetoFlow+ & 5.05 $\pm$ 0.31 & 2.35 $\pm$ 0.33 \\
\midrule
RFM-AN (ours) & 4.84 $\pm$ 0.04 & 2.36 $\pm$ 0.09 \\
RFM-APG (ours) & 4.83 $\pm$ 0.03 & 2.40 $\pm$ 0.15 \\
RFM-DAMG (ours) & 4.84 $\pm$ 0.04 & 2.36 $\pm$ 0.11 \\
\rowcolor{oursbest} RFM-Cone (ours) & 4.84 $\pm$ 0.04 & 2.34 $\pm$ 0.05 \\
\bottomrule\end{tabular}
\end{table}
 
\begin{table}[H]\centering
\caption{50th-percentile hypervolume results for RE (Part 1). For each task, methods within one standard deviation (or $5\times10^{-4}$) of the best are bolded and the best of the remaining methods is underlined; the shaded row marks the RFM variant with the most task-best results overall.}
\label{tab:hv_re1_proxy_hv50}
\scriptsize\setlength{\tabcolsep}{2.5pt}\resizebox{\textwidth}{!}{%
\begin{tabular}{lcccccccc}
\toprule
Methods & RE21 & RE22 & RE23 & RE24 & RE25 & RE31 & RE32 & RE33 \\ \midrule
$\mathcal{D}$(best) & -- & -- & -- & -- & -- & -- & -- & -- \\ \midrule
E2E & 4.25 $\pm$ 0.31 & 4.46 $\pm$ 0.21 & \textbf{4.71 $\pm$ 0.27} & 4.27 $\pm$ 0.03 & 4.14 $\pm$ 0.28 & \textbf{10.56 $\pm$ 0.12} & 10.63 $\pm$ 0.01 & 10.60 $\pm$ 0.01 \\
E2E + GradNorm & \textbf{4.33 $\pm$ 0.28} & \textbf{4.82 $\pm$ 0.02} & 3.74 $\pm$ 0.05 & 4.20 $\pm$ 0.03 & 4.36 $\pm$ 0.05 & 9.78 $\pm$ 0.07 & 9.84 $\pm$ 0.39 & 8.19 $\pm$ 0.84 \\
E2E + PcGrad & \textbf{4.59 $\pm$ 0.00} & \textbf{4.84 $\pm$ 0.00} & \textbf{4.84 $\pm$ 0.00} & 3.53 $\pm$ 0.05 & 4.27 $\pm$ 0.21 & \textbf{10.39 $\pm$ 0.55} & \textbf{10.64 $\pm$ 0.00} & 7.85 $\pm$ 0.42 \\
MH & \textbf{4.59 $\pm$ 0.00} & \textbf{4.77 $\pm$ 0.15} & \textbf{4.84 $\pm$ 0.00} & 4.46 $\pm$ 0.36 & \textbf{4.29 $\pm$ 0.78} & \textbf{10.65 $\pm$ 0.00} & \underline{10.63 $\pm$ 0.00} & 10.58 $\pm$ 0.03 \\
MH + GradNorm & \textbf{4.01 $\pm$ 0.92} & \textbf{4.02 $\pm$ 0.87} & 4.09 $\pm$ 0.14 & 3.39 $\pm$ 0.53 & 4.58 $\pm$ 0.24 & \textbf{10.39 $\pm$ 0.57} & 7.41 $\pm$ 0.94 & \textbf{9.50 $\pm$ 1.95} \\
MH + PcGrad & \textbf{4.59 $\pm$ 0.00} & \textbf{4.61 $\pm$ 0.34} & 4.09 $\pm$ 0.61 & 4.15 $\pm$ 0.37 & \textbf{4.69 $\pm$ 0.16} & 10.50 $\pm$ 0.15 & 10.39 $\pm$ 0.21 & 10.54 $\pm$ 0.08 \\
\bottomrule\end{tabular}}
\end{table}
 
\begin{table}[H]\centering
\caption{50th-percentile hypervolume results for RE (Part 2). For each task, methods within one standard deviation (or $5\times10^{-4}$) of the best are bolded and the best of the remaining methods is underlined; the shaded row marks the RFM variant with the most task-best results overall.}
\label{tab:hv_re2_proxy_hv50}
\scriptsize\setlength{\tabcolsep}{2.5pt}\resizebox{\textwidth}{!}{%
\begin{tabular}{lccccccc}
\toprule
Methods & RE34 & RE35 & RE36 & RE37 & RE41 & RE42 & RE61 \\ \midrule
$\mathcal{D}$(best) & -- & -- & -- & -- & -- & -- & -- \\ \midrule
E2E & \textbf{10.07 $\pm$ 0.02} & 10.49 $\pm$ 0.03 & 9.66 $\pm$ 0.17 & \textbf{6.68 $\pm$ 0.04} & \underline{20.58 $\pm$ 0.11} & 21.83 $\pm$ 0.58 & \textbf{108.30 $\pm$ 0.18} \\
E2E + GradNorm & 7.21 $\pm$ 1.27 & 10.37 $\pm$ 0.01 & 5.42 $\pm$ 0.03 & 6.59 $\pm$ 0.02 & 16.90 $\pm$ 1.03 & 19.65 $\pm$ 2.26 & 96.02 $\pm$ 3.08 \\
E2E + PcGrad & \textbf{10.07 $\pm$ 0.02} & \textbf{10.57 $\pm$ 0.00} & 9.65 $\pm$ 0.26 & \textbf{6.69 $\pm$ 0.01} & \textbf{20.69 $\pm$ 0.05} & \textbf{22.55 $\pm$ 0.15} & 101.10 $\pm$ 2.70 \\
MH & \textbf{10.06 $\pm$ 0.01} & \textbf{10.55 $\pm$ 0.02} & \underline{9.74 $\pm$ 0.09} & \textbf{6.67 $\pm$ 0.03} & 20.48 $\pm$ 0.07 & \underline{22.35 $\pm$ 0.12} & \textbf{107.94 $\pm$ 0.65} \\
MH + GradNorm & 8.69 $\pm$ 0.80 & 10.06 $\pm$ 0.32 & 4.96 $\pm$ 1.15 & 6.62 $\pm$ 0.03 & 20.48 $\pm$ 0.10 & 21.47 $\pm$ 0.14 & \textbf{104.36 $\pm$ 4.19} \\
MH + PcGrad & \underline{10.06 $\pm$ 0.01} & \textbf{10.47 $\pm$ 0.20} & 9.26 $\pm$ 0.16 & \textbf{6.67 $\pm$ 0.03} & 20.48 $\pm$ 0.08 & \textbf{22.17 $\pm$ 0.41} & 106.83 $\pm$ 0.24 \\
\bottomrule\end{tabular}}
\end{table}

\end{document}